\documentclass[preprint,3p,times,twocolumn]{elsarticle}
 
\usepackage{graphics} 
\usepackage{epsfig} 
\usepackage{amsmath} 
\usepackage{amssymb}  
\usepackage{amsthm}
\usepackage{upgreek}
\usepackage[linesnumbered,ruled]{algorithm2e}
\usepackage[utf8]{inputenc}
\usepackage[inline]{enumitem}
\usepackage{xcolor}

\usepackage{graphicx}
\usepackage{float}
\usepackage{multicol}
\usepackage{amsmath}
\usepackage{dblfloatfix}

\usepackage[official]{eurosym}

\usepackage{subfigure}

\usepackage{algpseudocode}
\usepackage{epstopdf}
\usepackage{multicol}
\usepackage[leftcaption]{sidecap}
\usepackage{booktabs}
\usepackage{siunitx}

\usepackage{todonotes}

\usepackage[hidelinks]{hyperref}
 \usepackage[normalem]{ulem}

\usepackage[nolist]{acronym} 

\newacro{vic}[VIC]{Variable Impedance Control}

\newacro{pi2}[PI\textsuperscript{2}]{Policy Improvement with Path Integrals}
\newacro{dmp}[DMP]{Dynamic Movement Primitive}
\newacro{pdmp}[PDMP]{Periodic \ac{dmp}}
\newacro{cmp}[CMP]{Compliant Movement Primitive}
\newacro{seds}[SEDS]{Stable Estimator of Dynamical Systems}
\newacro{sds}[SDS]{Stable Dynamical Systems}
\newacro{ilc}[ILC]{Iterative  Learning  Control}
\newacro{gp}[GP]{Gaussian Process}
\newacro{gpr}[GPR]{Gaussian Process Regression}
\newacro{gmm}[GMM]{Gaussian Mixture Model}
\newacro{gmr}[GMR]{Gaussian Mixture Regression}
\newacro{ppc}[PPC]{Passivity-Preservation Control}
\newacro{tpgmm}[TP-GMM]{Task-Parameterized \ac{gmm}}
\newacro{lfd}[LfD]{Learning from Demonstration}
\newacro{il}[IL]{Imitation Learning}
\newacro{wls}[WLS]{weighted least-squares}
\newacro{spd}[SPD]{Symmetric Positive Definite}
\newacro{emg}[EMG]{Electromyography}
\newacro{vil}[VIL]{Variable Impedance Learning}
\newacro{vilc}[VILC]{Variable Impedance Learning Control}
\newacro{ai}[AI]{Artificial Intelligent}
\newacro{sea}[SEA]{Series Elastic Actuation}
\newacro{dof}[DoF]{Degree of Freedom}
\newacro{vmp}[VMP]{Via-points Movement Primitive}
\newacro{lwr}[LWR]{Locally Weighted Regression}
\newacro{rl}[RL]{Reinforcement Learning}
\newacro{auv}[AUV]{Autonomous Underwater Vehicle}
\newacro{uav}[UAV]{Unmanned Areal Vehicle}
\newacro{cmaes}[CMA-ES]{Covariance Matrix Adaptation-Evolution Strategies}
\newacro{ccdmp}[CC-DMP]{Coordinate Change-\acp{dmp}}
\newacro{gpdmp}[GPDMP]{ Global Parametric Dynamic Movement Primitive}
\newacro{bbo}[BBO]{Black-Box Optimization}
\newacro{dmpbbo}[DMPBBO]{\ac{dmp} \ac{bbo}}
\newacro{ros}[ROS]{Robotic Operating System}
\newacro{cnn}[CNN]{Convolutional Neural Network}
\newacro{nn}[NN]{Neural Network}
\newacro{aedmp}[AEDMP]{AutoEncoded \ac{dmp}}
\newacro{power}[PoWER]{Policy Learning by Weighting Exploration with the Returns}
\newacro{pd}[PD]{Proportional Derivative}
\newacro{momp}[MoMP]{Mixture of Motor Primitives}
\newacro{promp}[ProMP]{Probabilistic Movement Primitives}
\newacro{hrl}[HRL]{Hierarchical \ac{rl}}
\newacro{kmp}[KMP]{Kernelized Movement Primitive}
\newacro{rbf}[RBF]{Radial Basis Function}
\newacro{rbfnn}[RBF-NN]{\ac{rbf}-\ac{nn}}
\newacro{mle}[MLE]{Maximum Likelihood Estimation}
\newacro{qpdmp}[QP-DMP]{Unit Quaternion-based Periodic \ac{dmp}}
\newacro{rmpdmp}[RMP-DMP]{Riemannian Metric-based Periodic \ac{dmp}}
\newacro{ds}[DS]{Dynamical System}
\newacro{rmp}[RMP]{Riemannian Motion Policy}
\newacro{fdm}[FDM]{Fast Diffeomorphic Matching}
\newacro{uq}[UQ]{Unit Quaternion}
\newacro{ours}[SDS-RM]{Stable Dynamical System on Riemannian Manifolds}
\newacro{cr}[CR]{Convergence Rate}
\newacro{rmse}[RMSE]{Root Mean Square Error}
\newacro{tt}[TT]{Training Time}
\newacro{ts}[TS]{Tangent Space}
\newacro{rgmm}[R-GMM]{Riemannian \ac{gmm}}
\newacro{eflow}[E-FLOW]{Euclideanizing Flows}
\newacro{iflow}[I-FLOW]{Imitation Flow}

\newcommand{\bm}[1]{\boldsymbol{\mathbf{#1}}}

\newcommand{\Kdemo}{{\bm{K}^{demo}}}

\newcommand{\q}{\bm{q}}

\newcommand{\LogQ}{\text{Log}^{q}}

\newcommand{\bP}{\bm{a}}
\newcommand{\bQ}{\bm{b}}
\newcommand{\gauss}{\mathcal{N}}
\newcommand{\riM}{\mathcal{M}}
\newcommand{\tsPsi}{\mathcal{T}_{\boldfrak{m}}\mathcal{M}}
\newcommand{\tsHatPsi}{\mathcal{T}_{\boldfrak{a}}\mathcal{M}}

\newcommand{\LogHatPsi}[1]{\text{Log}_{\boldfrak{a}}\left(#1\right)}

\newcommand{\LogHatPsiPsi}{\text{Log}_{\boldfrak{a}}\left(\boldfrak{g}\right)}
\newcommand{\LogPsi}[1]{\text{Log}_{\boldfrak{m}}\left(#1\right)}
\newcommand{\ExpPsi}[1]{\text{Exp}_{\boldfrak{m}}\left(#1\right)}

\newcommand{\LogG}{\text{Log}_{\boldfrak{g}}}
\newcommand{\Log}{\text{Log}}
\newcommand{\ExpG}{\text{Exp}_{\boldfrak{g}}}
\newcommand{\bFg}{\boldfrak{g}}
\newcommand{\bFp}{\boldfrak{a}}
\newcommand{\bFq}{\boldfrak{b}}
\newcommand{\bFm}{\boldfrak{m}}

\newcommand\norm[1]{\left\lVert#1\right\rVert}
\newcommand\expm[1]{\text{expm}(#1)}
\newcommand\logm[1]{\text{logm}(#1)}

\newcommand{\boldfrak}[1]{\bm{\mathfrak{#1}}}

\newcommand{\trsp}{{^{\top}}}

\newcommand{\figref}[1]{Fig.~\hyperref[#1]{\ref*{#1}}}
\newcommand{\figsref}[1]{Figures~\hyperref[#1]{\ref*{#1}}}
\newcommand{\Figref}[1]{Figure~\hyperref[#1]{\ref*{#1}}}

\newcommand{\tabref}[1]{Tab.~\hyperref[#1]{\ref*{#1}}}
\newcommand{\secref}[1]{Sec.~\hyperref[#1]{\ref*{#1}}}
\newcommand{\Secref}[1]{Section~\hyperref[#1]{\ref*{#1}}}
\newcommand{\algoref}[1]{Algorithm~\hyperref[#1]{\ref*{#1}}}

\newcommand{\wrt} {{wrt}~} %
\newcommand{\eg} {{e.g.,}~} %
\newcommand{\ie} {{i.e.,}~} %
\newcommand{\etal}{\MakeLowercase{{et al.}}}

\newlength{\Oldarrayrulewidth}

\definecolor{darkgreen}{rgb}{0.0,0.49,0.19}

\newcommand{\panda}{Franka Emika Panda}

\relax
\DeclareFontFamily{U}{eur}{\skewchar\font'177}
\DeclareFontShape{U}{eur}{m}{n}{%
	<-6> eurm5 <6-8> eurm7 <8-> eurm10}{}
\DeclareFontShape{U}{eur}{b}{n}{%
	<-6> eurb5 <6-8> eurb7 <8-> eurb10}{}
\DeclareSymbolFont{ugrf@m}{U}{eur}{m}{n}
\SetSymbolFont{ugrf@m}{bold}{U}{eur}{b}{n}
\DeclareMathSymbol{\upomega}{\mathord}{ugrf@m}{"21}
\graphicspath{{figs/}}

\newtheorem{theorem}{\bf Theorem}

\newtheorem{remark}{\bf Remark}

\journal{Robotics and Autonomous Systems}

\begin{document}
\begin{frontmatter}

\title{Learning Stable Robotic Skills on Riemannian Manifolds}
\author[1]{Matteo Saveriano\corref{cor1}}
\ead{Matteo.Saveriano@unitn.it}
\author[2]{Fares~J.~Abu-Dakka}
\ead{fares.abu-dakka@aalto.fi}
\author[2]{Ville Kyrki}
\ead{ville.kyrki@aalto.fi}
\cortext[cor1]{Corresponding author}
\address[1]{Automatic Control Lab, Department of Industrial Engineering, University of Trento, Trento, Italy}
\address[2]{Intelligent Robotics Group, Department of Electrical Engineering and Automation, Aalto University, Finland}

%



\begin{abstract}
In this paper, we propose an approach to learn stable dynamical systems evolving on Riemannian manifolds. The approach leverages a data-efficient procedure to learn a diffeomorphic transformation that maps simple stable dynamical systems onto complex robotic skills. By exploiting mathematical tools from differential geometry, the method ensures that the learned skills fulfill the geometric constraints imposed by the underlying manifolds, such as unit quaternion (UQ) for orientation and symmetric positive definite (SPD) matrices for impedance{, while preserving the convergence to a given target.} 
{The proposed approach is firstly tested in simulation on a public benchmark, obtained by projecting Cartesian data into UQ and SPD manifolds, and compared with existing approaches.}  
Apart from evaluating the approach on a public benchmark, several experiments were performed on a real robot performing bottle stacking in different conditions and a drilling task in cooperation with a human operator. The evaluation shows promising results in terms of learning accuracy and task adaptation capabilities.  
\end{abstract}
	

\end{frontmatter}
\section{Introduction}
Robots that successfully operate in {smart manufacturing} 
have to be capable of precisely controlling their behavior in terms of movements and physical interactions. In this respect, {modern industrial and} service robots need flexible representations of such intended behaviors in terms of motion, impedance, and force skills {(see Fig.~\ref{fig:first_img})}. Developing such representations is then a key aspect to speed-up the integration of robotic solutions in social and industrial environments. 

{Learning approaches have the possibility to unlock the full potential of smart robotic solutions. Among the others,} the \ac{lfd} paradigm \cite{schaal1999imitation} aims at developing learning solutions that allow the robot to enrich its skills via human guidance. Among the several existing approaches \cite{billard2016learning, Ravichandar2020Recent}, the idea of encoding robotic skills into stable \acp{ds} has gained interest in the \ac{lfd} community~\cite{Ijspeert2013Dynamical, khansari2011learning, zadeh2014learning, neumann2015learning, perrin2016fast, blocher2017learning, duan2017fast, lemme2014neural}. This class of \ac{lfd} approaches assumes that the demonstrations are observations of the time evolution of a \ac{ds}. \acp{ds} are flexible motion generators that allow, for example, to encode continuous and discrete motions~\cite{khadivar2021learning}, to reactively avoid possible collisions~\cite{ginesi2021dynamic, saveriano2017human, saveriano2019learning}, or to update the underlying dynamics in an incremental fashion~\cite{kronander2015incremental, saveriano2018incremental}.

Although \ac{ds}-based representations have several interesting properties, most works assume that data are observations of a Euclidean space. However, this is not always the case in robotics where data can belong to a non-Euclidean space. Typical examples of non-Euclidean data are orientation trajectories represented as rotation matrices and \acp{uq}, and \ac{spd} matrices for quantities such as inertia, impedance gains, and manipulability, which all belong to Riemannian manifolds. Applying standard arithmetic tools from Euclidean geometry on Riemannian manifolds leads to inaccuracies and incorrect representations, which can be avoided if the proper mathematical tools developed for Riemannian manifolds are used~\cite{do1992riemannian}.   
\begin{figure}[t]
\includegraphics[width=\columnwidth]{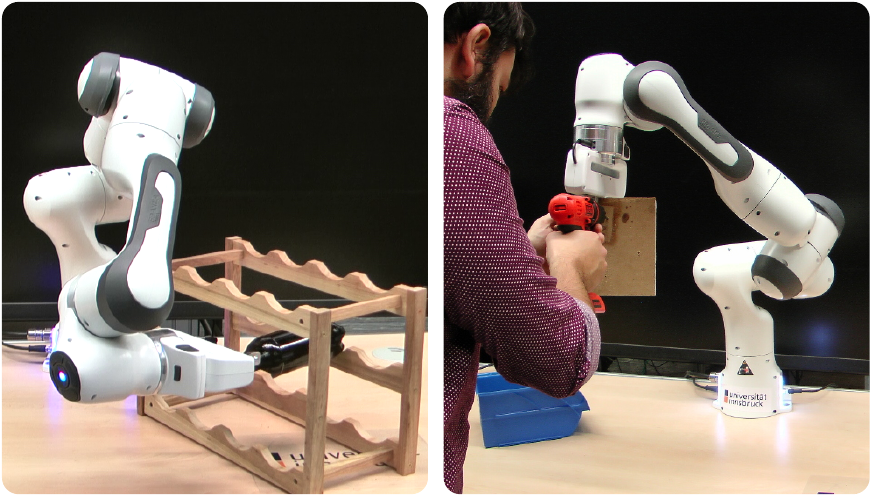}
\caption{Experimental setups involving a robot (Franka Emika Panda) and Riemannian data. (Left) The robot stacks a bottle in a rack. This task requires adaptation of position and orientation (\ac{uq}). (Right) The robot performs a collaborative drilling with a human. This task exploits stiffness (\ac{spd} matrix) modulation.}
\label{fig:first_img}
\end{figure}
%
%
%
\begin{table*}[t!]
     \centering
	\caption{{Comparison among state-of-the-art approaches for \ac{lfd} on Riemannian manifolds and \ac{ours}.}}
	\vspace{0.1cm}
     \label{tab:comparison_appraoches}
     \resizebox{\textwidth}{!}{%
     {\renewcommand\arraystretch{1.3} 
 	\begin{tabular}{ cccccccc }
 	\toprule
 	 & End-Point & Multiple demos & Multimodal & Time-independent & Data-efficiency & Accuracy$^+$ & Training time\\
 	 \midrule 
     DMP \cite{Ude2014Orientation, abudakka2020Geometry} & \checkmark & - & - & - &  high & medium & low\\
     R-GMM \cite{calinon20gaussians} & \checkmark & \checkmark  & - & - & high & medium & low\\
     %
     %
     FDM$^\dagger$ \cite{perrin2016fast} & \checkmark & - & - &  \checkmark  & high & medium & low\\
     E-FLOW \cite{rana2020euclidean} & \checkmark & \checkmark &  \checkmark  &  \checkmark& low & medium* & high\\
     I-FLOW \cite{urain2020imitationflow, urain2021learning} & \checkmark & \checkmark &  \checkmark  &  \checkmark & low & medium* & high\\
     \ac{ours} (ours) & \checkmark & \checkmark &  \checkmark & \checkmark & high  &  high & low\\
     \bottomrule
     \multicolumn{8}{l}{{\footnotesize $^\dagger$FDM is not designed to work on Riemannian manifolds. However, our formulation allows their use to learn a diffeomorphism in TS.}}\\
     \multicolumn{8}{l}{{\footnotesize *E-FLOW and I-FLOW need an hyper-parameters search to reach high accuracy. However, performing an hyper-parameters search requires a GPU}}\\
     \multicolumn{8}{l}{{\footnotesize cluster and it is beyond the scope of this paper. With trial-and-error we found an hyper-parameters configuration that gives a medium accuracy.}}
 \end{tabular}
 }}
 \end{table*}
%

In this paper, we propose
\begin{itemize}
    \item a novel geometry-aware approach to encode demonstrations evolving on a Riemannian manifold into a stable dynamical system (\textit{\ac{ours}}).
    \item mathematical foundations for \ac{ours} to work in any Riemannian manifold. In the paper, we provided two manifolds as case studies: (\emph{i}) \ac{uq}, and (\emph{ii}) \ac{spd} manifolds.
    \item an extension of the LASA handwriting dataset~\cite{khansari2011learning}, a popular benchmark in \ac{ds}-based \ac{lfd}, to generate \ac{uq} and \ac{spd} trajectories. The resulting \textit{Riemannian LASA} dataset, publicly available at \url{gitlab.com/geometry-aware/riemannian-lasa- dataset}, will serve as a benchmark for a quantitative comparison of newly developed approaches.
\end{itemize}
To this end, \ac{ours} leverages the concept of diffeomorphism (a bijective, continuous, and continuously differentiable mapping with a continuous and continuously differentiable inverse), to transform simple and stable (base) dynamics into complex robotic skills. Building on tools from Riemannian geometry, we first present rigorous stability proofs for the base and the diffeomorphed \acp{ds}. We then present a data-efficient approach that leverages a \ac{gmm}~\cite{cohn1996active} to learn the diffeomorphic mapping. As a result, we obtain a \ac{ds} that accurately represents data evolving on Riemannian manifolds and that preserves the convergence to a given target as well as the geometric structure of the data.

The rest of the paper is organized as follows. \Secref{sec:related} presents the related literature. Basic concepts of Riemannian geometry are given in~\secref{sec:background}. \Secref{sec:proposed} provides the theoretical foundations of \ac{ours}. In~\secref{sec:learning}, we present an approach to learn stable skills via diffeomorphic maps. \ac{ours} is evaluated on a public benchmark and compared against a state-of-the-art approach in~\secref{sec:validation}. Experiments on a real robot (Franka Emika Panda) are presented in~\secref{sec:robot_experiment}. \Secref{sec:concl} states the conclusion and proposes further research directions.

\section{Related Works}
\label{sec:related}
%
%
This section presents the related works in the field of \ac{lfd} both in general and with a specific focus on \ac{ds}-based and geometry-aware learning approaches. {For a clear comparison, the main features of \ac{ours} and of existing approaches are summarized in \tabref{tab:comparison_appraoches}.}
 
\textbf{\acf{lfd}} provides an user friendly framework that allows non-roboticists to teach robots and enables robots to autonomously perform new tasks based on that human demonstration~\cite{schaal1999imitation}. Over the last couple of decades many \ac{lfd} approaches have been developed \cite{billard2016learning, Ravichandar2020Recent}. 
\ac{lfd} approaches can be categorized in two main groups depending on the underlying learning strategy: 
\begin{enumerate*}[label={(\textit{\roman*})}]
	\item deterministic approaches that try to reproduce the demonstrations with a function approximator, \eg a neural network \cite{seker2019conditional, bahl2020neural}, and 
	\item probabilistic approaches that learn a probability distribution from the demonstrations. Examples of this category include \ac{gmm} \cite{reynolds2009gaussian}, \ac{tpgmm} \cite{calinon2014task}, \ac{promp} \cite{Paraschos2013Probabilistic}, and \ac{kmp} \cite{Huang2020Toward}.
\end{enumerate*}

In both groups, it is possible to learn and retrieve a \textit{static} or a \textit{dynamic} mapping between input and output. In a static mapping the current input, \eg the time, is mapped into the desired output, \eg the robot's joint angles, while a \ac{ds} maps the input, \eg the robot joint angles, into its time derivative(s), \ie the joint velocities. 
In this paper, we focus on learning stable \acp{ds} from data belonging to Riemannian manifolds. However, in order to provide a more comprehensive review of the existing literature, we also highlight some recent Riemannian-based approaches that learn static mappings.  
Zeestraten~\etal~\cite{zeestraten2017approach} exploited Riemannian metrics to learn orientation trajectories with \ac{tpgmm}. Huang~\etal~\cite{Huang2020Toward} proposed to train \ac{kmp} in the tangent space of unit quaternion trajectories. Kinaesthetic is used to estimate full stiffness matrices of an interaction tasks, which subsequently used to learn force-based variable impedance profiles using Riemannian metric-based \ac{gmm} / \ac{gmr} \cite{AbuDakka2018Force}. Later the authors proposed to train \ac{kmp} on the tangent space of \ac{spd} matrices \cite{abudakka2021probabilistic}. Jaquier~\etal~\cite{jaquier2017gaussian} formulated a tensor-based \ac{gmm}/\ac{gmr} on \ac{spd} manifold, which later has been exploited in manipulability transfer and tracking problem~\cite{jaquier2020geometry}. {Calinon~\cite{calinon20gaussians} extended the \ac{gmm} formulation to a variety of manifolds including \ac{uq} and \ac{spd}. We name this approach \ac{rgmm} and compare its performance againt \ac{ours} in Sec.~\ref{sec:validation}.}

\textbf{Encoding manipulation skills into stable \acp{ds}} is achieved by leveraging time-dependent or time-independent \acp{ds}.
\acp{dmp} \cite{Ijspeert2013Dynamical} are a prominent approach to encode robotic skills into time-dependent representations. The classical \ac{dmp} formulation has been extended in different ways \cite{saveriano2021dynamic}. 
Among the others, extensions to Riemannian manifolds are relevant for this work. Abu-Dakka~\etal. extend classical \acp{dmp} to encode discrete~\cite{AbuDakka2015Adaptation} and periodic~\cite{abudakka2021Periodic} unit quaternion trajectories, while the work in~\cite{Ude2014Orientation} also considers rotation matrices. The stability of orientation \acp{dmp} is shown in~\cite{saveriano2019merging}. In~\cite{abudakka2020Geometry}, \acp{dmp} are reformulated to generate discrete \ac{spd} profiles. \acp{dmp} exploit a time driven forcing term, learned from a single demonstration, to accurately reproduce the desired skill. The advantage of having a time driven forcing term is that the learning scale well to high-dimensional spaces---as the learned term only depends on a scalar input. The disadvantage is that the generated motion is forced to follow the stereotypical (\ie demonstrated) one and generalizes poorly outside the demonstration area \cite{neumann2015learning}. Alternative approaches~\cite{saveriano2018incremental, saveriano2020energy} learn a state-depended forcing term, but they still use a vanishing time signal to suppress the forcing term and retrieve asymptotic stability. The tuning of this vanishing signal affects the reproduction accuracy\footnote{The reader is referred to \cite{Ijspeert2013Dynamical, saveriano2021dynamic} for a thorough discussion on advantages and disadvantages of time-dependent and time-independent \ac{ds} representations.}. 

The \ac{seds}~\cite{khansari2011learning} is one of the first approaches that learn stable and autonomous \acp{ds}. It exploits Lyapunov theory to derive stability constraints for a \ac{gmr}-\ac{ds}. The main limitations of \ac{seds} are the relatively long training time and the reduced accuracy on complex motions. The loss of accuracy is caused by the stability constraints and it is also referred as the \textit{accuracy vs stability dilemma}~\cite{neumann2015learning}. To alleviate this issue, the approach in~\cite{duan2017fast} derives weak stability constraints for a neural network-based \ac{ds}. Contraction theory~\cite{lohmiller1998contraction} is used in~\cite{blocher2017learning} to derive stability conditions for a \ac{gmr}-\ac{ds} in Euclidean space, and in~\cite{ravichandar2019learning} to encode stable \ac{uq} orientations. %
The extension of \ac{seds} in~\cite{zadeh2014learning} significantly reduces inaccuracies and training time by separately learning a (possibly) unstable \ac{ds} and a stabilizing control input. 

Alternative approaches leverage a diffeomorphic mapping between Euclidean spaces to accurately fit the demonstrations while preserving the stability. Neumann~\etal~\cite{neumann2015learning} learn a diffeomorphism to map the demonstrations into a space where quadratic stability constraints introduce negligible deformations. The approach is effective but needs a long training time{, as experimentally shown in}~\cite{saveriano2020energy}. In a similar direction, \ac{eflow}~\cite{rana2020euclidean} fits a diffeomorphism that linearizes the demonstrations as if they were generated by a linear \ac{ds}. The opposite idea, \ie transform straight lines into complex motions, is exploited \ac{fdm}~\cite{perrin2016fast} and by \ac{iflow}~\cite{urain2020imitationflow, urain2021learning}. \ac{fdm} uses a composition of locally weighted (with an exponential kernel) translations to rapidly fit a diffeomorphism from a single demonstration. {\ac{ours} builds on similar ideas, but it works on Riemannian data and with multiple demonstrations.} \ac{ours} is compared with \ac{fdm} in \secref{sec:validation}.  
{\ac{eflow} and \ac{iflow} exploit deep invertible neural networks to learn a diffeomorphic mapping. } 
Invertible neural networks are slightly more accurate than classical methods{, as experimentally shown in~\cite{urain2020imitationflow}. However, as shown in the experimental comparison carried out in \secref{sec:validation},  approaches based on deep networks require long training time and intense hyperparameters search} which represents a limitation in typical \ac{lfd} settings. Therefore, \ac{ours} leverages a mature and data-efficient approach in \ac{lfd} (\ac{gmm}/\ac{gmr}) to fit a diffeomorphism between \acp{ts}. {It is worth mentioning that neither \ac{eflow} nor \ac{iflow} have been tested on a Riemannian manifold different form the \ac{uq}, but our formulation is rather general and allows us to extend existing approaches to different manifolds.} 

Finally, the \ac{rmp} framework~\cite{ratliff2018riemannian, mukadam2020riemannian} consist of closed-loop controllers embedded in a second-order dynamics that exploits the Riemannian metric to optimally track a reference motion. In this respect, \ac{rmp} and \ac{ours} are complementary: \ac{ours} can be used to generate the desired motion and \ac{rmp} to optimally track it.

\section{Background}
\label{sec:background}

A Riemannian manifold $\riM$ is a smooth differentiable topological space, for each point of which $\boldfrak{m} \in \riM$, it is possible to compute a \acf{ts} $\tsPsi$ {(see Fig.~\ref{fig:ts}). The \ac{ts} is equipped with a positive definite inner product $\langle \bP, \bQ \rangle_{\bFm} \in \mathbb{R}$, where $\bP, \bQ \in \tsPsi$ and $\bFm \in \riM$ is the point where the \ac{ts} is computed. The inner product allows the definition of the notion of distance on the manifold.} 
Depending on the Riemannian manifold, points on the \ac{ts} can be {vectors, matrices, or more complex mathematical objects}. For examples, a \ac{ts} of \acp{uq} consists 3D vectors, while a \ac{ts} of \ac{spd} matrices consists of symmetric matrices. In this work, we indicate a point on the \ac{ts} using bold capital letters, \ie $\bP \in \tsPsi$.  

\begin{figure}[t]
	\centering
	\def\svgwidth{\linewidth}
	{\fontsize{8}{8}
		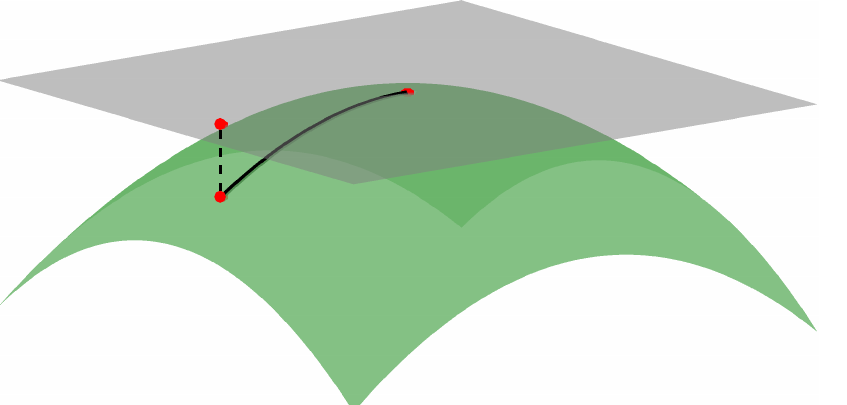}
	\caption{A Riemannian manifold $\riM$ (green surface) and its tangent space $\tsPsi$ (gray plane) centered at $\boldfrak{m}$. The logarithmic $\LogPsi{\cdot}$ and exponential $\ExpPsi{\cdot}$ maps move points from the manifold ($\boldfrak{a}$ and $\boldfrak{b}$) to the tangent space ($\bm{a}$ and $\bm{b}$) and vice-versa.}
	\label{fig:ts}
\end{figure}
The operator that transforms points from a Riemannian manifold to its tangent space is the \textit{logarithmic map} $\LogPsi{\cdot}:\,\riM\,\rightarrow\,\tsPsi$. Its inverse operator is called the \textit{exponential map} $\ExpPsi{\cdot}:\,\tsPsi\,\rightarrow\,\riM$. {The \textit{parallel transport} $\mathcal{T}_{\boldfrak{m}_1\rightarrow \boldfrak{m}_2}:\,\mathcal{T}_{\boldfrak{m}_1}\mathcal{M}\,\rightarrow\,\mathcal{T}_{\boldfrak{m}_2}\mathcal{M}$ moves elements between tangent spaces while preserving their angle.} 

{In case of \ac{uq} manifold, the operators are defines as
}
\begin{align}
	\ExpPsi{\bP} &= \boldfrak{m}\cos(\norm{\bP})+\frac{\bP}{\norm{\bP}}\sin(\norm{\bP}),\label{eq:expSm}\\
	\LogPsi{\boldfrak{a}} &= \frac{\boldfrak{a}-(\boldfrak{m}\trsp\boldfrak{a})\boldfrak{m}}{\norm{\boldfrak{a}-(\boldfrak{m}\trsp\boldfrak{a})\boldfrak{m}}}d(\boldfrak{m},\boldfrak{a}),	\label{eq:logSm} \\
	\begin{split}
	    \mathcal{T}_{\boldfrak{m}_1\rightarrow \boldfrak{m}_2}\left(\bP\right) &= \bP - \frac{\Log_{\boldfrak{m}_1}(\boldfrak{m}_2)\trsp\bP}{d(\boldfrak{m}_1,\boldfrak{m}_2)^2}\cdot\\&\quad\cdot\left(\Log_{\boldfrak{m}_1}(\boldfrak{m}_2)+\Log_{\boldfrak{m}_2}(\boldfrak{m}_1)\right)
	 \end{split}, \label{eq:transpSm}
\end{align}
where $d(\boldfrak{m},\bFp)=\arccos(\bFp\trsp\boldfrak{m})$ and $\norm{\cdot}$ is the Euclidean norm. 
\begin{table}[t!]
	\caption{Re-interpretation of basic standard operations in a Riemannian manifold \cite{Pennec2006}.}
	\label{tab:basicOperation}
	\resizebox{\columnwidth}{!}{  
		\begin{tabular}{lcc}
		\addlinespace[-\aboverulesep] 
    \cmidrule[\heavyrulewidth]{2-3}
			& \sc{Euclidean space}      &\sc{Riemannian manifold}       \\ \midrule
			\sc{Subtraction}		& $\overrightarrow{\bm{m}\bm{a}}=\bm{a}-\bm{m}$	& $\overrightarrow{\boldfrak{m}\boldfrak{a}}=\LogPsi{\boldfrak{a}}$             \\
			\sc{Addition}  		& $\bm{p} = \bm{m}+\overrightarrow{\bm{m}\bm{a}}$ & $\boldfrak{a}=\ExpPsi{\overrightarrow{\boldfrak{m}\boldfrak{a}}}$             \\
			\sc{Distance}		& $\text{dist}(\bm{m},\bm{a})=\parallel \bm{a}-\bm{m}\parallel$ & $\text{dist}(\boldfrak{m},\boldfrak{a})=\parallel\overrightarrow{\boldfrak{m}\boldfrak{p}}\parallel_{\boldfrak{m}}$              \\
			\sc{Interpolation}	& $\bm{m}(t)=\bm{m}_1+t \overrightarrow{\bm{m}_1\bm{m}_2}$ & $\boldfrak{m}(t)=\text{Exp}_{\boldfrak{m}_1}(t\overrightarrow{\boldfrak{m}_1\boldfrak{m}_2})$             \\ \bottomrule
		\end{tabular}
		}
\end{table}


In case of \ac{spd} manifold, the operators are defined as in~\cite{Pennec2006,sra2015conic}
\begin{align}
	\ExpPsi{\bP} =& \boldfrak{m}^\frac{1}{2}\text{expm}\Big(\boldfrak{m}^{-\frac{1}{2}}\bP\boldfrak{m}^{-\frac{1}{2}}\Big)\boldfrak{m}^\frac{1}{2},\label{eq:expSPD}\\
	\LogPsi{\bFp} =& 	\boldfrak{m}^\frac{1}{2}\text{logm}\Big(\bFm^{-\frac{1}{2}}\bFp\boldfrak{m}^{-\frac{1}{2}}\Big)\boldfrak{m}^\frac{1}{2},
	\label{eq:logSPD} \\
	\mathcal{T}_{\boldfrak{m}_1\rightarrow \boldfrak{m}_2}\left(\bP\right) &= \boldfrak{m}_2^\frac{1}{2}\boldfrak{m}_1^\frac{1}{2}\,\bP\,\boldfrak{m}_1^\frac{1}{2}\boldfrak{m}_2^\frac{1}{2}, \label{eq:transpSPD}
\end{align}
where $\expm{\cdot}$ and $\logm{\cdot}$ are the matrix exponential and logarithm functions respectively. Here, the affine-invariant distance~\cite{Pennec2006} is used to compute $d(\bFp, \boldfrak{m})$
\begin{equation}
d(\bFp, \boldfrak{m}) = \norm{\text{logm}\Big(\bFm^{-\frac{1}{2}}\bFp\boldfrak{m}^{-\frac{1}{2}}\Big)}_{F},
\label{eq:affine_invariant}
\end{equation}
where $\norm{\cdot}_{F}$ is the Frobenius norm.

{Other basic operations on manifolds can be computed as shown in Tab.~\ref{tab:basicOperation}.}  

\section{Diffeomorphed \ac{ds} on Manifolds}
\label{sec:proposed}
{Inspired by~\cite{rana2020euclidean},} we apply a diffeomorphism to a stable (base) \ac{ds} evolving on the Riemannian manifold (see \figref{fig:idea}). The diffeomorphism, learned from a set of demonstrations, deforms the trajectories of the base \ac{ds} to accurately match the  demonstrations. In order to effectively apply this idea, we first need to design a stable \ac{ds} evolving on the Riemannian manifold. This \ac{ds} provides the basic motion that connects the initial point to a desired goal. Second, we need to show that a diffeomorphic transformation preserves the stability of the base \ac{ds}. This result, known for the Euclidean space, extends to Riemannian manifolds as shown in this section.
\begin{figure}[t]
	\centering
	\def\svgwidth{\linewidth}
	{\fontsize{8}{8}
		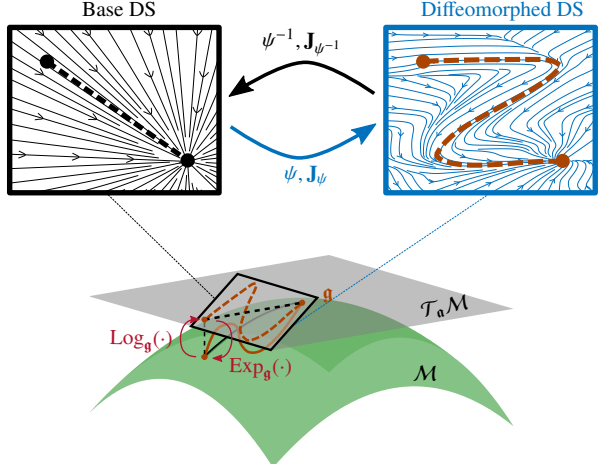}
	\caption{The idea of \ac{ours} is to learn a diffeomorphic mapping between \acp{ts} to transform the trajectories of a base \ac{ds} into complex motions. Motion on the \ac{ts} is projected back and forth to the underlying Riemannian manifold using exponential and logarithmic maps respectively.}
	\label{fig:idea}
\end{figure}

\subsection{Base \ac{ds} on Riemannian manifolds}
In Euclidean space, a linear \ac{ds} in the form 
\begin{equation}
    \dot{\bm{x}} = - k_{\bm{x}}(\bm{x} - \bm{g})
    \label{eq:linear_ds}
\end{equation}
is often used as base\footnote{Note that some authors assume, without loss of generality, that $\bm{g} = \bm{0}$.} \ac{ds}. Indeed, if the gain $k_{\bm{x}} > 0$, the linear \ac{ds} in~\eqref{eq:linear_ds} is globally exponentially stable~\cite{slotine1991applied}. This implies that the linear \ac{ds} generates trajectories connecting any initial point $\bm{x}(0) \in \mathbb{R}^n$ with any goal $\bm{g} \in \mathbb{R}^n$.

We seek for an ``equivalent'' of the stable linear dynamics for Riemannian manifolds.
In this work, the base \ac{ds} is the non-linear dynamics
\begin{equation}
    \dot{\bFp} = {k_{\bFp}\bP = k_{\bFp}\LogHatPsiPsi ~ \in \tsHatPsi},
    \label{eq:log_ds}
\end{equation}
where the logarithmic map is introduced in \secref{sec:background} and its expression depends on the considered Riemannian manifold. 
The non-linear \ac{ds} in~\eqref{eq:log_ds}, called a \textit{geodesic} \ac{ds}, shares similarities with the linear \ac{ds} in~\eqref{eq:linear_ds} and, for this reason, it is used in this work as base \ac{ds}. Indeed, {similarly to the term $(\bm{x} - \bm{g})$ in~\eqref{eq:linear_ds}, the logarithmic map in~\eqref{eq:log_ds}} represents the displacement between the current point $\bFp \in \riM$ and the goal $\bFg \in \riM$. {It is worth mentioning that in~\eqref{eq:log_ds} we consider the \ac{ts} at the current point $\bFp$ instead of a fixed \ac{ts} at the goal $\bFg$. As discussed in~\cite{calinon20gaussians}, working in a single \ac{ts} is inaccurate and introduces severe deformations in the learned trajectory. On the contrary, in our formulation we always place the \ac{ts} at the current point, therefore minimizing the distortion.}

{The base \ac{ds} in~\eqref{eq:log_ds} is also exponentially stable if the gain $k_{\bFp} > 0$.}
This implies that the base \ac{ds} in~\eqref{eq:log_ds} generates trajectories connecting any initial point $\bFp(0)$ with any goal $\bFg$ (see \figref{fig:log_ds_k}). {For completeness,} stability results are formally stated in Theorem~\ref{th:stab_log_ds} (asymptotic stability) and Remark~\ref{th:exp_stab_log_ds} (exponential stability). 
\begin{theorem}
\label{th:stab_log_ds}
{Let the gain  $k_{\bFp} > 0$. Assume also that $\Log_{\bFg}(\bFp) = - \mathcal{T}_{\bFg\rightarrow \bFp}\LogHatPsiPsi$, where $\mathcal{T}_{\bFg\rightarrow \bFp}$ is the parallel transport from $\bFg$ to $\bFp$. Under these assumptions, the \ac{ds} in~\eqref{eq:log_ds} has a globally (in its domain of definition) asymptotically stable equilibrium at $\bFg$.}
\end{theorem}
\begin{proof}
{Let us express the velocity field~\eqref{eq:log_ds} in $\bFg$ using parallel transport. By assumption, it holds that
\begin{equation}
    \mathcal{T}_{\bFg\rightarrow \bFp} \dot{\bFp} = - k_{\bFp}\Log_{\bFg}(\bFp) ~ \in \mathcal{T}_{\bFg}\mathcal{M}.
    \label{eq:log_ds_par_transport}
\end{equation}
To prove the stability of~\eqref{eq:log_ds_par_transport}, one can define the Lyapunov candidate function~\cite{pait2010properties}
 \begin{equation}
     V(\bFp) = \langle \bP, \bP \rangle_{\bFg} = \langle \Log_{\bFg}(\bFp), \Log_{\bFg}(\bFp) \rangle_{\bFg}
     \label{eq:lyapunov_candidate}
 \end{equation}
 and follow the same arguments outlined in~\cite[Section 4.2.2]{jaquier2020geometry} for \ac{spd} matrices.}
\end{proof}
{%
\begin{remark}
\label{rem:parallel_transport}
The assumption made in Theorem~\ref{th:stab_log_ds} that $\log_{\bFg}(\bFp) = - \mathcal{T}_{\bFg\rightarrow \bFp}\LogHatPsiPsi$ holds for any Riemannian manifold~\cite[Theorem 6]{fiori2021manifold}.
\end{remark}
}
{%
\begin{remark}
\label{rem:single_chart}
The results of Theorem~\ref{th:stab_log_ds} hold where the logarithmic map is uniquely defined, \ie in a region that does not contain points conjugate to $\boldfrak{g}$~\cite{pait2010properties}. For \ac{spd} matrices, this holds everywhere~\cite{Pennec2006}. Hence, Theorem~\ref{th:stab_log_ds} is globally valid on the manifold of \ac{spd} matrices. For unit $m$-sphere (including \ac{uq}), instead, the logarithmic map $\LogHatPsi{\cdot}$ is defined everywhere apart from the antipodal point $-\boldfrak{a}$~\cite{pennec2006intrinsic}.
\end{remark}
}

Even if it is not strictly required in our approach, it is interesting to show that, like the linear \ac{ds} in~\eqref{eq:linear_ds}, the \ac{ds} in~\eqref{eq:log_ds} is exponentially stable.
\begin{remark}
\label{th:exp_stab_log_ds}
Under the assumptions of Theorem~\ref{th:stab_log_ds}, the \ac{ds} in~\eqref{eq:log_ds} has a globally (in its domain of definition) exponentially stable equilibrium at $\bFg$~\cite[Section 4.2.2]{jaquier2020geometry}.
\end{remark}

\subsection{Diffeomorphed \ac{ds}}
{A diffeomorphic map or a diffeomorphism $\psi$ is a bijective, continuous, and with continuous inverse $\psi^{-1}$ change of coordinates. In this work, we assume that $\psi: \tsHatPsi \rightarrow \tsHatPsi$, \ie the diffeomorphism transforms a global coordinate $\bP \in  \tsHatPsi$ into another global coordinate $\bQ = \psi(\bP) \in  \tsHatPsi$. Further, the diffeomorphism $\psi$ is assumed to be bounded, \ie it maps bounded vectors into bounded vectors.}

{In order to match a set of demonstrations, we apply $\psi$ to the base \ac{ds} in~\eqref{eq:log_ds}. More in details, let us assume that $\bP = \LogHatPsiPsi$, and that the dynamics of the global coordinates $\dot{\bFp}$ is described by~\eqref{eq:log_ds}. By taking the time derivative of $\bFq$, we obtain the \ac{ds}~\cite{rana2020euclidean}}
\begin{equation}
    \dot{\bFq} = \frac{\partial \psi}{\partial \bP} \dot{\bFp} =k_{\bFp} \bm{J}_{\psi} (\psi^{-1}(\bQ)) \psi^{-1}(\bQ),
    \label{eq:diff_log_ds}
\end{equation}
where $\bm{J}_{\psi} (\cdot)$ is the \textit{Jacobian matrix} of $\psi$ evaluated at a particular point and the inverse mapping $\bP = \psi^{-1}(\bQ)$ is used to remove the dependency on $\bP$. Having assumed that $\psi$ is a bounded diffeomorphism, the right side of~\eqref{eq:diff_log_ds} satisfies the Lipschitz condition and, therefore, the \ac{ds} in~\eqref{eq:diff_log_ds} admits a unique solution.
The stability of the \textit{diffeomorphed dynamics} in~\eqref{eq:diff_log_ds} is stated by the following theorem.
\begin{theorem}
\label{th:stab_diff_ds}
The diffeomorphed \ac{ds} in~\eqref{eq:diff_log_ds} inherits the stability properties of the base \ac{ds} in~\eqref{eq:log_ds}. That is, if the base \ac{ds} is globally (in its domain of definition) asymptotically stable so is the  diffeomorphed \ac{ds}. 
\end{theorem}
\begin{proof}
{From the proof of Theorem~\ref{th:stab_log_ds}, it holds that $V(\bFp)$ defined in~\eqref{eq:lyapunov_candidate} is a Lyapunov function for the base \ac{ds} in~\eqref{eq:log_ds}. As shown in~\cite[Section 3.2]{rana2020euclidean}, the function $V_\psi(\bFq) = \langle \psi^{-1}(\bQ), \psi^{-1}(\bQ) \rangle_{\bFg'}$, where $\bFg'$ is the point where $\psi^{-1}(\LogHatPsi{\bFg}) = \bm{0}$, is a valid Lyapunov function for the diffeomorphed \ac{ds} in~\eqref{eq:diff_log_ds}.}
\end{proof}
\begin{remark}
\label{rem:diff_ds_goal}
Theorem~\ref{th:stab_diff_ds} states the convergence of the \ac{ds}~\eqref{eq:diff_log_ds} to the equilibrium $\bFg'$. This point may differ from the equilibrium $\bFg$ of the base \ac{ds}~\eqref{eq:log_ds}. However, in \ac{lfd}, we are interested in converging to a given goal---let's say  $\bFg$ for simplicity. Assuming that the inverse mapping $\psi^{-1}(\cdot)$ is identity at the goal, \ie \emph{$\psi^{-1}(\LogHatPsi{\bFg})=\LogHatPsi{\bFg}=\bm{0}$}, it is straightforward to show from Theorem~\ref{th:stab_diff_ds} that the \ac{ds}~\eqref{eq:diff_log_ds} also convergences to $\bFg$.
\end{remark}

{
Given the global coordinate $\bQ$, we compute the corresponding manifold point $\bFq$ through the exponential map as
\begin{equation}
    \bFq = \text{Exp}_{\bFp}(\bQ).
    \label{eq:projection}
\end{equation}
Recalling that the exponential map is a local diffeomorphism, the composite mapping $\text{Exp}_{\bFp}(\bQ) = \text{Exp}_{\bFp} \circ \psi(\bP)$ can be considered a diffeomorphism between manifolds. 
}
\section{Learning Stable Skills via Diffeomorphisms}
\label{sec:learning}
The stability theorems provided in \secref{sec:proposed} give solid theoretical foundations to our learning approach. In this section, we describe how training data are used to learn a diffeomorphism that maps the trajectory of the base \ac{ds} into a desired robotic skill. Moreover, we discuss how to apply our approach to two popular Riemannian manifolds, namely the unit quaternions and the \ac{spd} matrices. 

\subsection{Data pre-processing}
We aim at learning a diffeomophism $\psi(\cdot)$ that maps the trajectory $\bFp(t)$, solution of the base \ac{ds} in~\eqref{eq:log_ds}, into an arbitrary complex demonstration. To this end, let us assume the user provides a set of $D \geq 1$ demonstrations each containing $L$ points on a Riemannian manifold. Demonstrations are organized in the set $\boldfrak{B}=\left\lbrace \bFq^d_l \right\rbrace^{L,D}_{l=1, d=1}$, where each $\bFq^d_l \in \mathcal{M}$. We also assume that the demonstrations converge to the same goal ($\boldfrak{b}^1_L = \cdots = \bFq^D_L = \boldfrak{g}$) and that a sampling time $\delta t$ is known. {It is worth mentioning that, when orientation trajectories are collected from demonstrations with a real robot, it is needed to extract \acp{uq} from rotation matrices. This is because the robot's forward kinematics is typically expressed as a homogeneous transformation matrix~\cite{siciliano2009robotics}. While numerically extracting \acp{uq} from a sequence of rotation matrices, it can happen that the approach returns a quaternion at time $t$ and its antipodal at $t+1$. This is because antipodal \acp{uq} represent the same rotation. To prevent this discontinuity, one can check that the dot product $\boldfrak{q}_t\cdot \boldfrak{q}_{t+1} > 0$, otherwise replace $\boldfrak{q}_{t+1}$ with $-\boldfrak{q}_{t+1}$.}

Given the set of demonstrations $\boldfrak{B}$, we generate a set of $D$ base trajectories by {projecting~\eqref{eq:log_ds} on the manifold}. 
More in details, we set the initial condition $\bFp^d_1 = \bFq^d_1$ and {project the tangent space velocity on the manifold using the exponential map as: } 
{
\begin{equation}
    \quad\bFp^d_{l+1} =  \text{Exp}_{\bFp^d_l}\left(\delta t \,\dot{\bFp}^d_{l} \right)~\forall\,l,d
    \label{eq:integration}
\end{equation}
}
The time derivative $\dot{\bFp}^d_{l}$ is defined as in~\eqref{eq:log_ds}, and the exponential/logarithmic maps {for \ac{uq} and \ac{spd} manifolds} are defined as in~\secref{sec:background}.
	
The $D$ base trajectories are organized in a set $\boldfrak{A}=\left\lbrace \bFp^d_l \right\rbrace^{L,D}_{l=1, d=1}$. {In order to transform the datasets $\boldfrak{A}$ and $\boldfrak{B}$ into suitable training data we proceed as follows. We use the logarithmic map $\bP^d_l = \text{Log}_{\bFp^d_l}(\bFg),\,\forall l,d$ to project the goal $\bFg$ in each \ac{ts} placed at $\bP^d_l$. We use the logarithmic map $\bQ^d_l = \text{Log}_{\bFp^d_l}(\bFq^d_{l}),\,\forall l,d$ to project each point in $\boldfrak{B}$ in the \ac{ts} placed at $\bP^d_l$. As a result, we obtain the sets $\bm{\mathcal{A}}=\left\lbrace \bP^d_l \right\rbrace^{L,D}_{l=1, d=1}$ and $\bm{\mathcal{B}}=\left\lbrace \bQ^d_l \right\rbrace^{L,D}_{l=1, d=1}$. In other words, we have in $\bm{\mathcal{A}}$ the points from the base the DS~\eqref{eq:log_ds} that exponentially converge towards $\bFg$ and in $\bm{\mathcal{B}}$ their demonstrated values. Note that each $\bP^d_l$ and $\bQ^d_l$ is expressed in the same \ac{ts} to make them comparable.}
%

After this procedure, the learning problem becomes how to fit a mapping between  $\bm{\mathcal{A}}$ and $\bm{\mathcal{B}}$ while preserving the stability. Exploiting the theoretical results in Theorem~\ref{th:stab_diff_ds} and Remark~\ref{rem:diff_ds_goal}, this learning problem is solved by fitting a diffeomorphism  between  $\bm{\mathcal{A}}$ and $\bm{\mathcal{B}}$. The resulting approach is presented in the rest of this section.

\subsection{\ac{gmm}/\ac{gmr}-based diffeomorphism}
A \ac{gmm} \cite{cohn1996active} models the joint probability distribution $p(\cdot)$ between training data as a weighted sum of $K$ Gaussian components $\mathcal{N}(\cdot)$, \ie
\begin{equation}
p(\bP, \bQ|\bm{\Theta}_k) = \sum_{k=1}^K \pi_k \mathcal{N}(\bP, \bQ|\bm{\mu}_k, \bm{\Sigma}_k), 
\label{eq:gmm}
\end{equation}
where each $\bm{\Theta}_k = \{\pi_k, \bm{\mu}_k, \bm{\Sigma}_k\}$ contains learning parameters. The $K$ mixing weights $\pi_k$ satisfy $\sum_{k=1}^K \pi_k = 1$, while the means and covariance matrices are defined as
\begin{equation}
\bm{\mu}_k = \begin{bmatrix} \bm{\mu}_k^{a} \\ \bm{\mu}_k^{b} \end{bmatrix}
,  \quad  \bm{\Sigma}_k = \begin{bmatrix} \bm{\Sigma}_k^{aa} & \bm{\Sigma}_k^{ab} \\ \bm{\Sigma}_k^{ba} & \bm{\Sigma}_k^{bb}\end{bmatrix}.
\end{equation}

As shown in \cite{khadivar2021learning} for periodic \acp{ds} in Euclidean space, we can use conditioning and expectation on the joint distribution in~\eqref{eq:gmm} to compute the mapping $\psi(\bP)$ and its inverse $\psi^{-1}(\bQ)$. The sought mappings $\psi(\bP) = \mathbb{E}[p(\bQ|\bP)]$ and $\psi^{-1}(\bQ) = \mathbb{E}[p(\bP|\bQ)]$  are computed in closed-form using \ac{gmr} \cite{cohn1996active, calinon09book} as:
\begin{equation}
\begin{split}
    \psi(\bP) &= \sum_{k=1}^K h_k(\bP)\left(\bm{\mu}_k^{b} \right. \\ & \left. \quad+ \bm{\Sigma}_k^{ba}(\bm{\Sigma}_k^{aa})^{-1}(\bP - \bm{\mu}_k^{a}) \right) , \\
    h_k(\bP) &= \frac{\pi_k\mathcal{N}(\bP|\bm{\mu}_k^{a}, \bm{\Sigma}_k^{aa})}{\sum_{i=1}^K \pi_i\mathcal{N}(\bP|\bm{\mu}_i^{a}, \bm{\Sigma}_i^{aa})},
\end{split}
\label{eq:gmr_diffeo}
\end{equation}
and
\begin{equation}
\begin{split}
    \psi^{-1}(\bQ) &=  \sum_{k=1}^K h_k(\bQ)\left(\bm{\mu}_k^{a} \right. \\ & \left.\quad + \bm{\Sigma}_k^{ab}(\bm{\Sigma}_k^{bb})^{-1}(\bQ - \bm{\mu}_k^{b}) \right) , \\
    h_k(\bQ) &= \frac{\pi_k\mathcal{N}(\bQ|\bm{\mu}_k^{b}, \bm{\Sigma}_k^{bb})}{\sum_{i=1}^K \pi_i\mathcal{N}(\bQ|\bm{\mu}_i^{b}, \bm{\Sigma}_i^{bb})}.
\end{split}
\label{eq:gmr_inv_diffeo}
\end{equation}
It is worth noticing that since both $\psi(\bP)$ and its inverse $\psi^{-1}(\bQ)$ exist and are differentiable, $\psi(\bP)$ is a diffeomorphism.  

In order to build the \ac{ds} in~\eqref{eq:diff_log_ds}, we need to compute the Jacobian matrix $\bm{J}_{\psi} (\bP)$ which has the closed-form expression given in~\eqref{eq:gmr_jacobian}. For completeness, we provide the full derivation of $\bm{J}_{\psi} (\bP)$ in Appendix~\ref{app:gmr_jacobian}.
Note that the term $\hat{\psi}(\bP)$ in~\eqref{eq:gmr_jacobian} is already computed in~\eqref{eq:gmr_diffeo} and can be reused to speed-up the computation of the Jacobian.

	\begin{equation}
		\begin{split}
			\bm{J}_{\psi} (\bP) &= \sum_{k=1}^K h_k(\bP) \Biggl[ \bm{\Sigma}_k^{ba}(\bm{\Sigma}_k^{aa})^{-1} +\\ 
			&\ \, \Biggl(\sum_{i=1}^K h_i(\bP) \left(\bm{\Sigma}_i^{aa}\right)^{-1}(\bP - \bm{\mu}_i^{a}) - \\
			&\quad\ \ h_k(\bP) \left(\bm{\Sigma}_k^{aa}\right)^{-1}(\bP - \bm{\mu}_k^{a})\Biggr) \hat{\psi}(\bP)\trsp \Biggr], \\
			\hat{\psi}(\bP) &= \bm{\mu}_k^{b} + \bm{\Sigma}_k^{ba}(\bm{\Sigma}_k^{aa})^{-1}(\bP - \bm{\mu}_k^{a}).
		\end{split}
		\label{eq:gmr_jacobian}
	\end{equation}

\subsection{Point-to-point motion on Riemannian manifolds}\label{subsec:proposed}
The \ac{gmm}/\ac{gmr}-based diffeomorphism presented in the previous section does not explicitly consider that we aim at reproducing discrete motions, \ie motions with a specific initial and final point. In particular, there is no guarantee that the learned diffeomorphism is an identity at the goal, \ie that $\psi(\text{Log}_{\boldfrak{g}}(\boldfrak{g}))=\psi(\bm{0})=\psi^{-1}(\bm{0})=\bm{0}$, which is sufficient to guarantee that base and diffeomorphed \acp{ds} have the same goal (Remark~\ref{rem:diff_ds_goal}).
This property is of importance in \ac{ds}-based \ac{lfd}, as we are generally interested in converging to a given target that is independent from the learning process. Moreover, since base and diffeomorphed \acp{ds} have the same initial condition ($\bFp_0=\bFq_0$), it is also beneficial that the learned diffeomorphism is an identity at the initial point, \ie that {$\psi(\text{Log}_{\boldfrak{a}_0}(\bFg))=\bm{a}_0 = \psi^{-1}(\text{Log}_{\boldfrak{a}_0}(\bFq_0)) = \bm{b}_0$}, to prevent discontinuities in the initial velocity.

In order to force the diffeomorphism to be identity at the goal, we augment the learned \ac{gmm} with a ``small'' component placed at {$\text{Log}_{\boldfrak{g}}(\boldfrak{g})=\bm{0}$}. 
More in details, we augment the $K$ learned components of the \ac{gmm}~\eqref{eq:gmm} with $\pi_{K+1}$ and $\mathcal{N}(\bP, \bQ|\bm{\mu}_{K+1}, \bm{\Sigma}_{K+1})$ and set $\bm{\mu}_{K+1} = \bm{0}$ and $\bm{\Sigma}_{K+1} = 
k_{\mathcal{N}}\bm{I}$. We re-scale the priors from $1$ to $K$ as $\pi_k = \pi_k - \pi_{K+1}/K$ to ensure that the $K+1$ priors sum up to one. Conditioning with this new component makes points to be mapped arbitrarily close to the goal. The distance to the goal depends on the gain $k_{\mathcal{N}}$. In this work,  we set $k_{\mathcal{N}}=\num{1e-5}$ and $\pi_{K+1} = 0.01$.  
We use a similar strategy to enforce a smooth start. Given the initial point on the manifold {$\bFp_0=\bFq_0=\overline{\bFp}$}, we project it into the \ac{ts} and place a small Gaussian around it, \ie $\pi_0 = 0.01$ and $\mathcal{N}(\bP, \bQ|\bm{\mu}_{0}= {[\text{Log}_{\overline{\bFp}}(\bFg)\trsp,\bm{0}\trsp]}\trsp, \bm{\Sigma}_{0}=\num{1e-5}\bm{I})$. Conditioning with this new component ensures that {$\psi(\text{Log}_{\overline{\bFp}}(\bFg)) \approx \psi^{-1}(\text{Log}_{\overline{\bFp}}(\bFq_0)) \approx \text{Log}_{\overline{\bFp}}(\bFg)$}. 
 
The possibility to change the goal, even on-the-fly (goal switching), is one of the appealing properties of \ac{ds}-based skill representations. {Changing the goal in our \ac{ours} also is possible. However, as already known for other \ac{ds}-based approaches~\cite{abudakka2020Geometry}, switching the goal may cause jumps in the generated trajectories and consequently high acceleration of the robot. In order to do avoid this problem, we exploit a geodesic dynamics that smoothly varies the goal from $\bFg$ to $\bFg_{new}$. In formulas
\begin{equation}
    \dot{\bFg}_{new} = k_{\bFg} \Log_{\bFg}(\bFg_{new}),
    \label{eq:goal_switching}
\end{equation}
where $k_{\bFg}>0$ ensures convergence to $\bFg_{new}$ as stated in Theorem~\ref{th:stab_log_ds}.}
%
An example of this procedure applied to \ac{uq} is shown in \secref{subsec:stack}.

\begin{figure}[t]
	\centering
	\def\svgwidth{\columnwidth}{\fontsize{8}{8}
		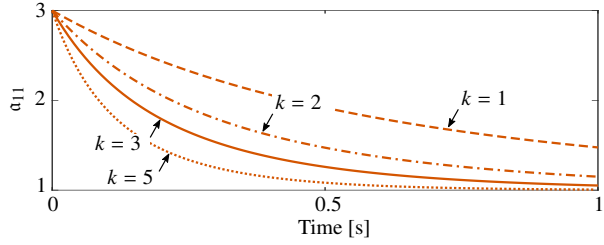}
	\caption{The convergence rate of the base \ac{ds} in~\eqref{eq:log_ds} depends on the gain $k_{\bFp} = \frac{k}{L \delta t}$. The figure shows the trajectory of the first entry $\mathfrak{a}_{11}$ of a $2\times 2$ \ac{spd} matrix for $\mathfrak{a}_{11}(0)=3$, $\mathfrak{g}_{11}=1$, $L=100$, $\delta t = 0.01\,$s, and $k \in [1,2,3,5]$.}
	\label{fig:log_ds_k}
\end{figure}

The tunable gain $k_{\bFq}$ controls the convergence rate of the base \ac{ds} in~\eqref{eq:log_ds} (see \figref{fig:log_ds_k}). Given the demonstrations $\boldfrak{B}=\left\lbrace \bFq^d_l \right\rbrace^{L,D}_{l=1, d=1}$ and the sampling time $\delta t$, we want that $\bFp^d_L$---obtained {using~\eqref{eq:integration}}---reaches $\bFq^d_L$ within a certain time. As shown in \figref{fig:log_ds_k}, too small values of $k_{\bFq}$ may fail to ensure that $\bFp^d_L \approx \bFq^d_L$. On the contrary, too large values of $k_{\bFq}$ cause a large part of the trajectory to be close to $\bFq^d_L$. This makes the learning problem harder as similar points need to be mapped into potentially different ones, \ie the data distribution tends to be multi-modal. Inspired by results from linear systems analysis, we can rewrite the control gain as $k_{\bFp} = \frac{k}{L \delta t}$. Dividing by $L \delta t$ makes the control gain independent from number of samples and training time. Therefore, the newly introduced parameter $k$ produces same results at different temporal scales. Given the initial point on the manifold $\bFp(0)$, one can choose how well the base trajectory has to reach the goal and define $k$ accordingly. 

The proposed approach to learn stable \acs{ds} on Riemannian manifolds is summarized in \algoref{alg:ds_riem}.
\begin{algorithm}[h!]
\caption{\ac{ours}}
	\label{alg:ds_riem}
	\begin{algorithmic}[1]
		\State{Pre-process data}
		\begin{itemize}
			\item[--] Collect demonstrations $\boldfrak{B}=\left\lbrace \bFq^d_l \right\rbrace^{L,D}_{l=1, d=1}$
			\item[--] Define the sampling time $\delta t$
			\item[--] Compute base trajectories $\boldfrak{A}=\left\lbrace \bFp^d_l \right\rbrace^{L,D}_{l=1, d=1}$ using \eqref{eq:integration}
			\item[--] Project to \ac{ts} using $\text{Log}_{\boldfrak{g}}(\cdot)$  (\eqref{eq:logSm} or \eqref{eq:logSPD}) to obtain  $\bm{\mathcal{A}}=\left\lbrace \bP^d_l \right\rbrace^{L,D}_{l=1, d=1}$, $\bm{\mathcal{B}}=\left\lbrace \bQ^d_l \right\rbrace^{L,D}_{l=1, d=1}$. (For \ac{spd} profiles, vectorize data using Mandel's notation.)
		\end{itemize}
		\State{Learn a diffeomorphism represented as a \ac{gmm}} 
		\begin{itemize}
			\item[--] Place a small Gaussian component at the origin of the \ac{ts} ($\pi_{K+1} = \overline{\pi}$, $\mathcal{N}(\bP, \bQ|\bm{0}, k_\mathcal{N}\bm{I})$) 
			\item[--] Place a small Gaussian component at the initial point ($\pi_{0} = \overline{\pi}$, $\mathcal{N}(\bP, \bQ|\overline{\bP}, k_\mathcal{N}\bm{I})$) 
		\end{itemize}
		\State{Generate Riemannian motion via \ac{gmr}}
		\begin{itemize}
			\item[--] Compute $\psi^{-1}$ from \eqref{eq:gmr_inv_diffeo}, $\bm{J}_{\psi}$  from \eqref{eq:gmr_jacobian}, and the velocity from \eqref{eq:diff_log_ds}. 
			\item[--] Project on the manifold using~\eqref{eq:integration}.
		\end{itemize} 
	\end{algorithmic}	
\end{algorithm}
{
\subsection{Learning in current and fixed TS}
\begin{figure}[t]
    \centering 
    \subfigure[Current TS]{\includegraphics[width=0.48\columnwidth,trim={2cm 2.5cm 1.5cm 2cm},clip]{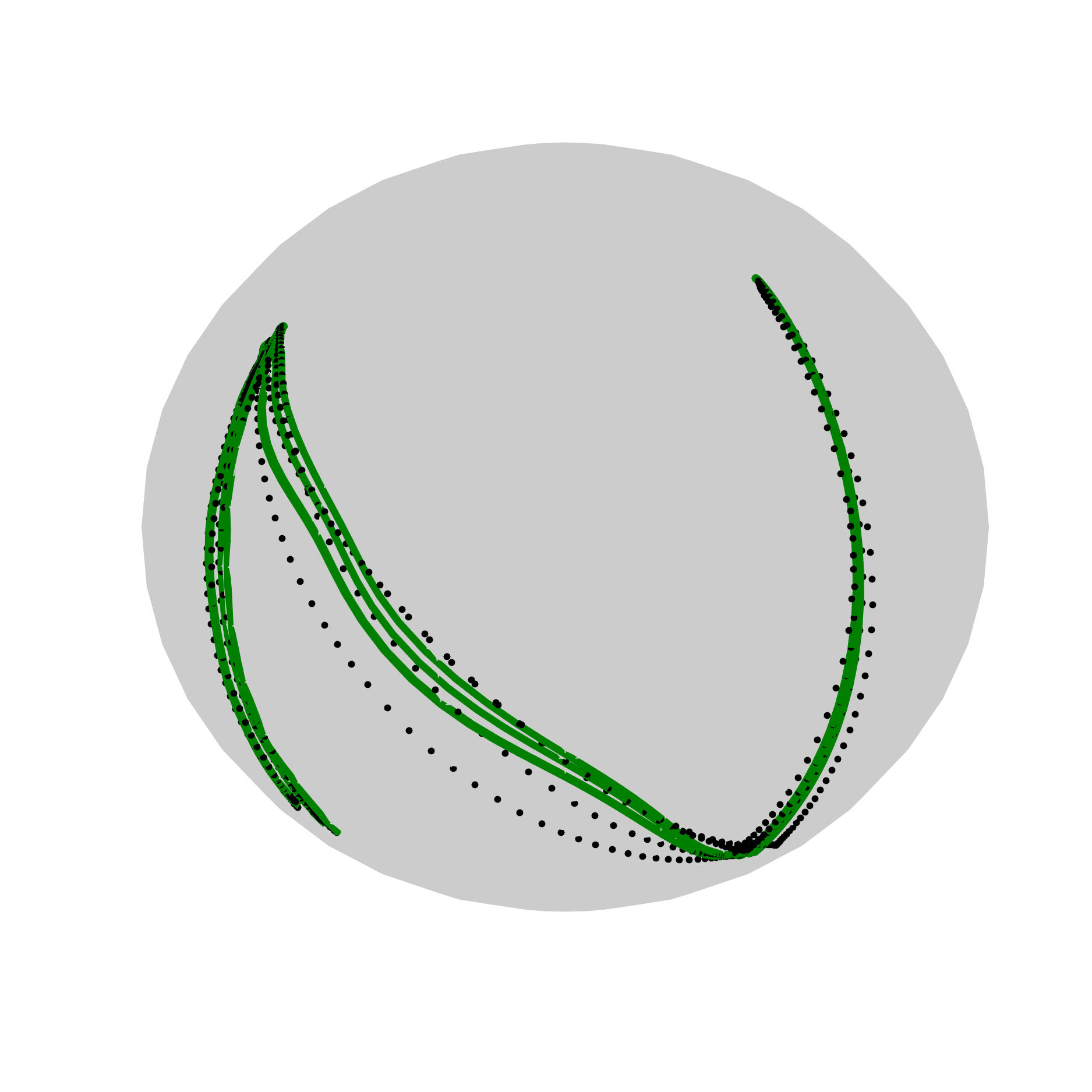}}
    %
    \subfigure[Lie algebra]{\includegraphics[width=0.48\columnwidth,trim={2cm 2.5cm 1.5cm 2cm},clip]{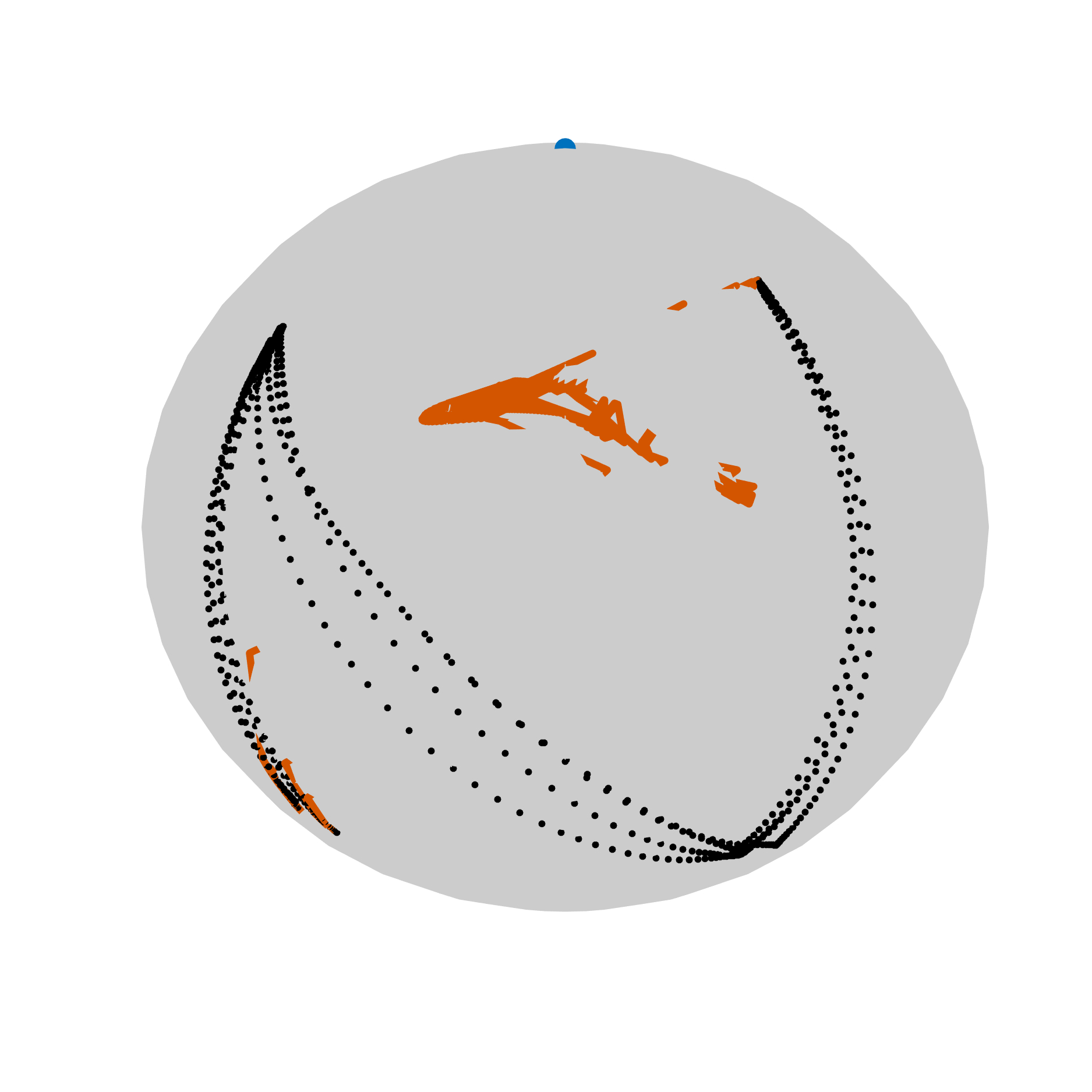}}
    \caption{Results obtained with \ac{ours} on the ``N'' shape.}
     \label{fig:N_shape}
\end{figure}
In this example, we show the benefits of learning manifold motions in the tangent space placed at the current point, called the \textit{current \ac{ts}} in this work, in contrast to projecting the entire entire trajectory in a unique (fixed) \ac{ts}. We use the ``N'' shape on $\mathcal{
S}^2$ provided in~\cite{calinon20gaussians} and shown in \figref{fig:N_shape} (black dots). The trajectory is designed to span both the north and south hemisphere, where the Lie algebra is known to introduce a non-negligible approximation~\cite{calinon20gaussians}.  

We follow the steps in \algoref{alg:ds_riem} using in one case the current \ac{ts} and the Lie algebra (\ie the \ac{ts} at the north pole) in the other. Qualitative results in \figref{fig:N_shape}(a) confirm that, using the current \ac{ts}, \ac{ours} can effectively encode the trajectory. The same result can be obtained using a \ac{ts} at the goal and parallel transporting the data at each step (see the assumption in Theorem~\ref{th:stab_log_ds}). However, this choice would increase the computational complexity due to the need of the parallel transport.  As expected, using the Lie algebra results in severe distortions (\figref{fig:N_shape}(b)). 
}

\section{Validation}
\label{sec:validation}
\begin{figure*}
    \centering 
    \vspace{-12pt}
    \includegraphics[trim={148 234 132 210},clip,width=0.12\textwidth]{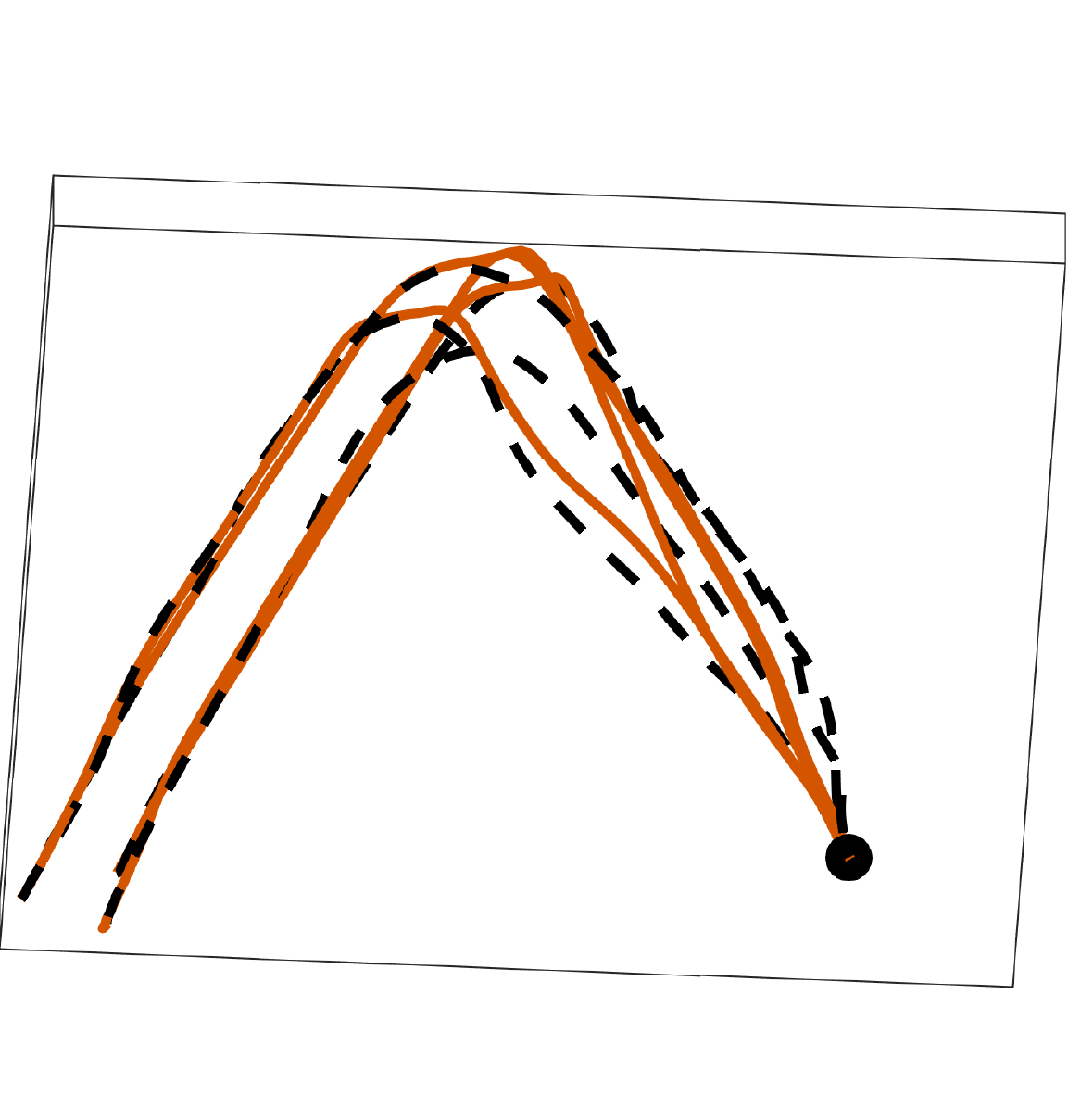}\hspace{-1mm}
    \vspace{-12pt}
    \includegraphics[trim={148 234 132 210},clip,width=0.12\textwidth]{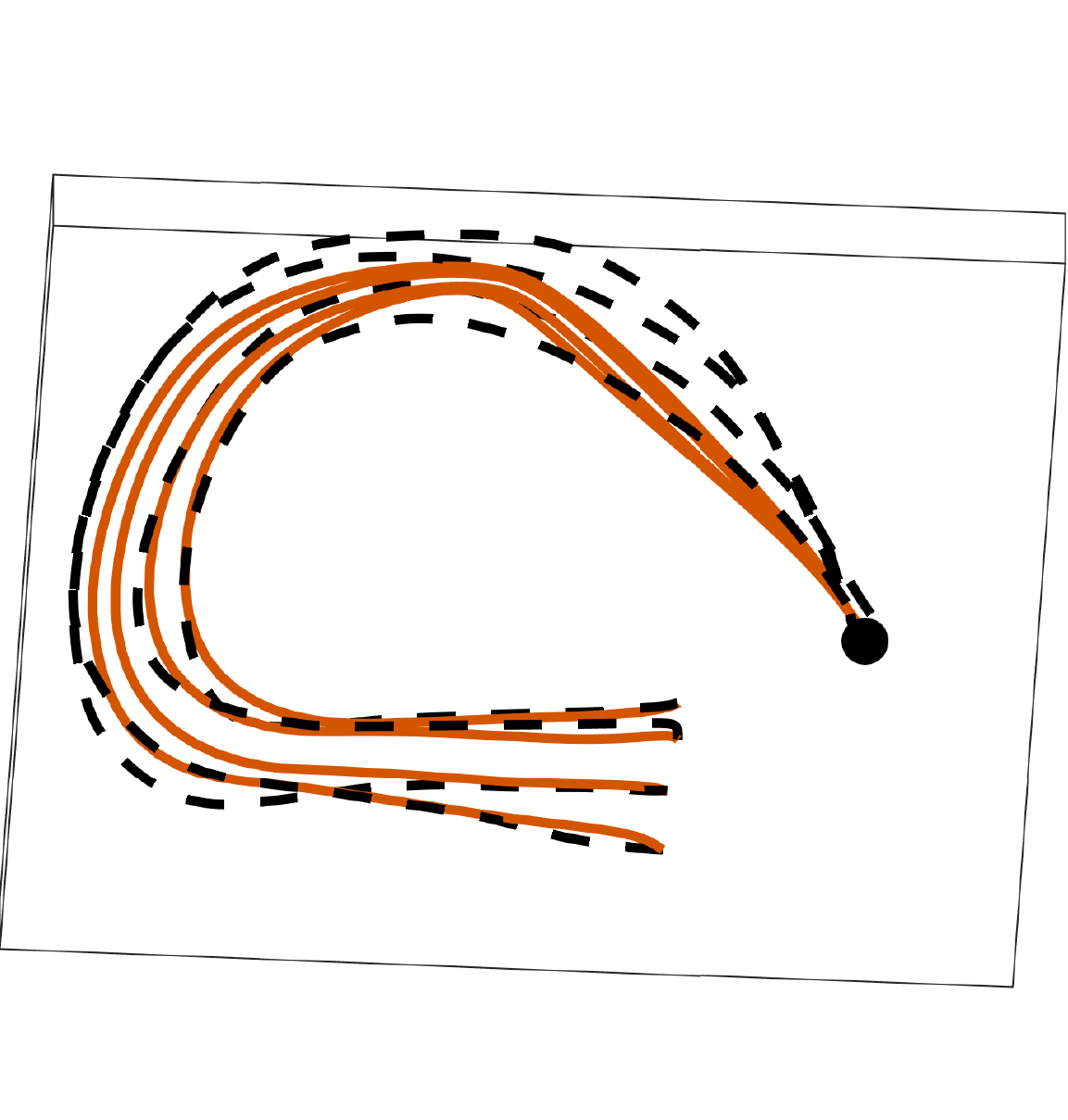}\hspace{-1mm}
    \includegraphics[trim={148 234 132 210},clip,width=0.12\textwidth]{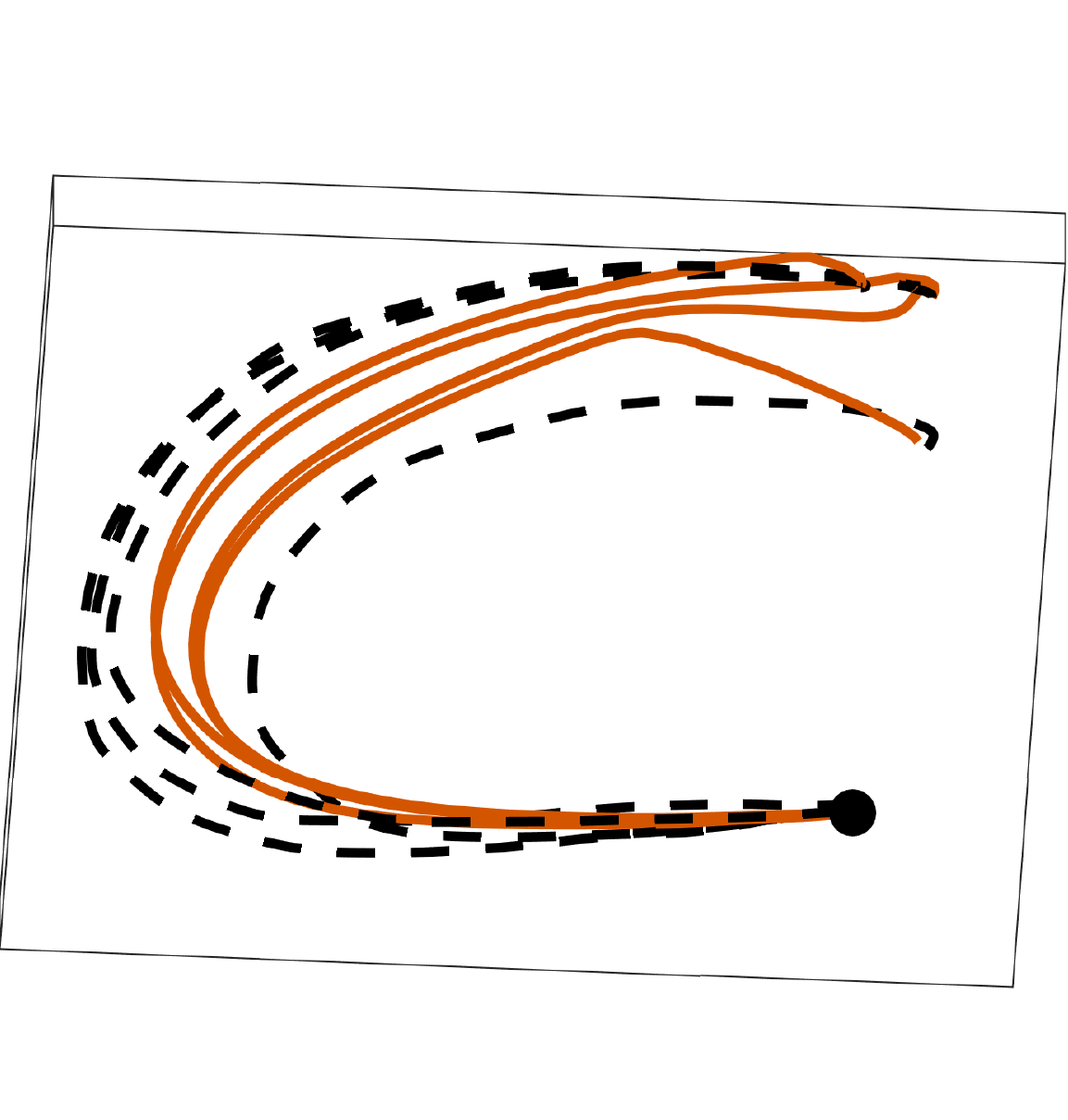}\hspace{-1mm}
    \includegraphics[trim={148 234 132 210},clip,width=0.12\textwidth]{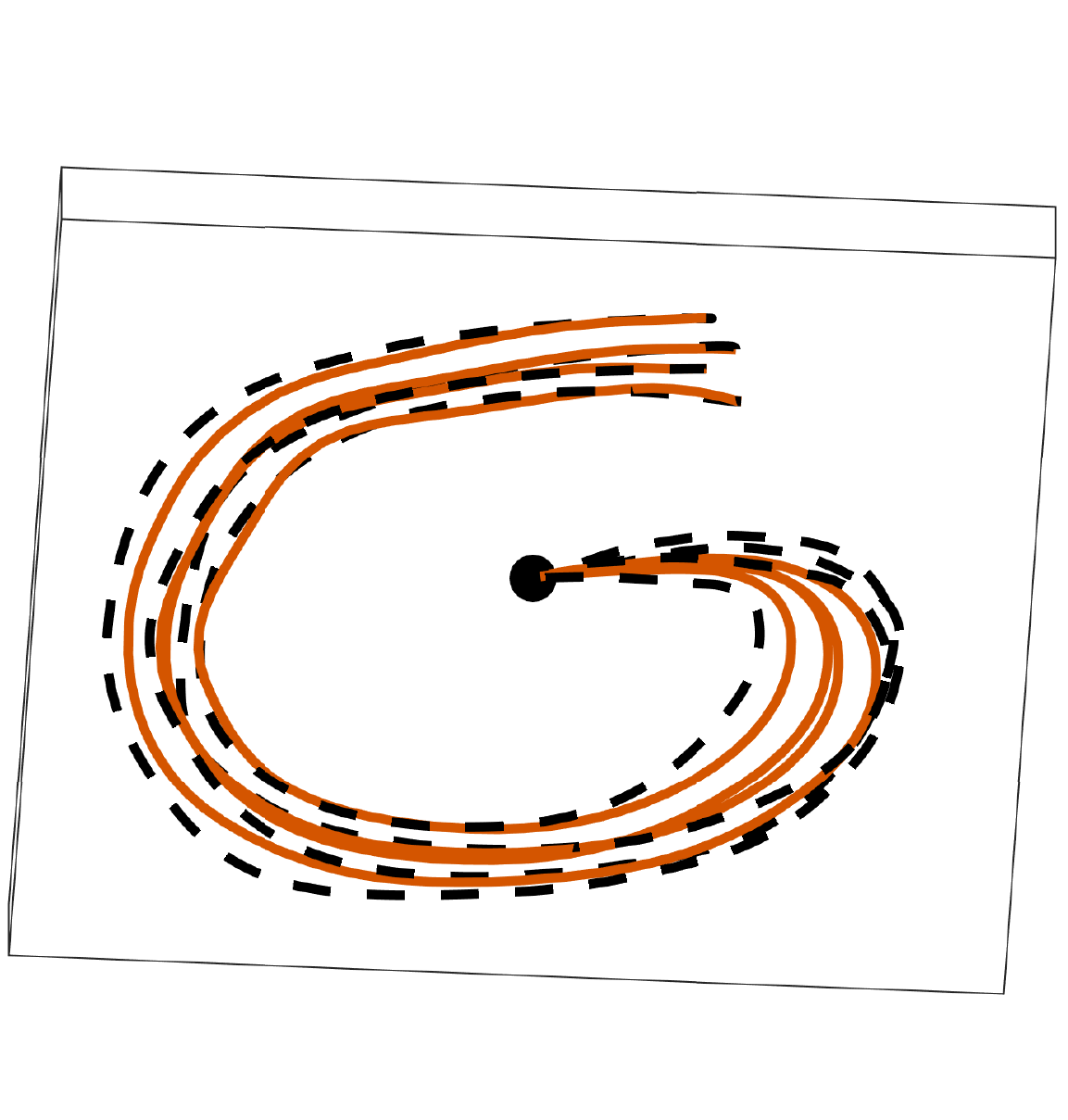}\hspace{-1mm}
    \includegraphics[trim={148 234 132 210},clip,width=0.12\textwidth]{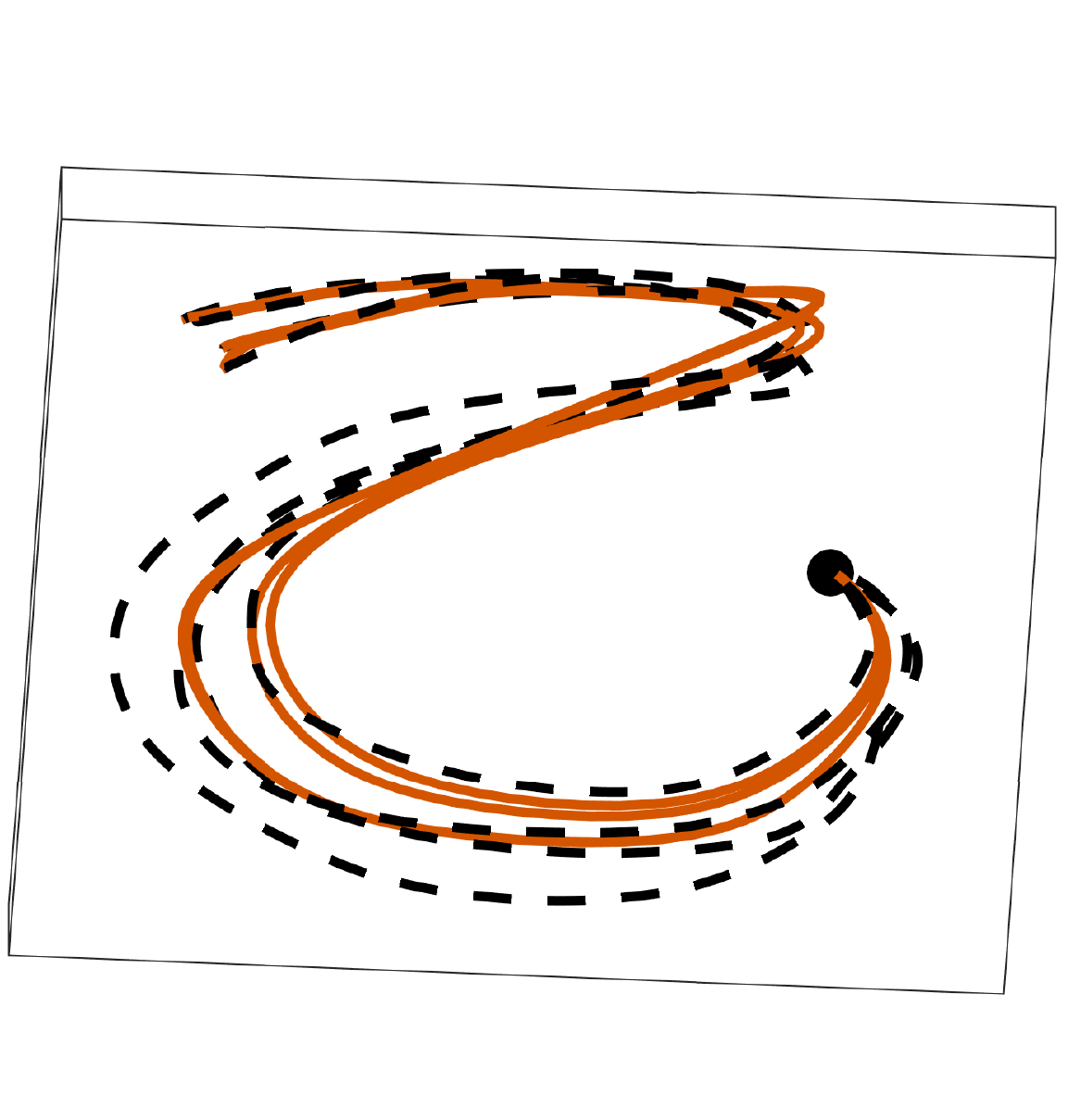}\hspace{-1mm}
    \includegraphics[trim={148 234 132 210},clip,width=0.12\textwidth]{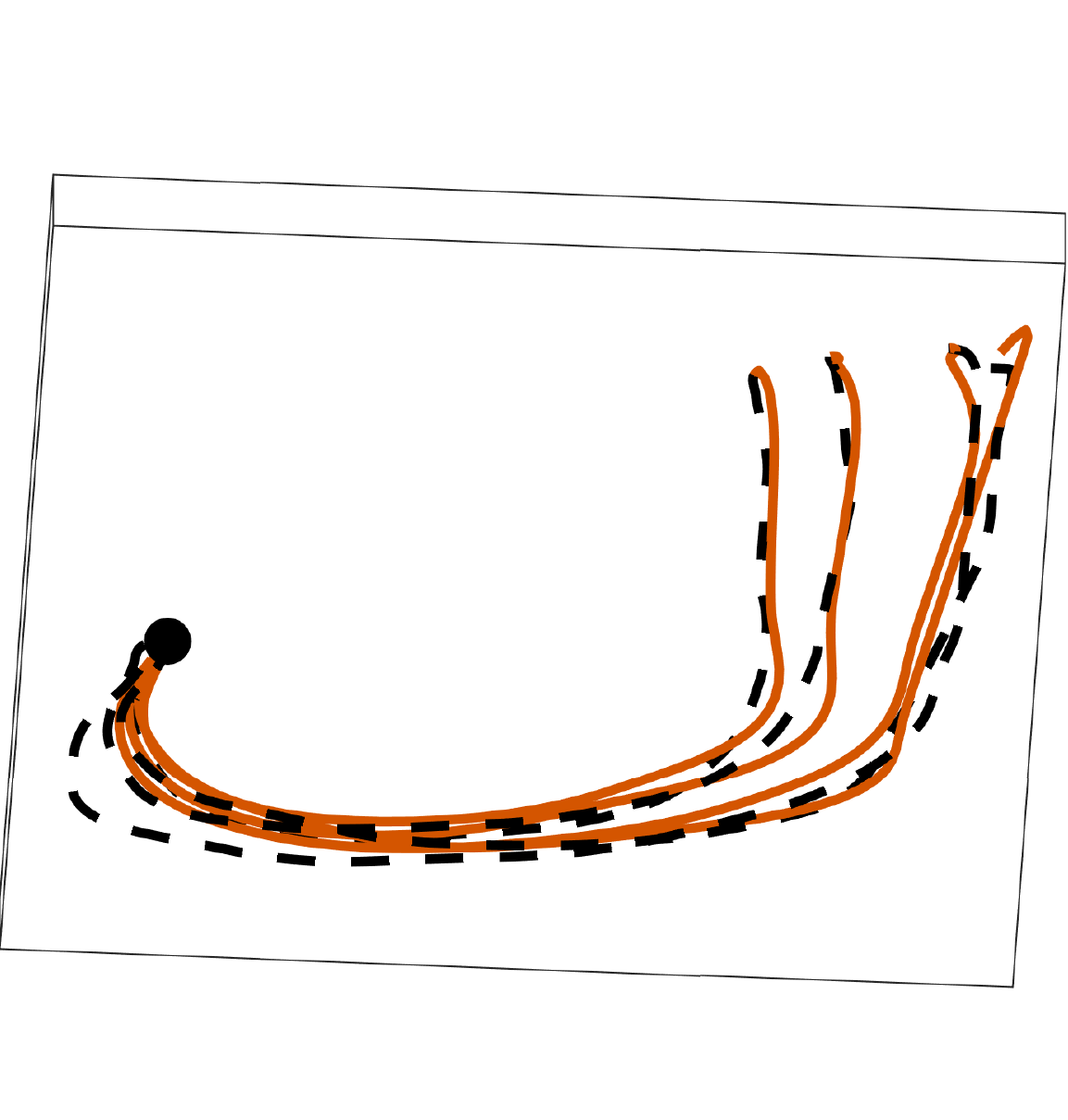}\hspace{-1mm}
    \includegraphics[trim={148 234 132 210},clip,width=0.12\textwidth]{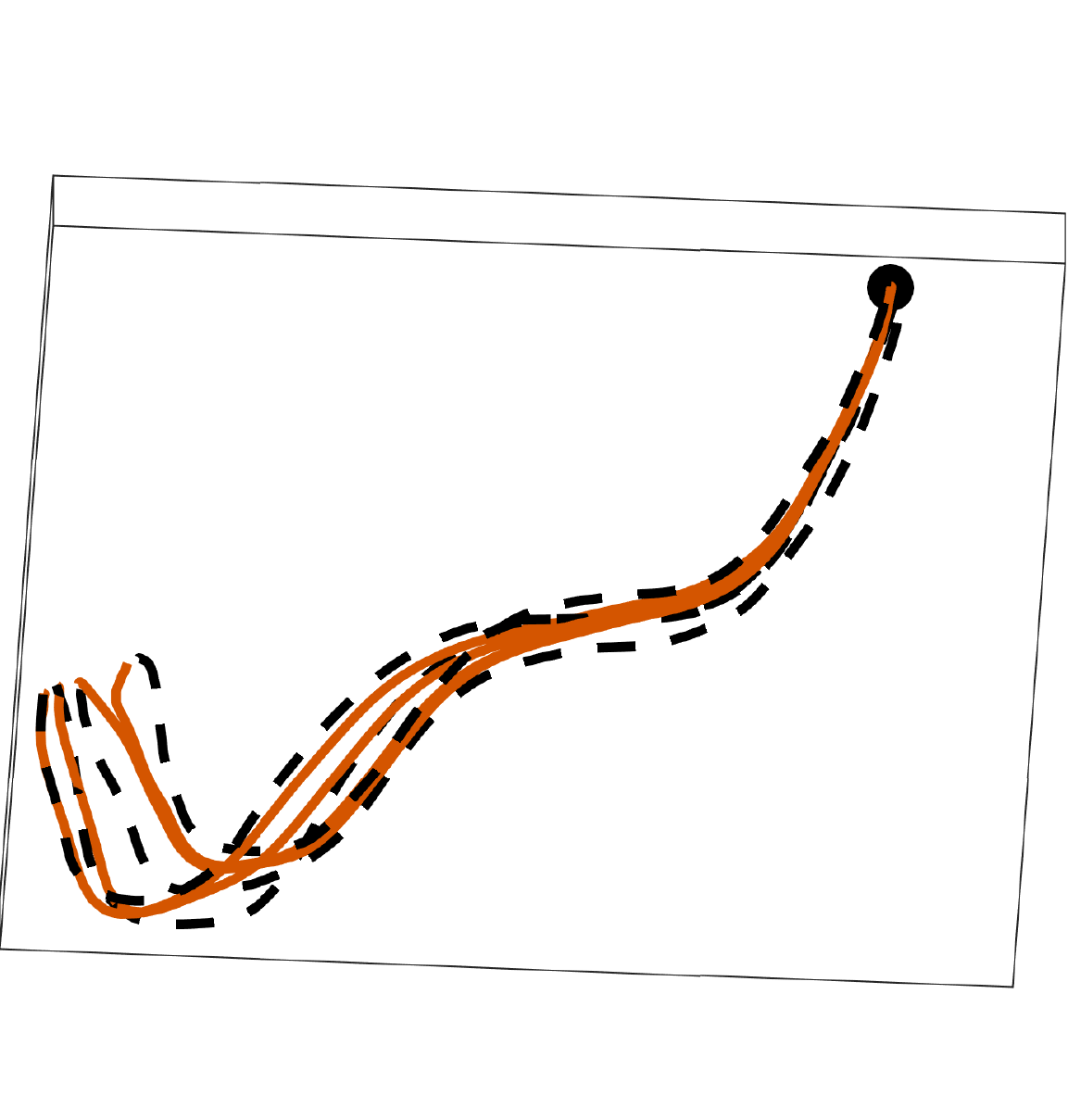}\hspace{-1mm}
    \includegraphics[trim={148 234 132 210},clip,width=0.12\textwidth]{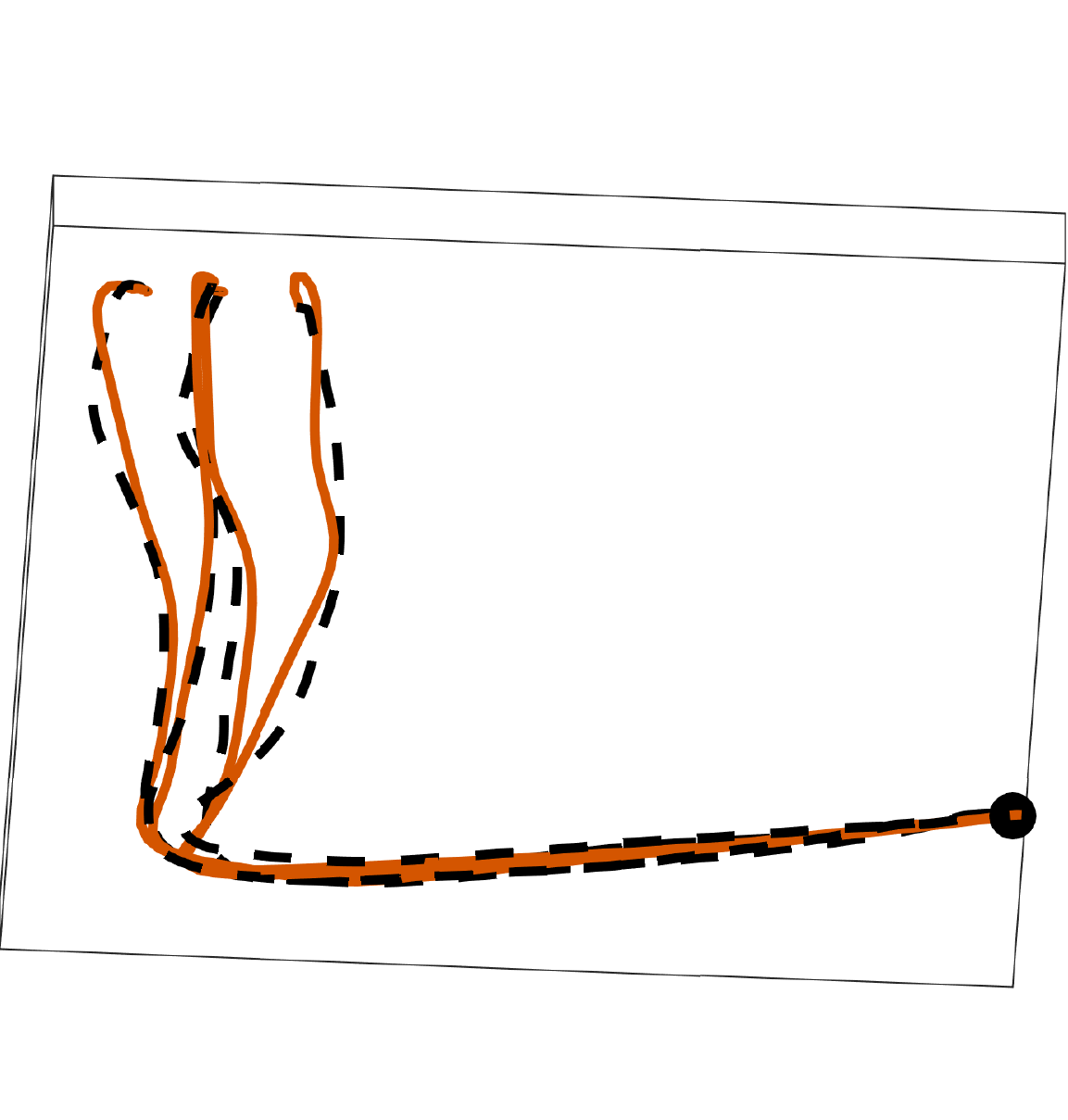}
    \vspace{-12pt}
    \includegraphics[trim={148 234 132 210},clip,width=0.12\textwidth]{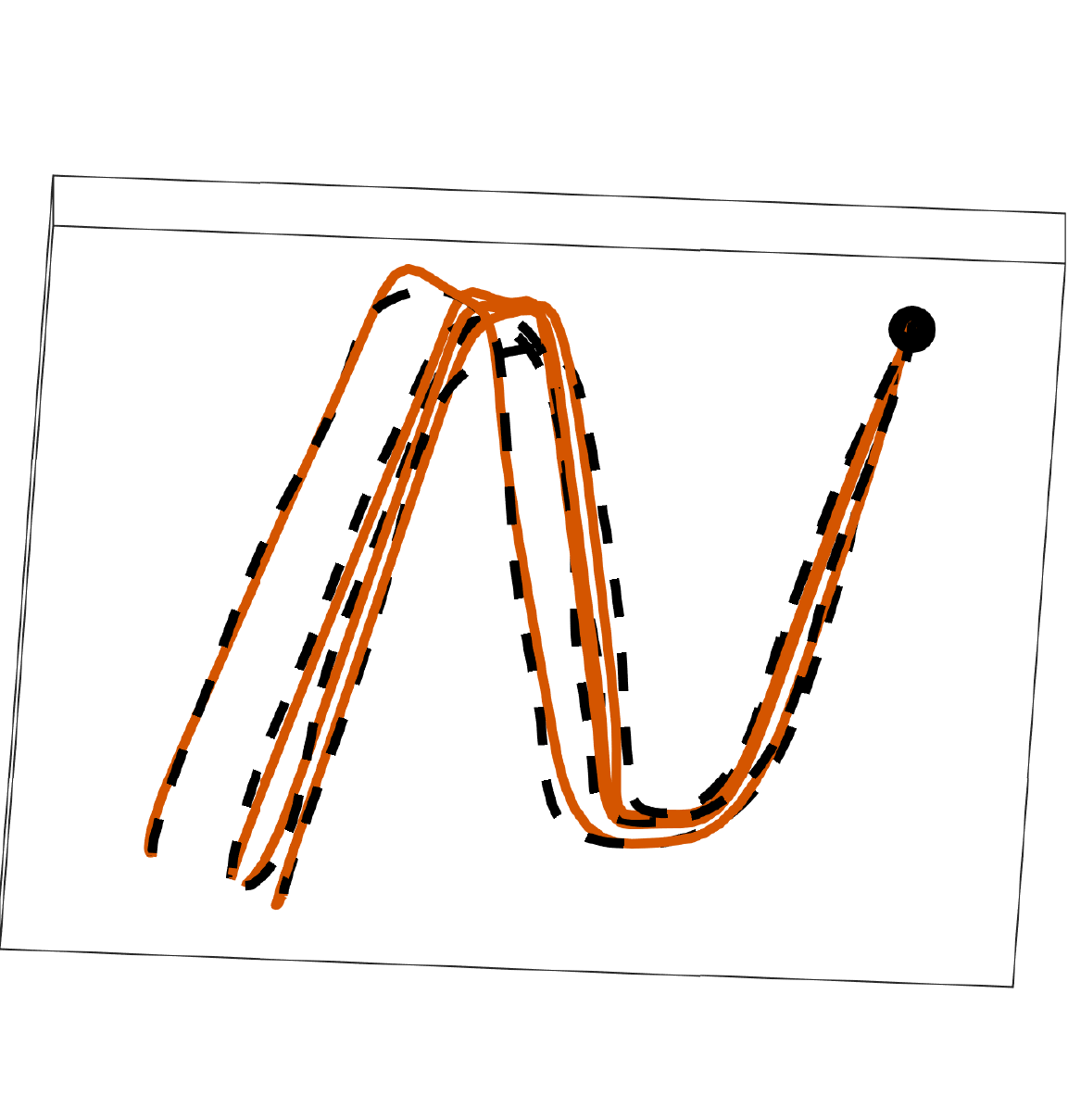}\hspace{-1mm}
    \includegraphics[trim={148 234 132 210},clip,width=0.12\textwidth]{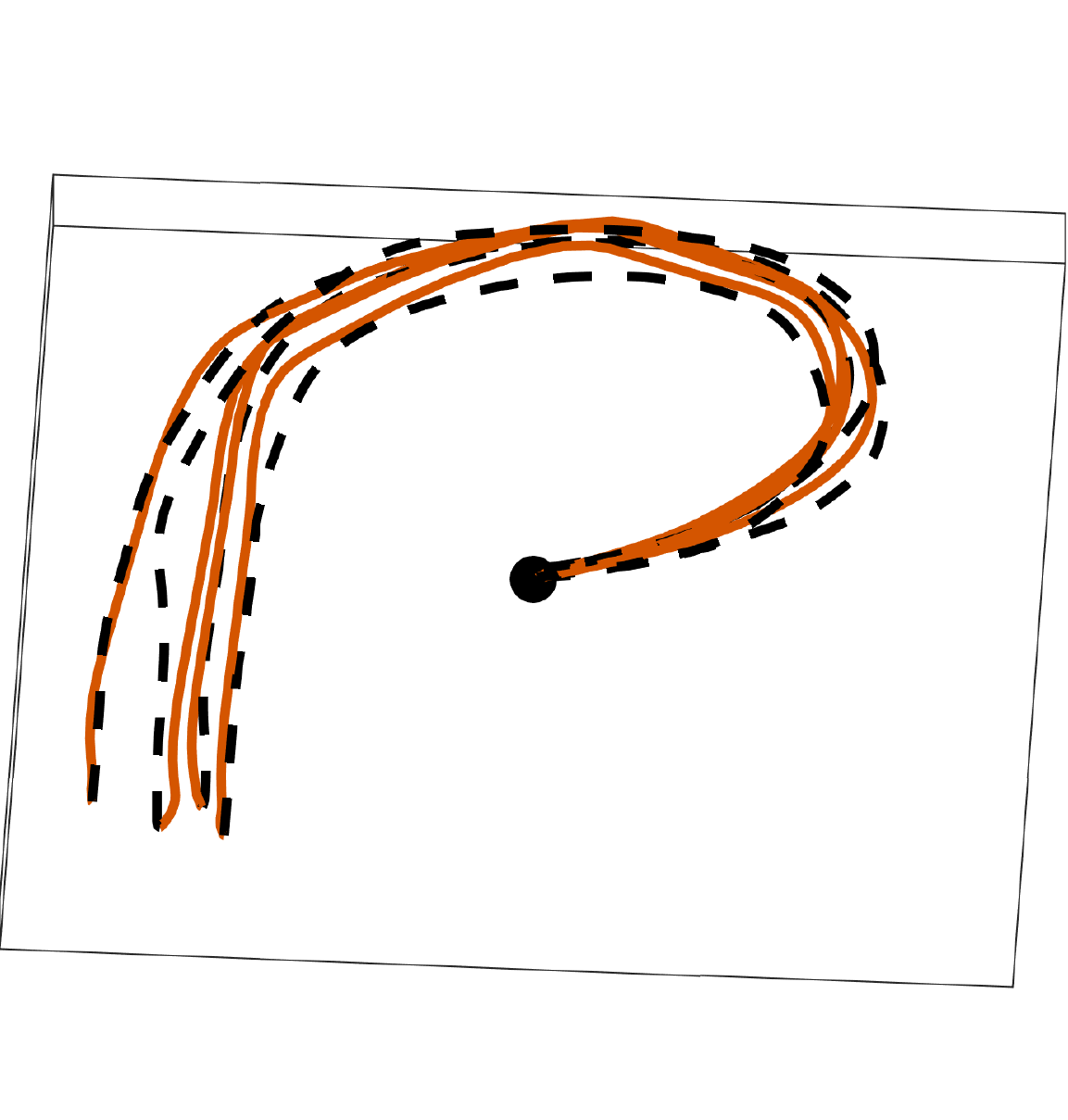}\hspace{-1mm}
    \includegraphics[trim={148 234 132 210},clip,width=0.12\textwidth]{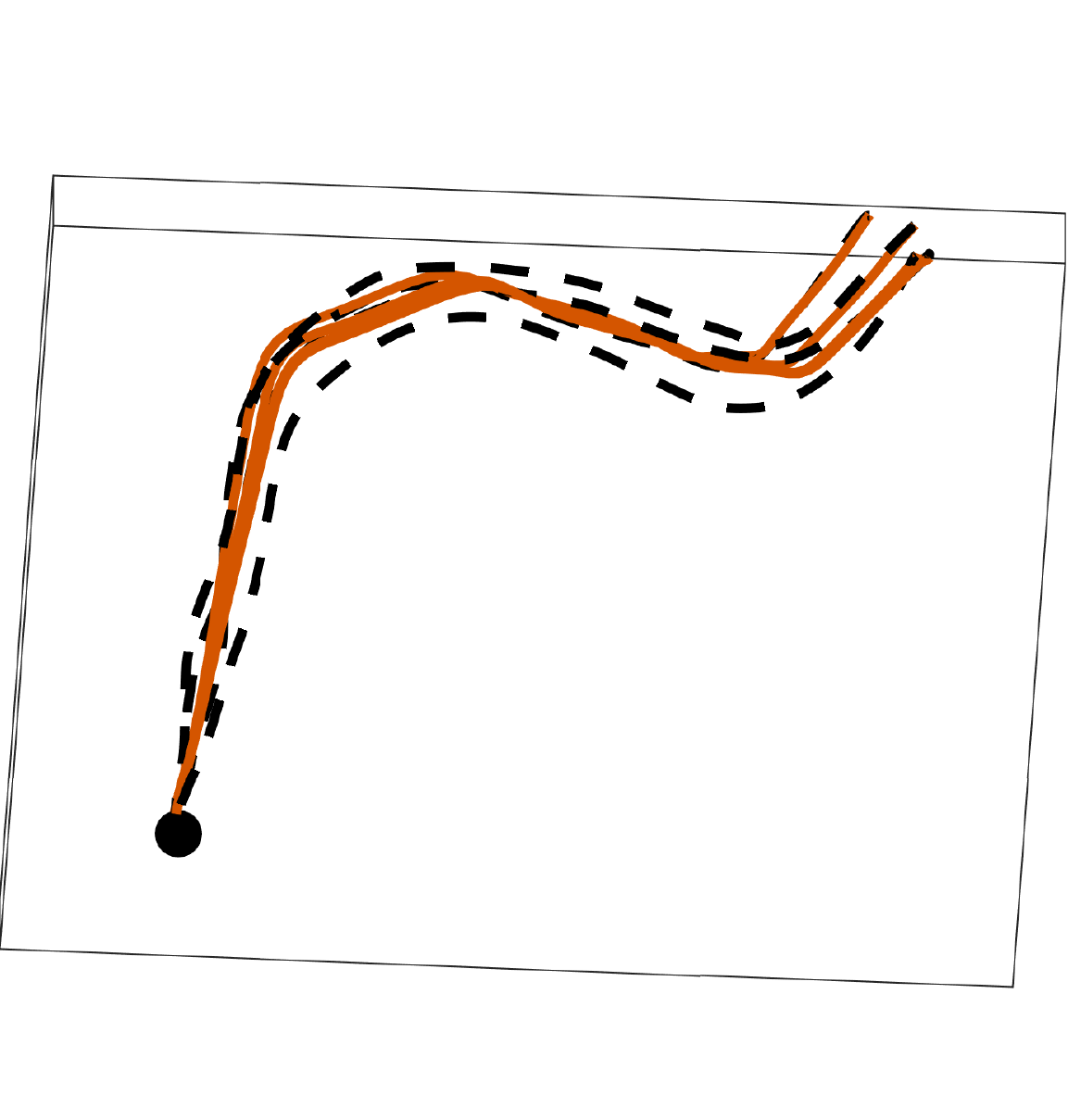}\hspace{-1mm}
    \includegraphics[trim={148 234 132 210},clip,width=0.12\textwidth]{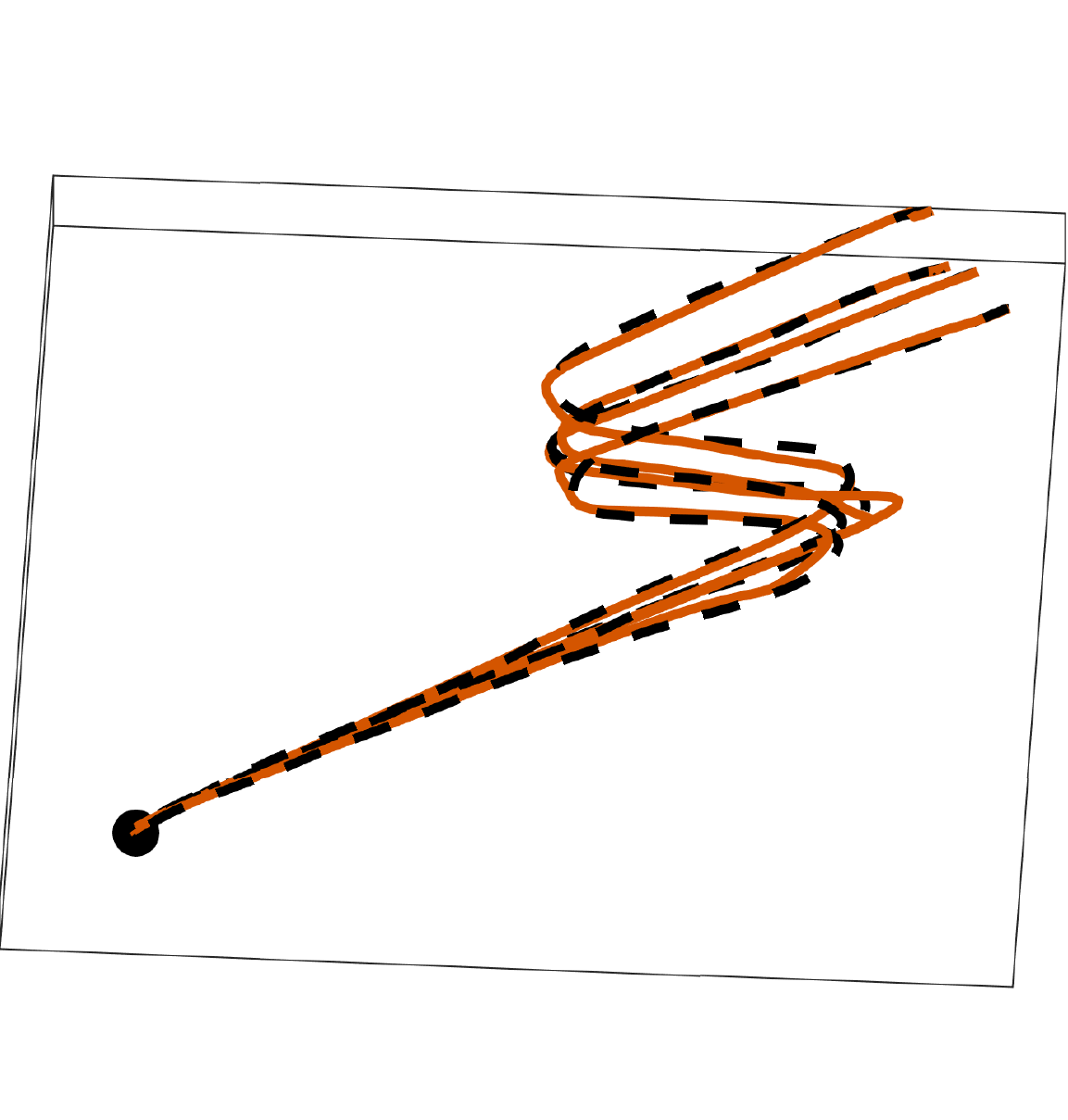}\hspace{-1mm}
    \includegraphics[trim={148 234 132 210},clip,width=0.12\textwidth]{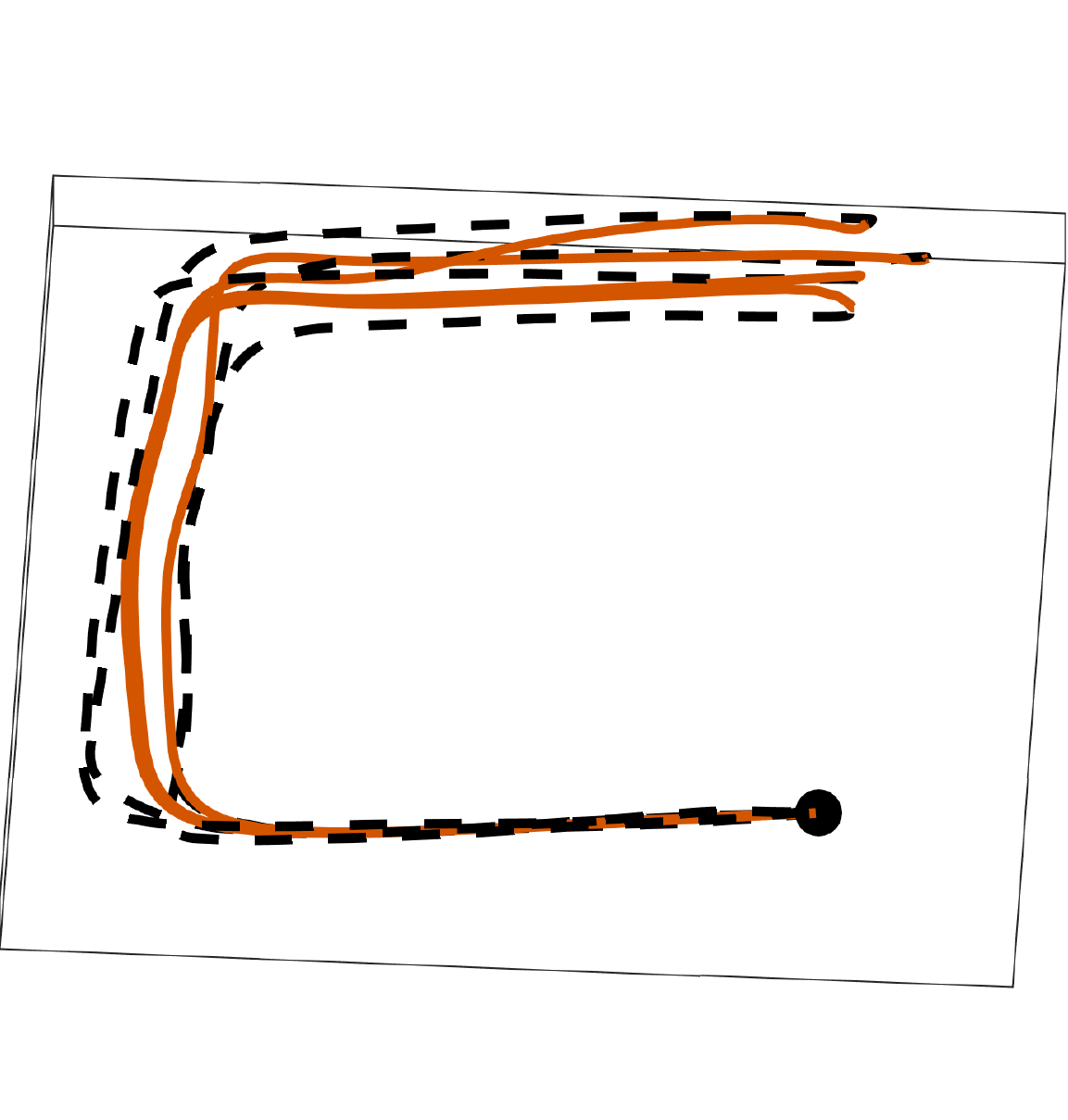}\hspace{-1mm}
    \includegraphics[trim={148 234 132 210},clip,width=0.12\textwidth]{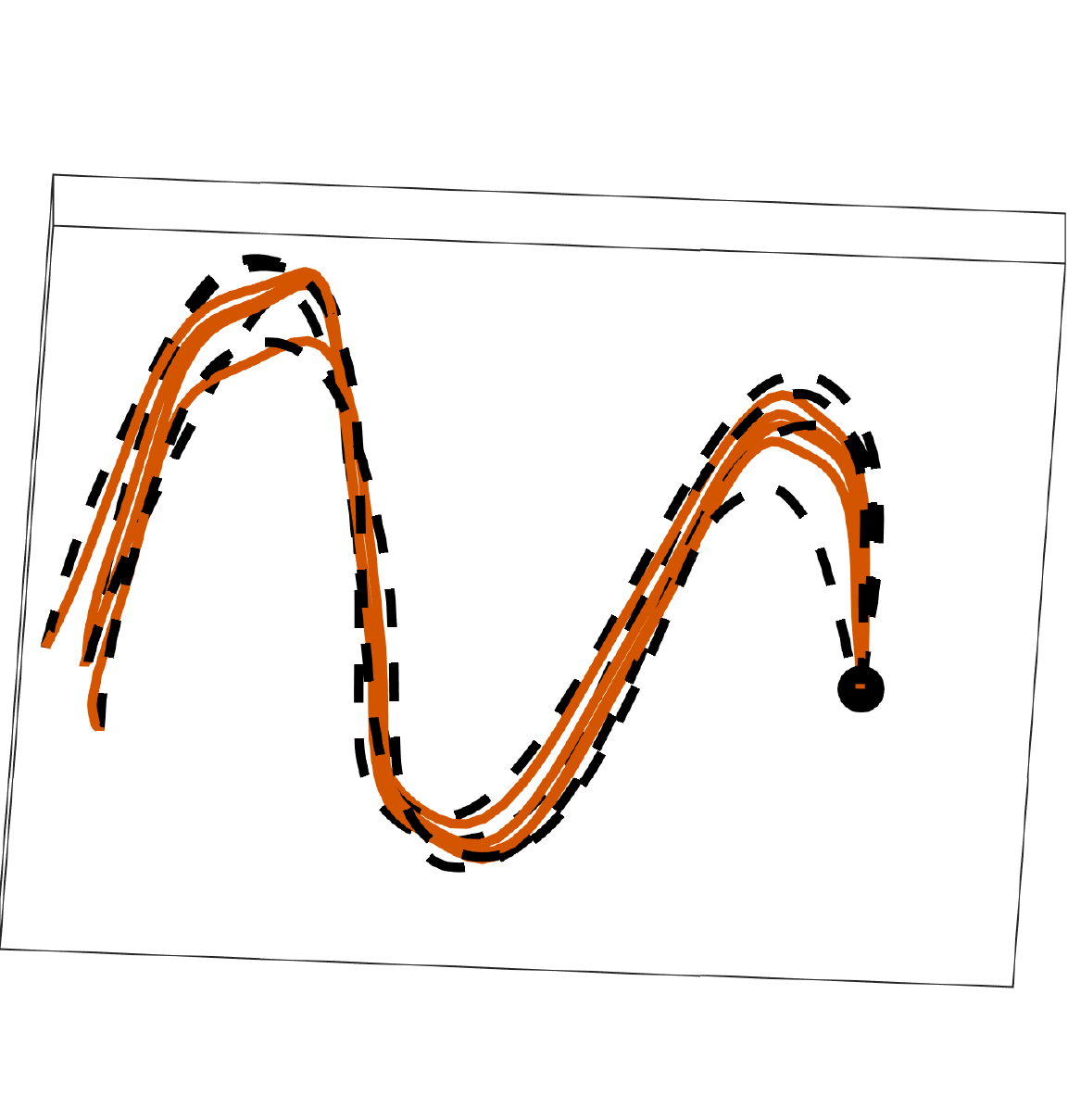}\hspace{-1mm}
    \includegraphics[trim={148 234 132 210},clip,width=0.12\textwidth]{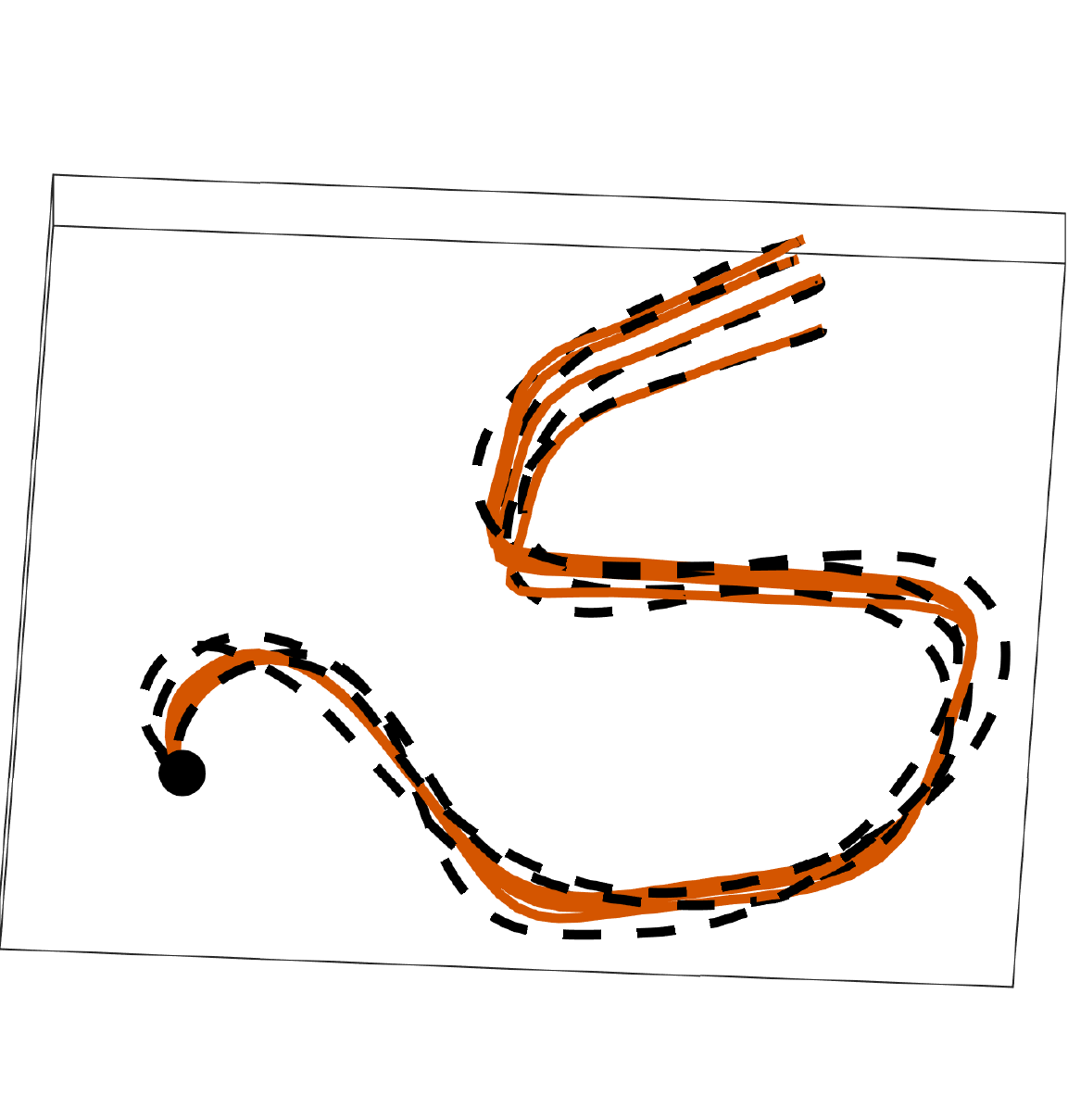}\hspace{-1mm}
    \includegraphics[trim={148 234 132 210},clip,width=0.12\textwidth]{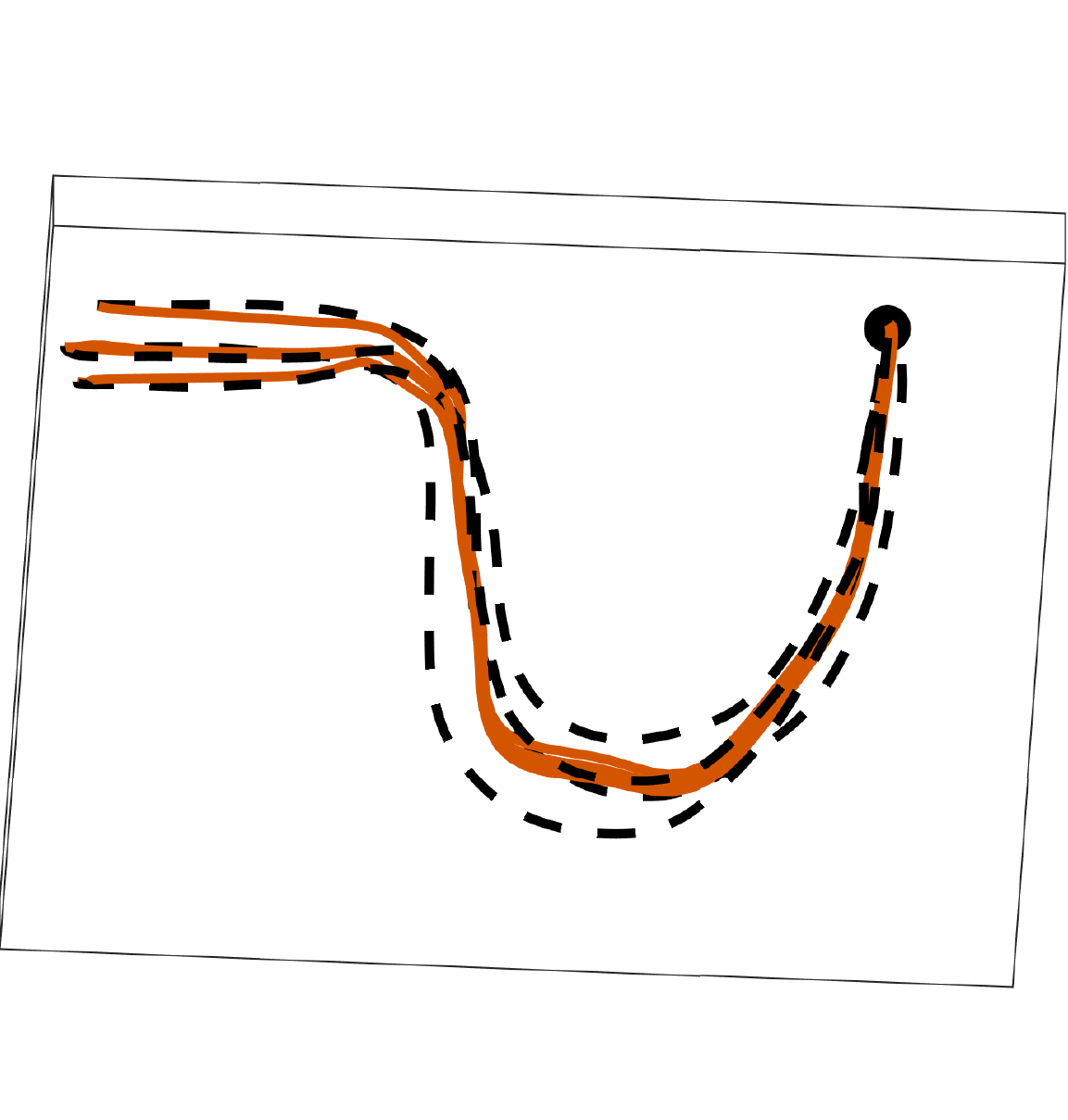}
    \vspace{-12pt}
    \includegraphics[trim={148 234 132 210},clip,width=0.12\textwidth]{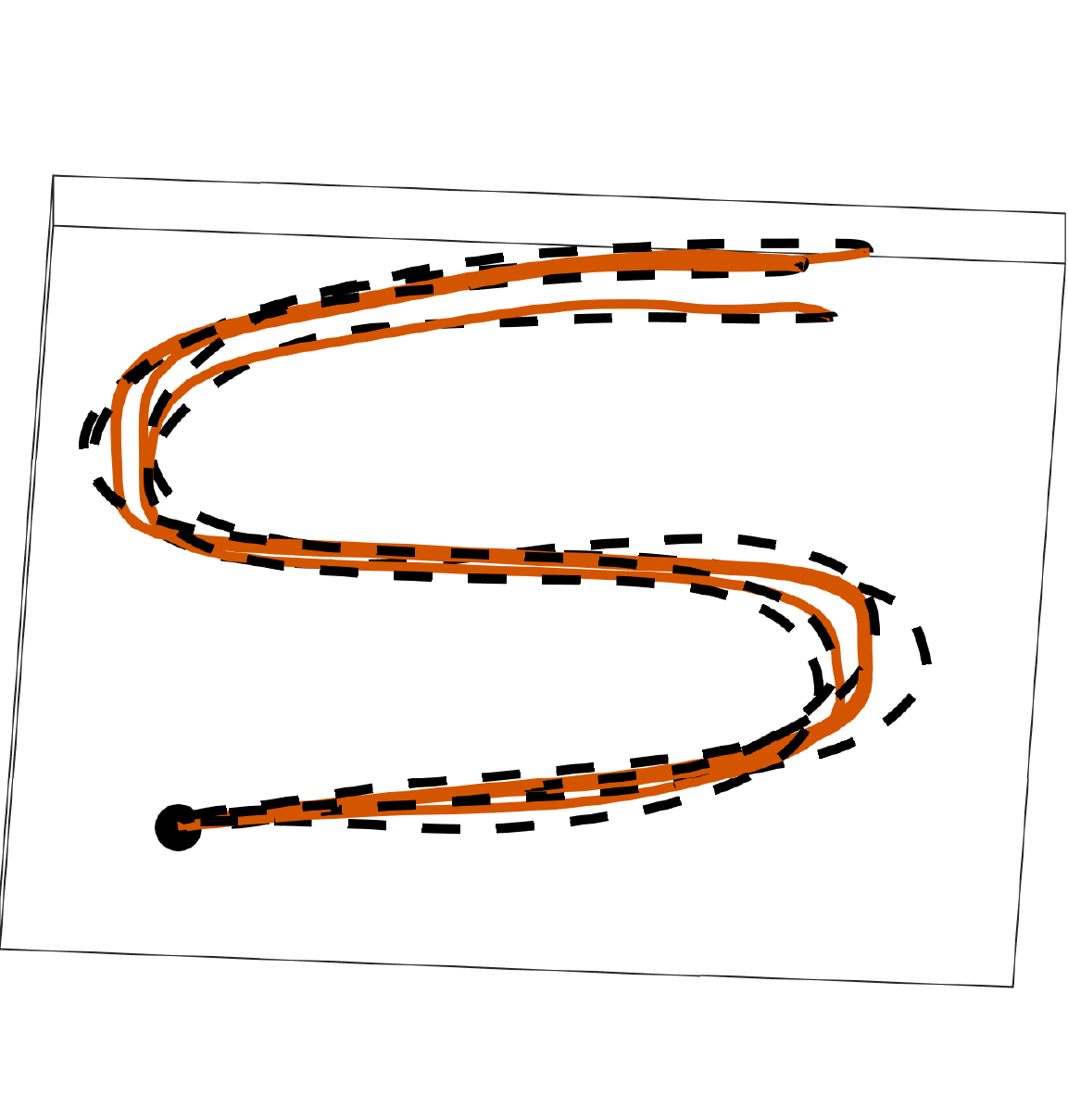}\hspace{-1mm}
    \includegraphics[trim={148 234 132 210},clip,width=0.12\textwidth]{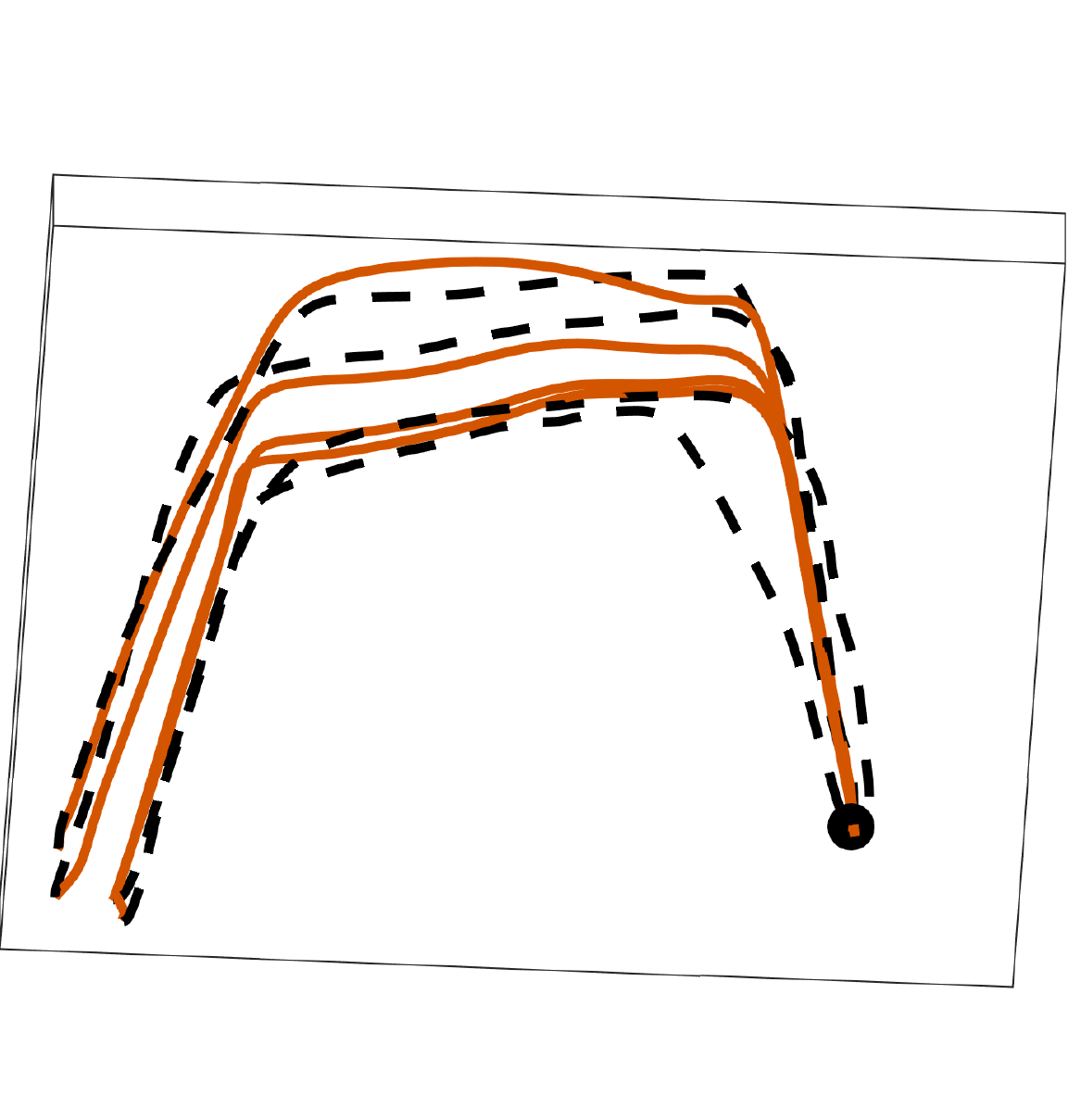}\hspace{-1mm}
    \includegraphics[trim={148 234 132 210},clip,width=0.12\textwidth]{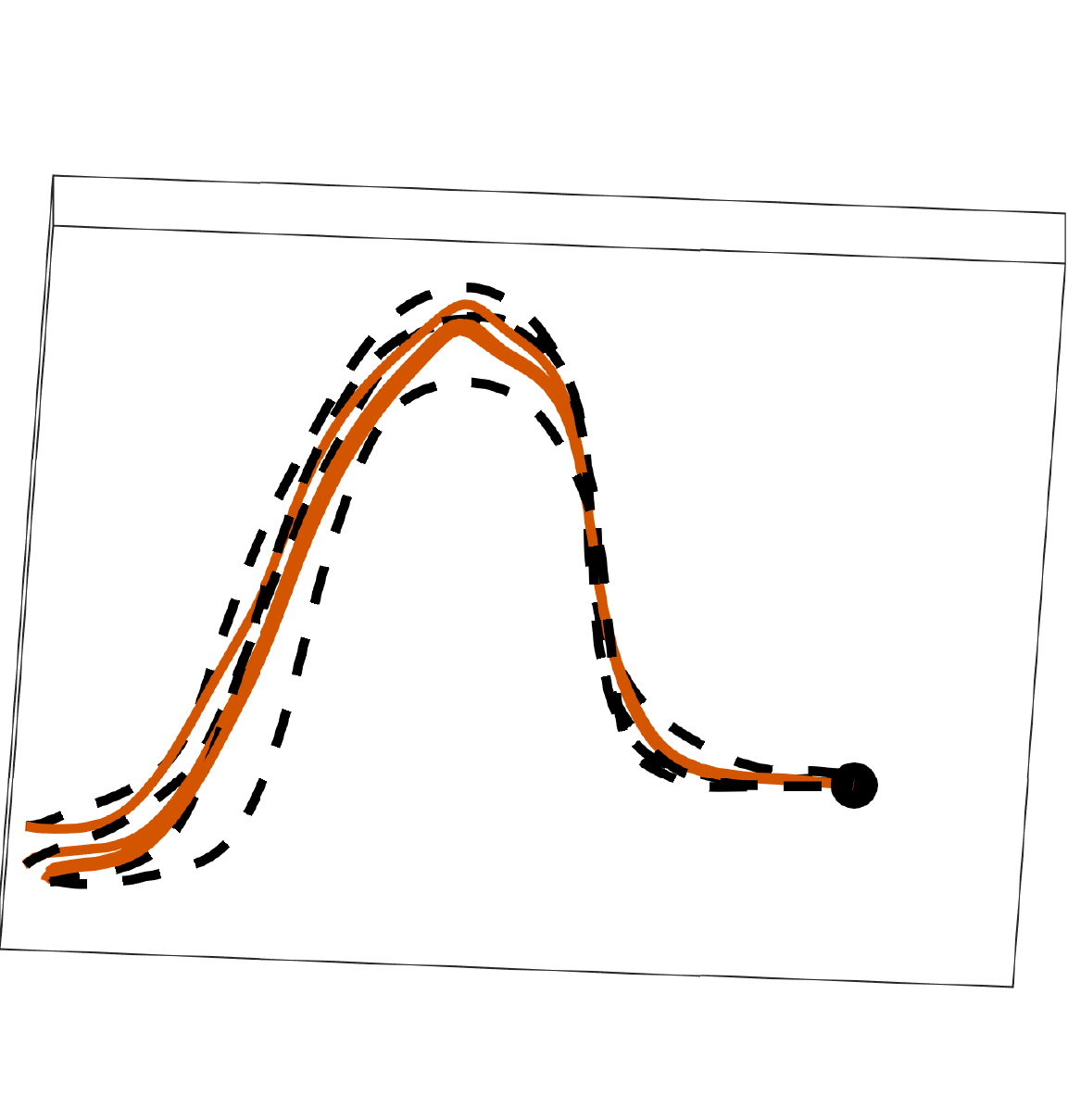}\hspace{-1mm}
    \includegraphics[trim={148 234 132 210},clip,width=0.12\textwidth]{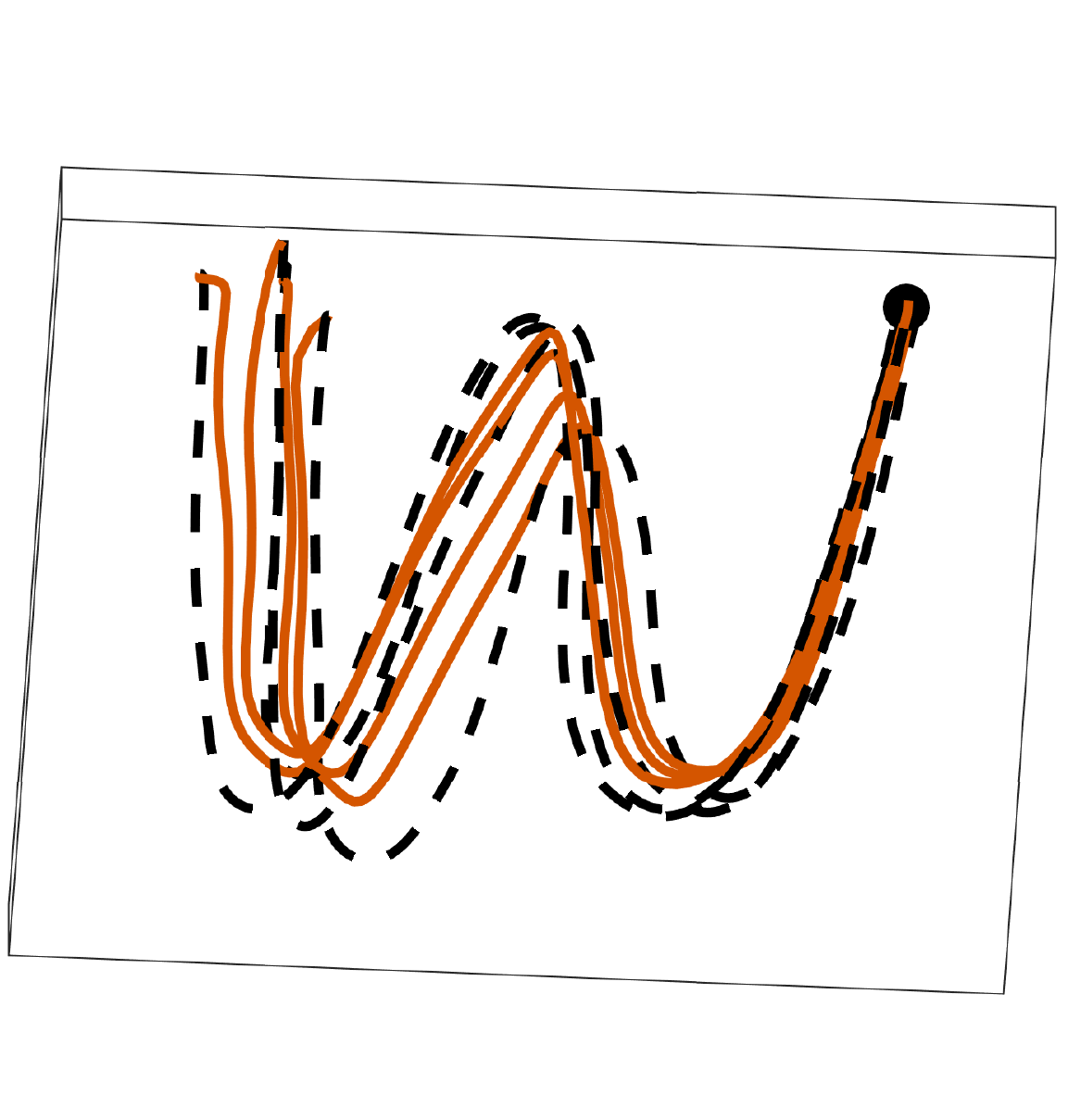}\hspace{-1mm}
    \includegraphics[trim={148 234 132 210},clip,width=0.12\textwidth]{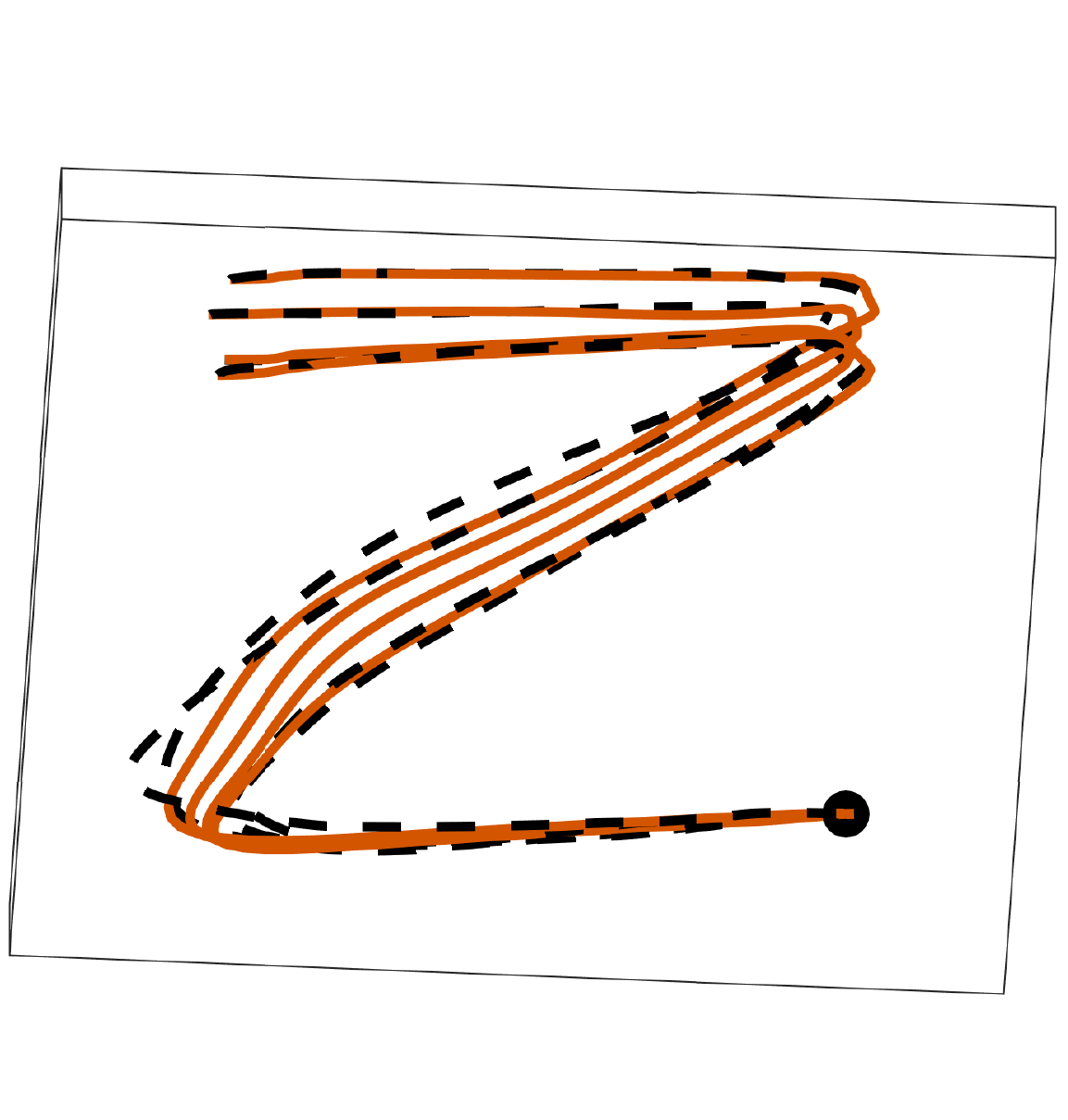}\hspace{-1mm}
    \includegraphics[trim={148 234 132 210},clip,width=0.12\textwidth]{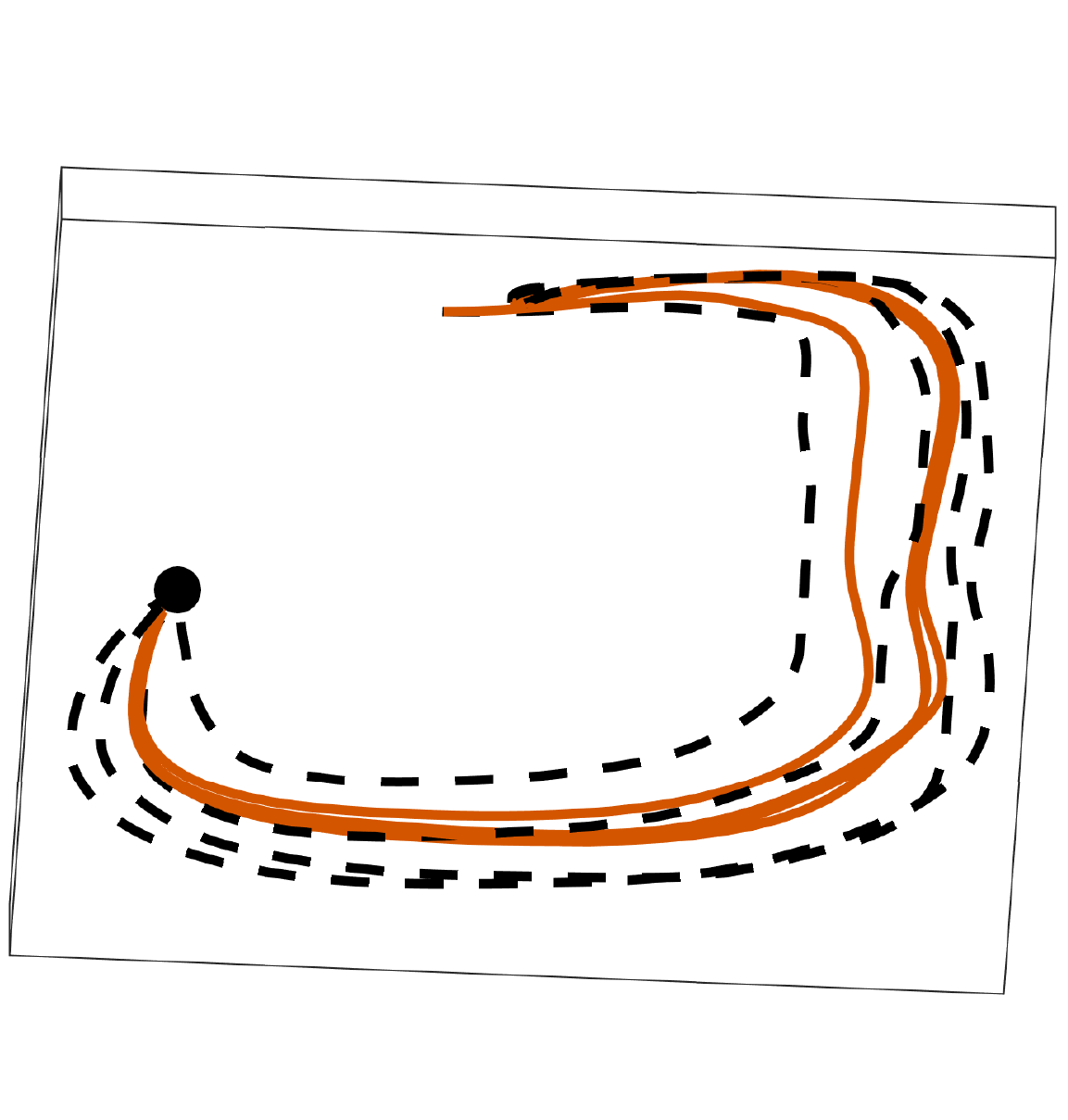}\hspace{-1mm}
    \includegraphics[trim={148 234 132 210},clip,width=0.12\textwidth]{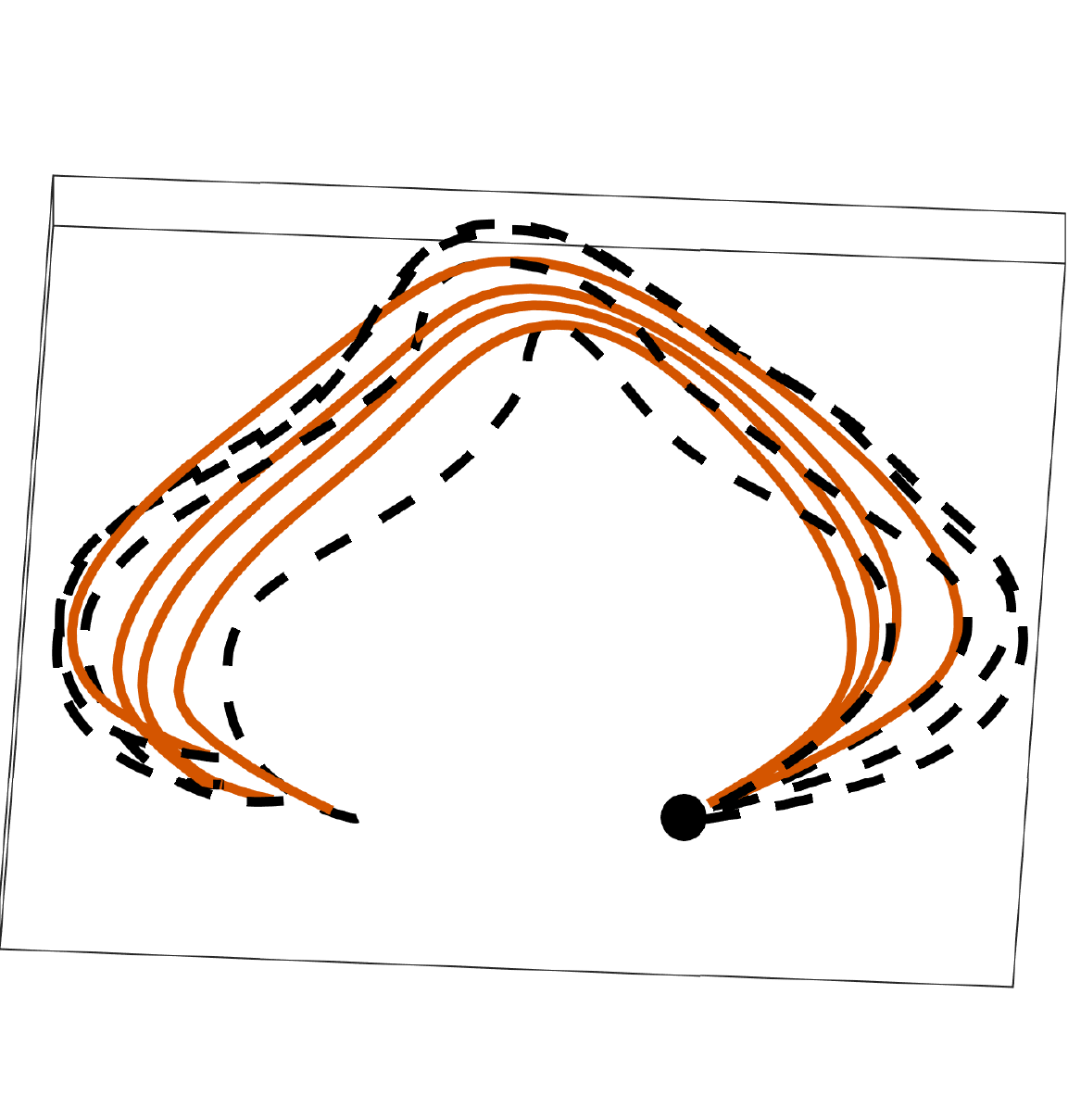}\hspace{-1mm}
    \includegraphics[trim={148 234 132 210},clip,width=0.12\textwidth]{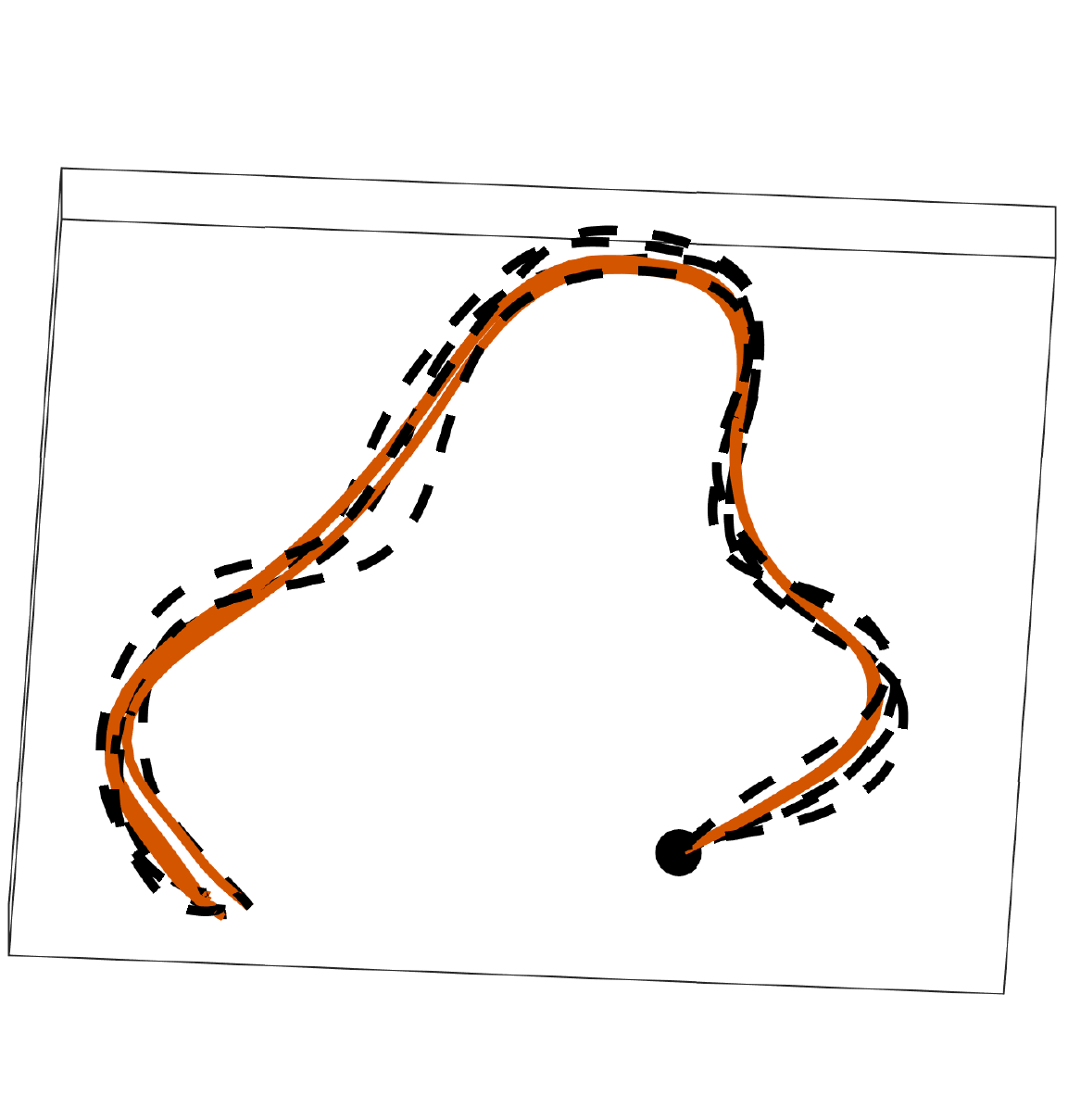}
    \includegraphics[trim={148 234 132 210},clip,width=0.12\textwidth]{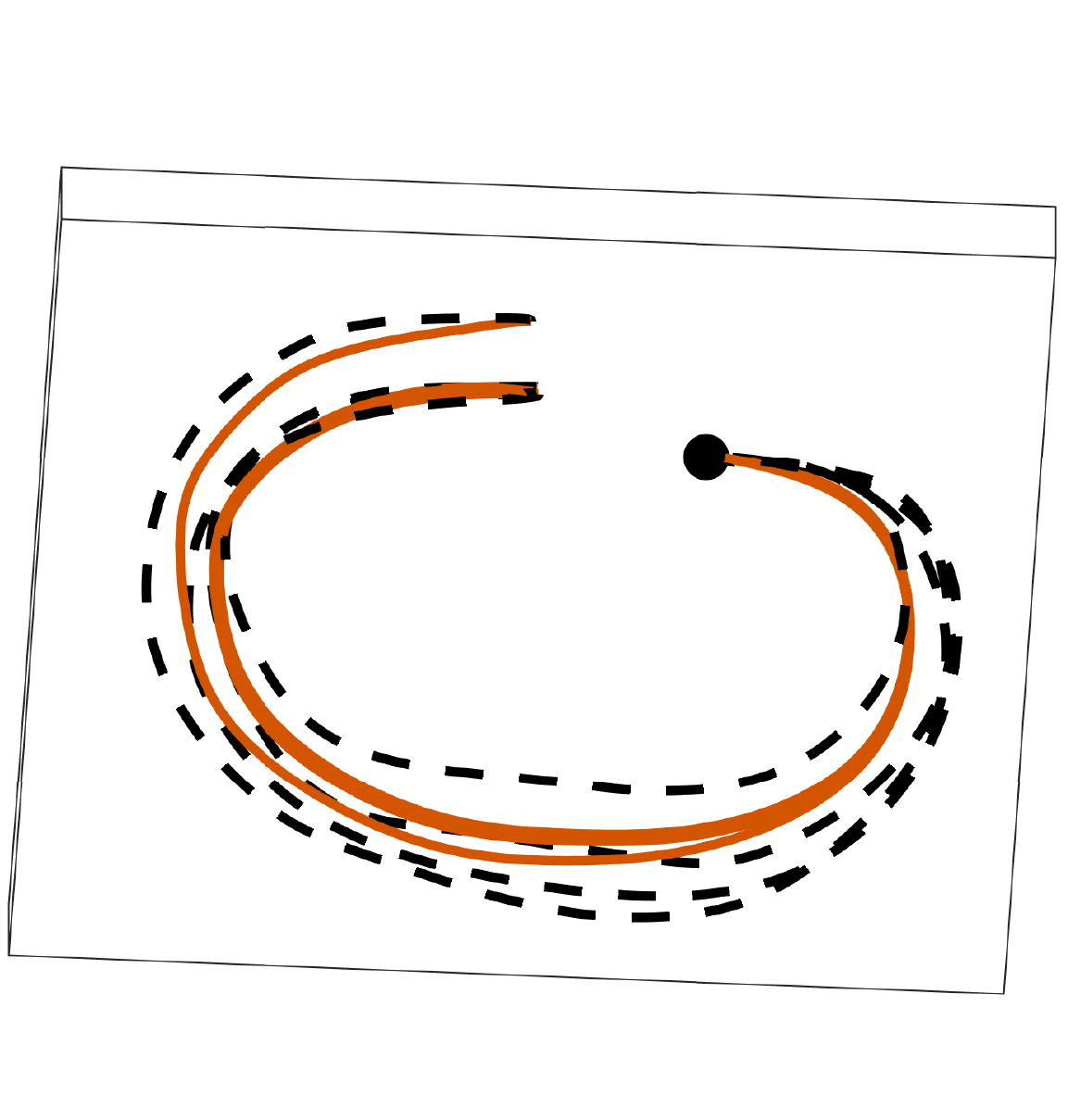}\hspace{-1mm}
    \includegraphics[trim={148 234 132 210},clip,width=0.12\textwidth]{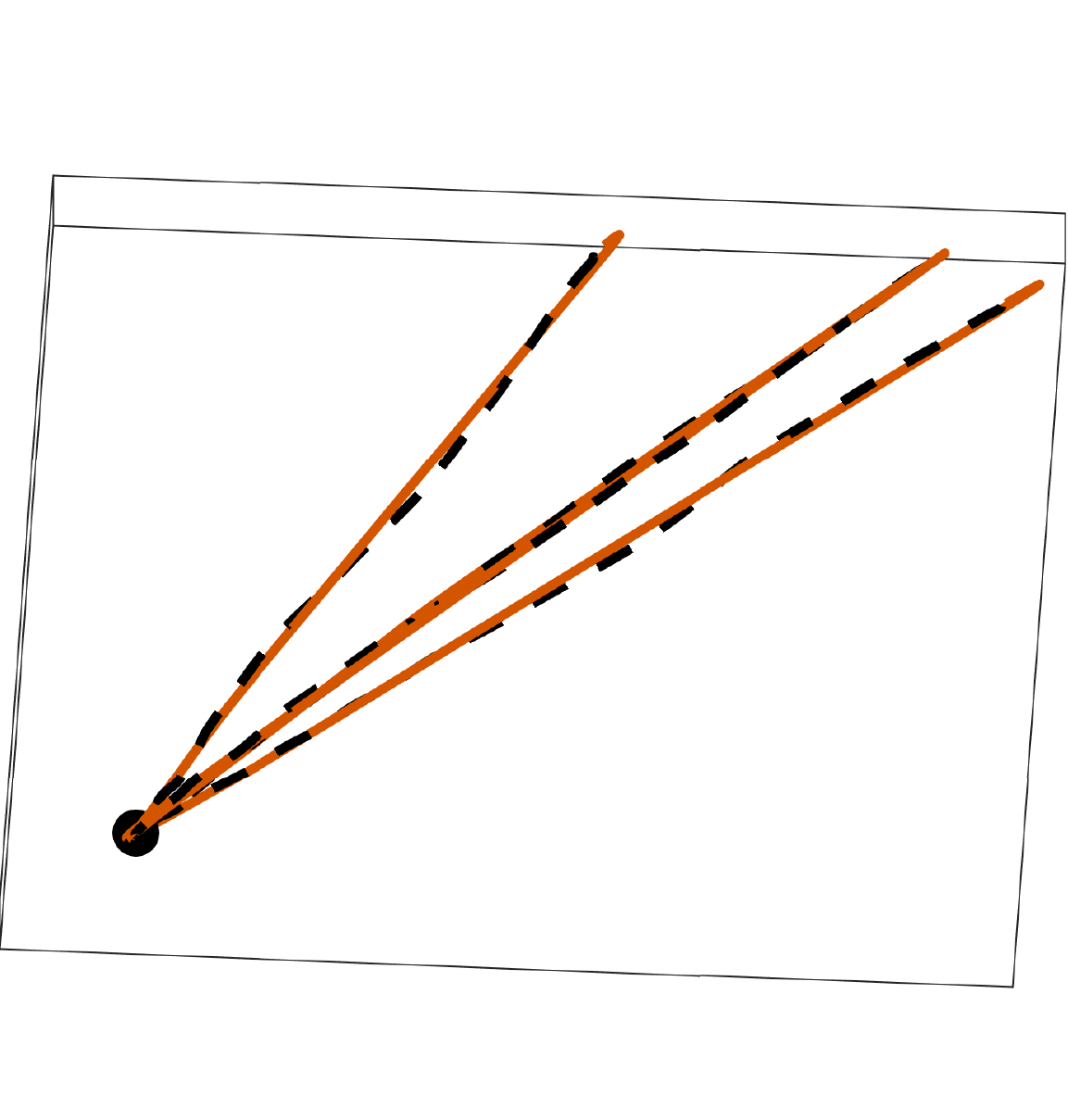}\hspace{-1mm}
    \includegraphics[trim={148 234 132 210},clip,width=0.12\textwidth]{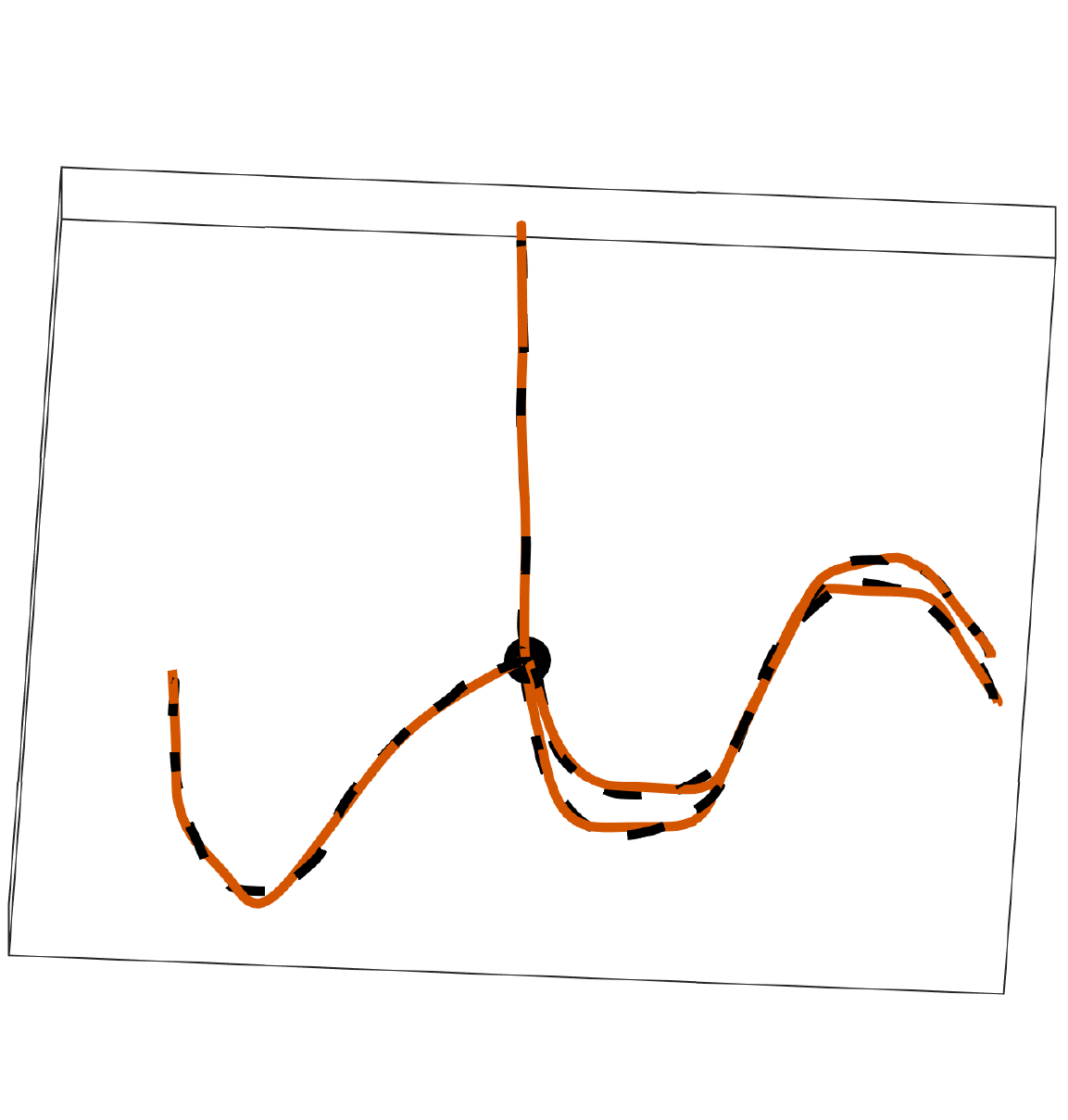}\hspace{-1mm}
    \includegraphics[trim={148 234 132 210},clip,width=0.12\textwidth]{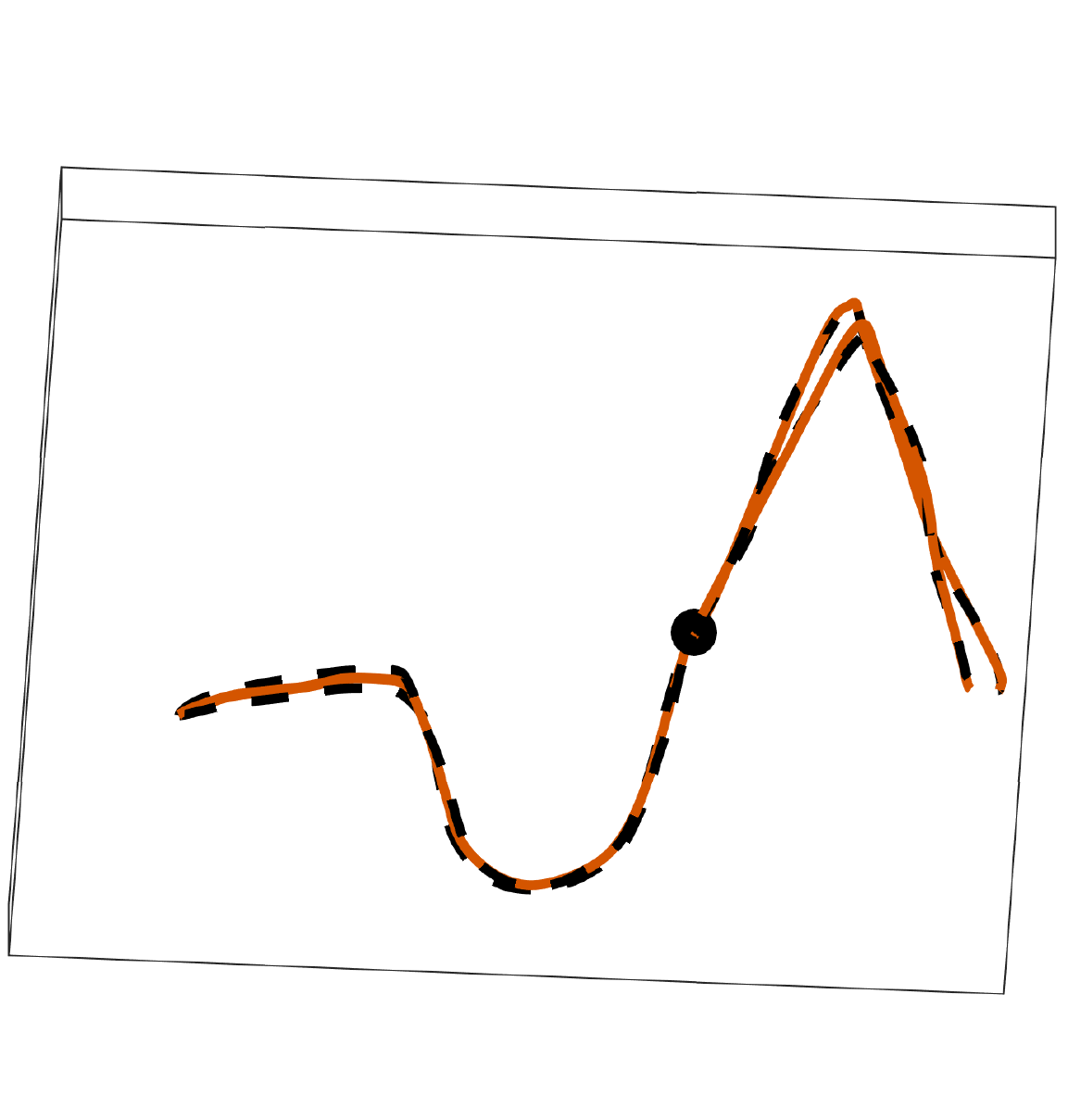}\hspace{-1mm}
    \includegraphics[trim={148 234 132 210},clip,width=0.12\textwidth]{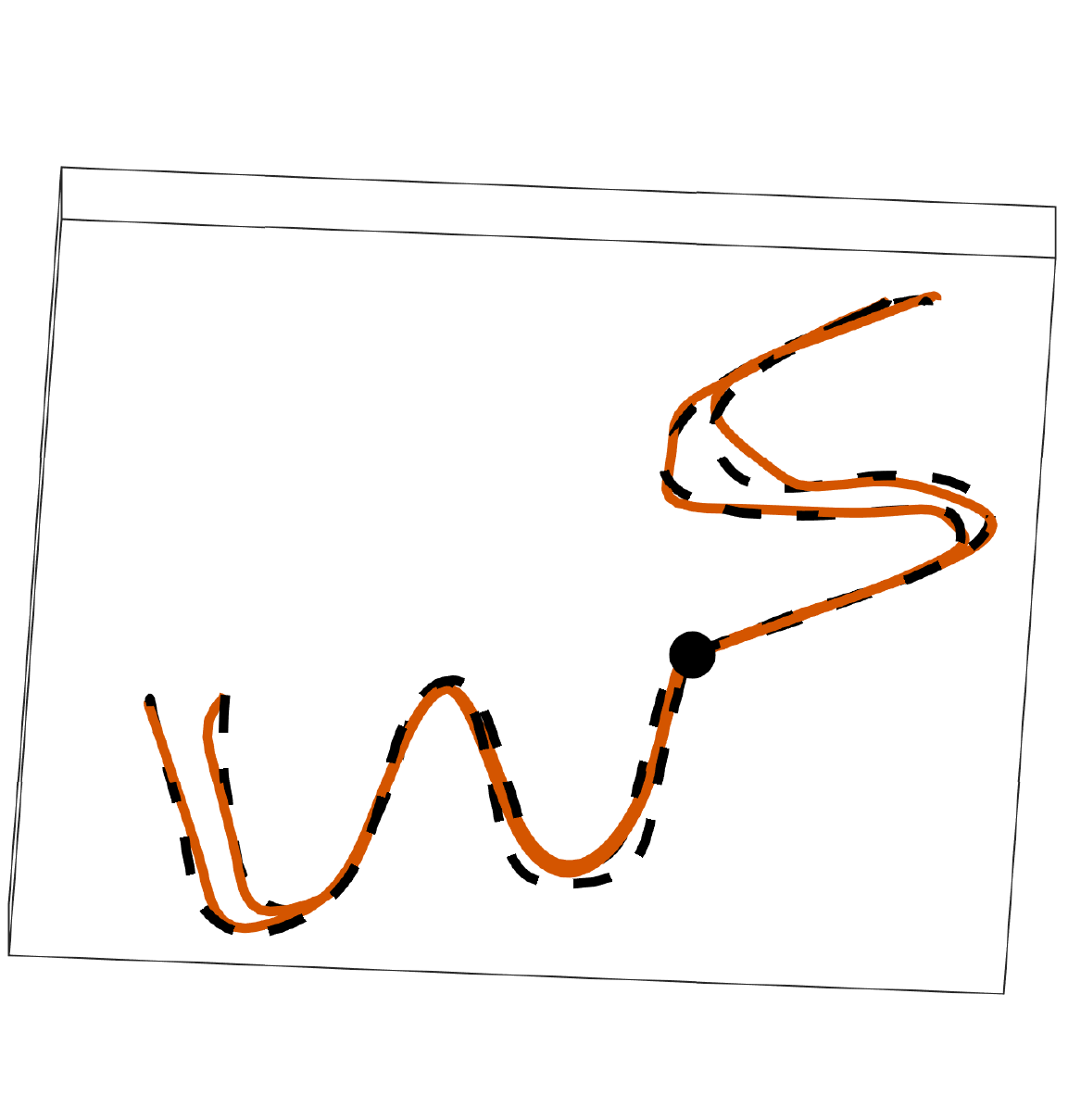}\hspace{-1mm}
    \includegraphics[trim={148 234 132 210},clip,width=0.12\textwidth]{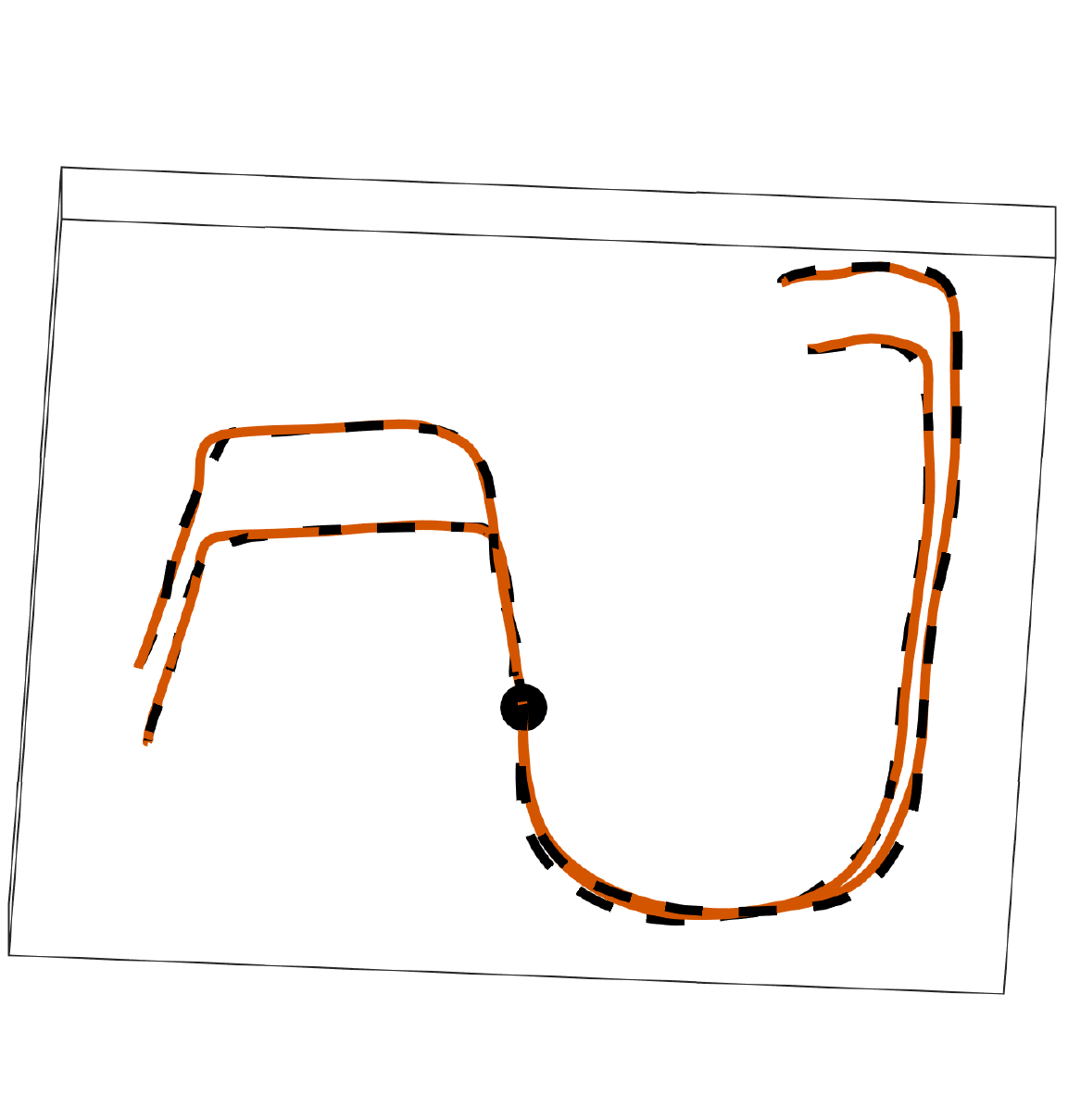}
    \caption{Qualitative results obtained on the Riemannian LASA dataset. Reproduced trajectories (brown solid lines) are obtained by applying the diffeomorphism learned with \ac{ours} on the \ac{ts} demonstrations (black dashed lines). {It is worth noticing that \ac{ts} data are in $3$D, but we choose a view angle that makes the plot similar to the original $2$D data.}}
     \label{fig:dataset_all}
\end{figure*}
\begin{figure}[t]
	\centering
	\def\svgwidth{\columnwidth}{\fontsize{8}{8}
		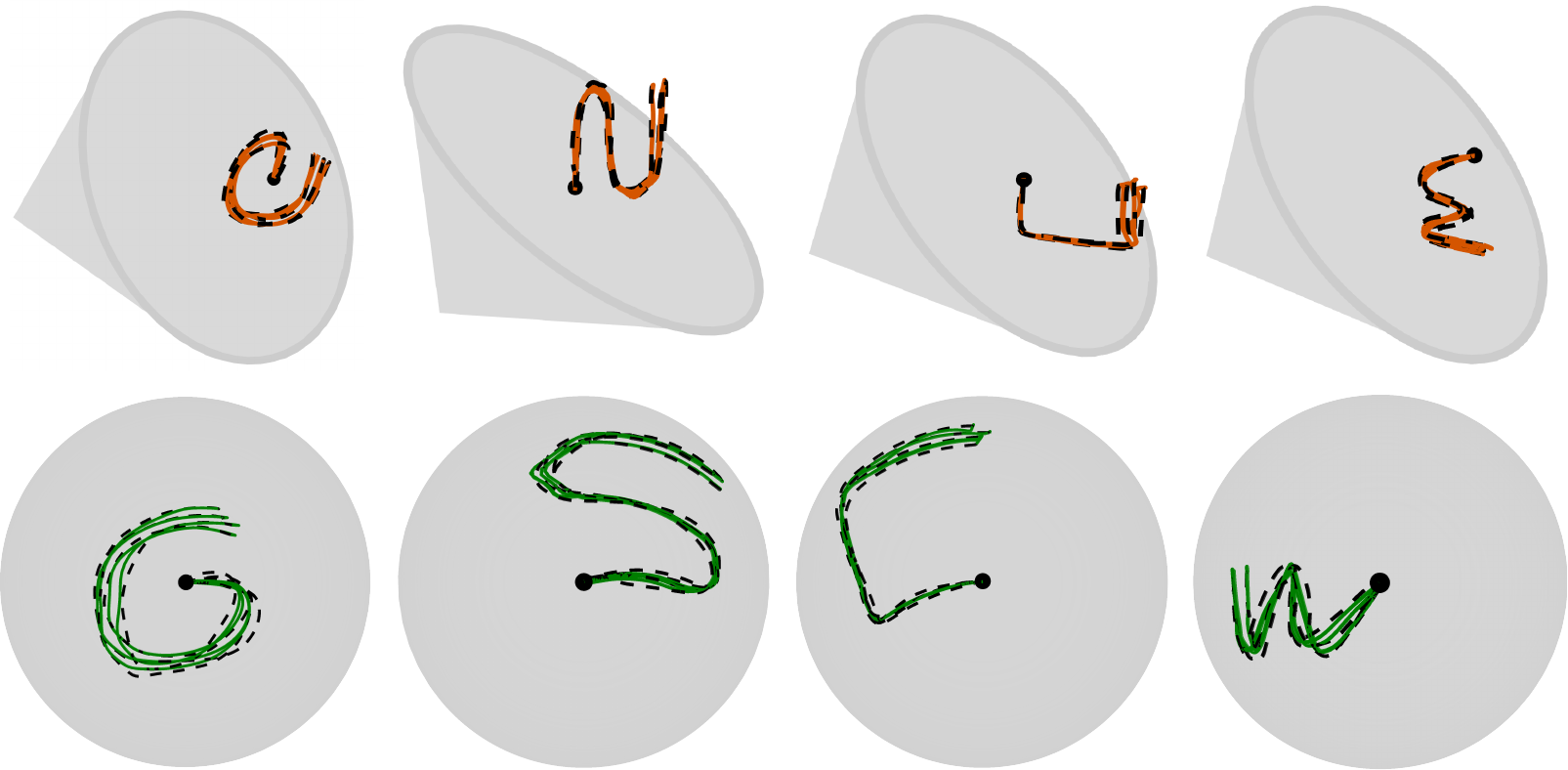}
	\caption{Four motion classes of the Riemannian LASA dataset. (Top) Demonstrations (black dashed lines) and trajectories generated by \ac{ours} (brown solid lines) evolving on the \ac{spd} cone. (Bottom) Demonstrations (black dashed lines) and trajectories generated by \ac{ours} (green solid lines) evolving on the \ac{uq} sphere.}
	\label{fig:fig_letters_manifold}
\end{figure}
In this section, we validate the proposed approach on a public benchmark---properly modified to represent trajectories evolving on \ac{uq} and \ac{spd} manifolds---and compare the results against \ac{fdm}~\cite{perrin2016fast}{, \ac{rgmm}~\cite{calinon20gaussians}, and \ac{eflow}~\cite{rana2020euclidean}. It is worth mentioning that \ac{fdm} has not been designed to work on Riemannian manifolds. However, the procedure described in \secref{sec:proposed} allows to exploit different approaches to fit a diffeomorphism between \acp{ts}.}

\subsection{The Riemannian LASA dataset}
 
In the \ac{lfd} literature, there is no available dataset to test \ac{ds}-based algorithm on Riemannian manifolds. Therefore, we have created a new one by modifying the popular benchmark---the LASA handwriting data-set~\cite{khansari2011learning}---to generate manifold (namely \ac{uq} and \ac{spd}) motions. The LASA handwriting
contains $30$ classes of $2$D Euclidean motions starting from different initial points and converging to the same goal $[0,0]\trsp$. Each motion is demonstrated $7$ times. A demonstration has exactly $1000$ samples and includes position, velocity, and acceleration profiles.

The key idea to generate Riemannian data form Euclidean points is to consider each demonstration as an observation of a motion in the \ac{ts} of a given Riemannian manifold. This allows us to use the exponential map to project the motion onto the manifold. As discussed in~\secref{sec:background}, the \ac{ts} is computed \wrt a point on the manifold. For converging motions, as the one generated by \ac{ours}, the \ac{ts} can be conveniently placed at the goal. We defined the goal as $\boldfrak{q}_g= 1 + [0,0,0]\trsp$ for \acp{uq} and as $\bm{G} = \text{diag}([100, 100])$ for \ac{spd} matrices, but other choices are possible. It is worth noticing that the described procedure is rather general and can be applied to Euclidean benchmarks different from the LASA dataset.

Data in the original LASA dataset are $2$D ($xy$-plane), but the \ac{ts} of \acp{uq} and $2\times 2$ \ac{spd} matrices are\footnote{The \ac{ts} of a \ac{spd} matrix is a symmetric matrix which can be vectorized. For $2\times 2$ \ac{spd} matrices the \ac{ts} has $3$ independent components.} $3$D. To add the third dimension, we consider the $7$ demonstrations of each motion class as a matrix $\bm{C}_i$ for $i=1,\ldots,30$ with $14$ rows and $1000$ columns. Out of each $\bm{C}_i$, we extract the $4$ matrices  $\bm{C}_{1,i} = \bm{C}_i[{0:2}, :]$, $\bm{C}_{2,i} = \bm{C}_i[{4:6}, :]$, $\bm{C}_{3,i} = \bm{C}_i[{8:10}, :]$, and $\bm{C}_{4,i} = \bm{C}_i[[12,13,0], :]$. As a result, we obtain $4$ demonstrations for each motion class, with the third component sampled for the demonstration along the $x$-axis. In this way, the third component is similar across different demonstrations of the same motion---as in a typical \ac{lfd} setting---and contains sufficient complexity. Finally, the $3$D trajectories contained in the matrices $\bm{C}_{1,i}$ to  $\bm{C}_{4,i}$ are then projected to the corresponding manifold using the exponential map.   For \ac{uq}, we scale the data between $[-1,1]$ before projecting them on the unit sphere. 

\subsection{Evaluation procedure}
{We use the Riemannian LASA dataset described in the previous section to compare \ac{ours} against two baselines and three state-of-the-art approaches. The baselines are built considering as base \ac{ds} the Euclidean dynamics in \eqref{eq:linear_ds} and a \ac{gmm}-based diffeomorphism in Euclidean space. The baseline for \acp{uq},  named DS+Normalize, performs an extra normalization step to fulfill manifold constraints. The baseline for \ac{spd} matrices,  named DS+Cholesky, exploits Cholesky decomposition and Mandel's notation to vectorize the matrix for training and re-build an \ac{spd} matrix from the generated vector. The other approaches included in the comparison are \ac{fdm}, \ac{rgmm}, and \ac{eflow}.}
The Riemannian LASA dataset contains $30$ classes. In this experiment, we consider a subset of $26$ individual motions that neglects the $4$  multi-modal motions. The multi-modal motions, obtained by combining $2$ or $3$ individual motions, contain different patterns in different areas of the state space. An approach that effectively learns from multiple demonstrations, like \ac{ours}, \ac{rgmm}, and \ac{eflow}, can encode such a variability. {This is qualitatively shown in the last four plots of \figref{fig:dataset_all}}. On the contrary, approaches that learn from a single demonstrations, like \ac{fdm} and \ac{dmp}, are expected to perform poorly. In order to have a fair comparison with \ac{fdm}, we neglect the $4$  multi-modal motions in our comparison. 
For each of the $26$ classes, we considered all the $4$ available demonstrations in \ac{ours} and only one (average) demonstration for \ac{fdm}. We down-sampled the trajectories to contain exactly $100$ points to significantly speed-up the testing procedures {and test the data-efficiency of each approach.}

{The two baselines, as well as \ac{rgmm} and \ac{ours}, have a single hyper-parameter $k$ that is the number of Gaussian components. For \ac{fdm}, instead, the hyper-parameter $k$ is the number of kernel functions used to fit the diffeomorphism. On the contrary, \ac{eflow} has a few open hyper-parameters including the network structure (number of layers, neurons per layer), the learning rate, and the number of epochs. Performing an exhaustive hyper-parameter search requires a GPU cluster and it is beyond the scope of this work. Hence, we keep fixed the structure of the network and the learning rate (provided by the author's implementation) and vary the number of epochs $k$.}  %
Table~\ref{tab:comparison_lasa_new} reports the value of $k$ used in this experiment.

The performance of each approach is measured considering the accuracy in reproducing the demonstrations contained in the data-set and the \ac{tt}. The accuracy is measured as the \ac{rmse} between each demonstration and the corresponding generated motion (first $100$ steps), \ie a trajectory generated starting from the same point on the manifold. Depending on the manifold (\ac{uq} or \ac{spd}), distances between points are computed considering the proper Riemannian distance (see \secref{sec:background}). The \ac{tt} is divided by the number of classes to represent the time needed to encode a single motion. 

\subsection{Results}
\label{sec:experiments}

\begin{table}[t]
    \centering
    \caption{Results for different learning approaches applied to the Riemannian LASA dataset. We report mean and standard deviation for \ac{rmse} and \ac{tt}.}
    \vspace{0.1cm}
    \label{tab:comparison_lasa_new}
    \resizebox{\columnwidth}{!}{  
 	\begin{tabular}{ccccc}
 	\toprule
 	 $\mathcal{M}$ & \sc{Approach}  & $k$ & RMSE &  TT [s] \\
 	\midrule
 	  & \ac{ours} (ours) &  $10$ & $0.029 \pm 0.019$ & $0.755 \pm 0.241$ \\
 	 & \ac{ds} + Normalize &  $10$ & $0.126 \pm 0.113$ & $1.621 \pm 0.707$ \\
 	 $\mathcal{S}^3$ & \ac{fdm}~\cite{perrin2016fast} &  $250$ &  $0.043 \pm 0.030$  &  $0.201 \pm 0.053$ \\
 	  & \ac{rgmm}~\cite{calinon20gaussians} &  $10$ & $0.036 \pm 0.016$ & $0.387 \pm 0.047$ \\
 	  & \ac{eflow}~\cite{rana2020euclidean} &  $1000$ & $0.141 \pm 0.128$ & $112.252 \pm 0.570$ \\
 	 \midrule
 	  &\ac{ours} (ours) &   $10$ & $0.029 \pm 0.019$ & $1.514 \pm 0.726$ \\
 	  &	 \ac{ds} + Cholesky &   $10$ & $0.121 \pm 0.043$ & $1.899 \pm 0.490$  \\
 	 $\mathcal{S}^2_{++}$ &\ac{fdm}~\cite{perrin2016fast} &  $250$ &  $0.042 \pm 0.029$  &  $0.221 \pm 0.023$ \\
 	 &\ac{rgmm}~\cite{calinon20gaussians} &  $10$ & $ 0.037 \pm 0.017 $ & $14.514 \pm 1.560$ \\
 	& \ac{eflow}~\cite{rana2020euclidean} &  $1000$ & $ 0.140 \pm 0.126 $ & $127.546 \pm 8.434$ \\
 	\midrule
	\multicolumn{5}{l}{\scriptsize Authors would like to thank N. Perrin and P. Schlehuber-Caissier for  providing the source code}\\
	\multicolumn{5}{l}{\scriptsize of the  \ac{fdm} approach in~\cite{perrin2016fast}.}\\
	\bottomrule
    \end{tabular}  
    }
\end{table}

The accuracy of \ac{ours} in learning the motions in the Riemannian LASA dataset is qualitatively shown in~\figref{fig:dataset_all}. We show the demonstrated (black dashed lines) and reproduced (brown solid lines) motions in the \ac{ts}. Recall that, in the  Riemannian LASA dataset, \ac{uq} and \ac{spd} motions share the same \ac{ts} up to a scaling factor. Therefore, we expect similar results in both the manifolds. This is also shown in~\figref{fig:fig_letters_manifold} where the learned trajectories for $4$ motion classes in the dataset are projected on the \ac{spd} (top row) and \ac{uq} (bottom row) manifold. As expected, the generated motion on the manifold follows accurately the demonstration. 
\begin{figure}[h]
    \centering 
    \subfigure[]{\includegraphics[width=0.48\columnwidth]{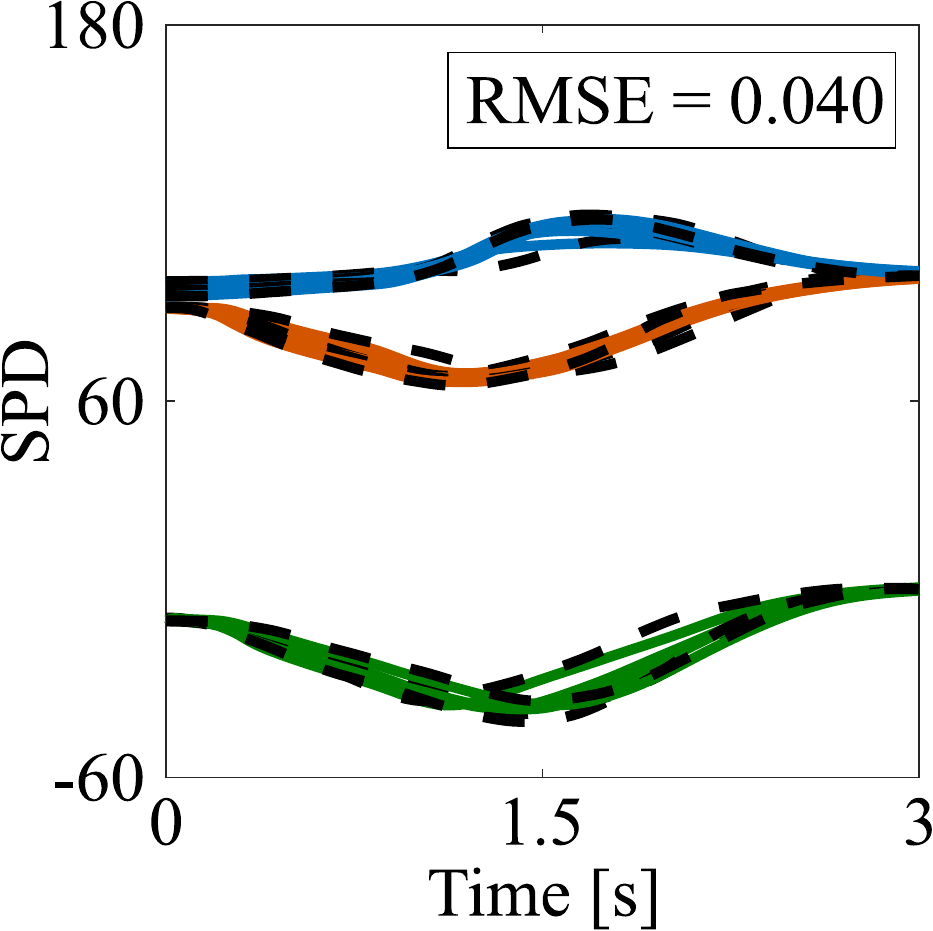}}
    \subfigure[]{\includegraphics[width=0.48\columnwidth]{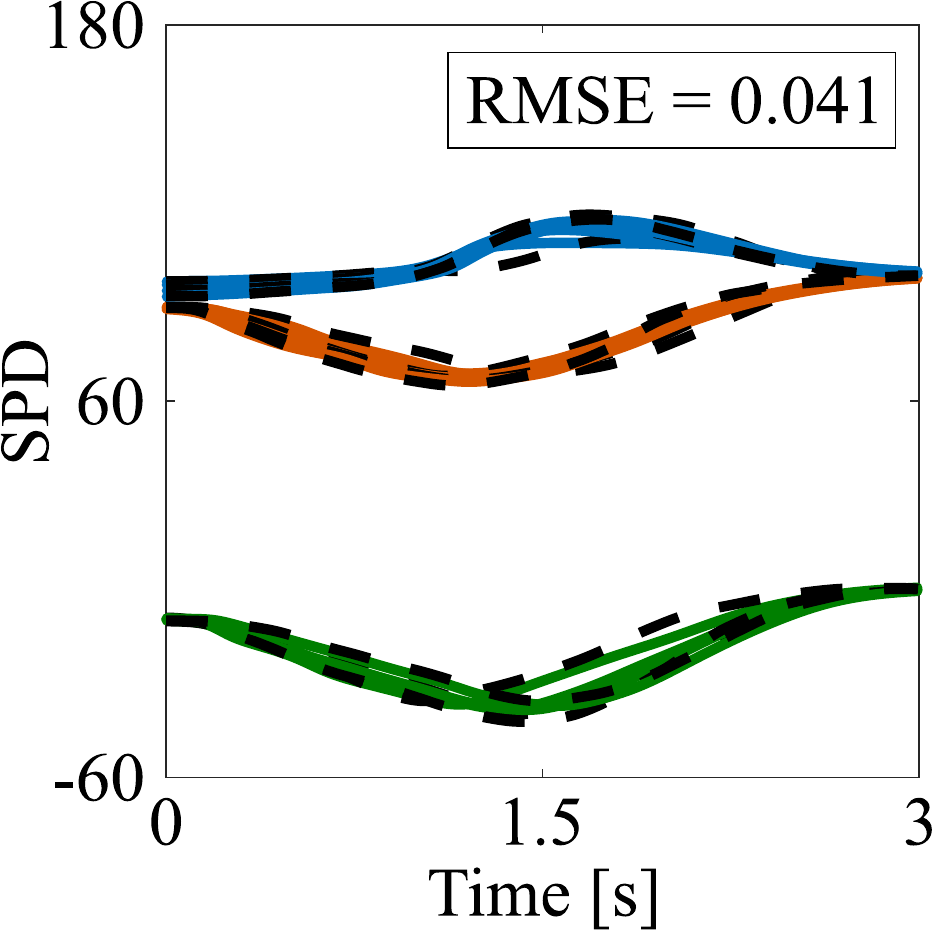}}
    
    \subfigure[]{\includegraphics[width=0.48\columnwidth]{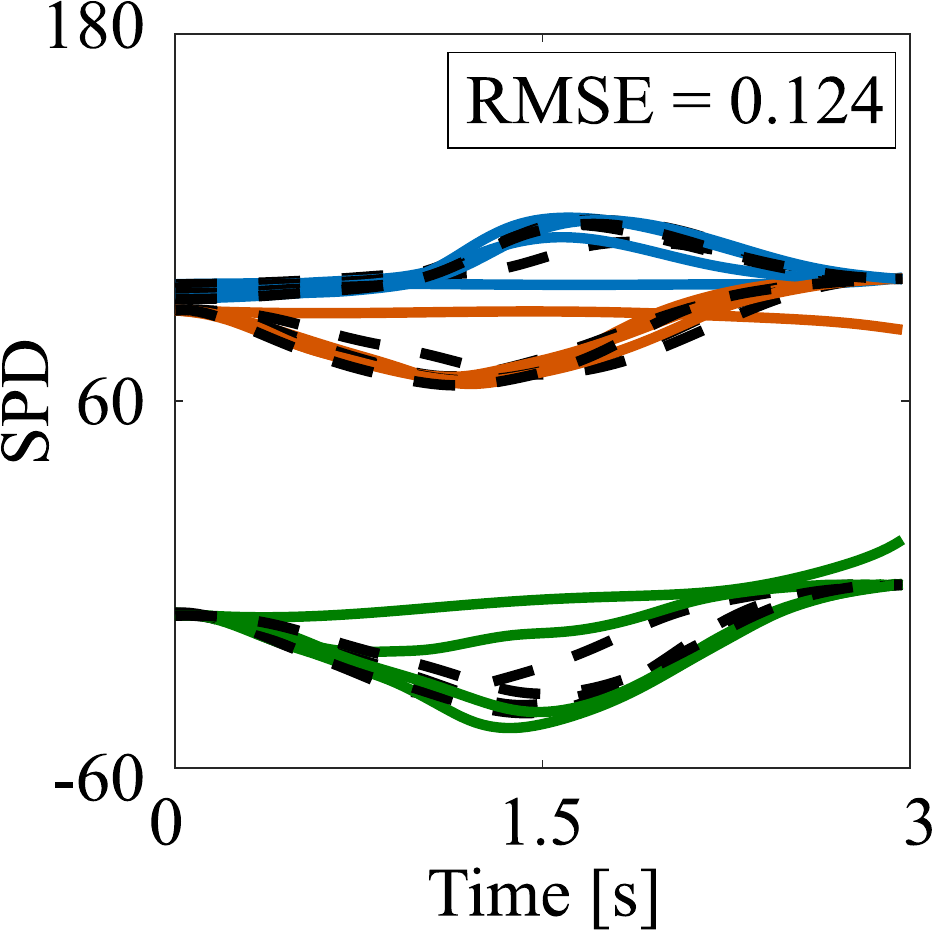}}
    \subfigure[]{\includegraphics[width=0.48\columnwidth]{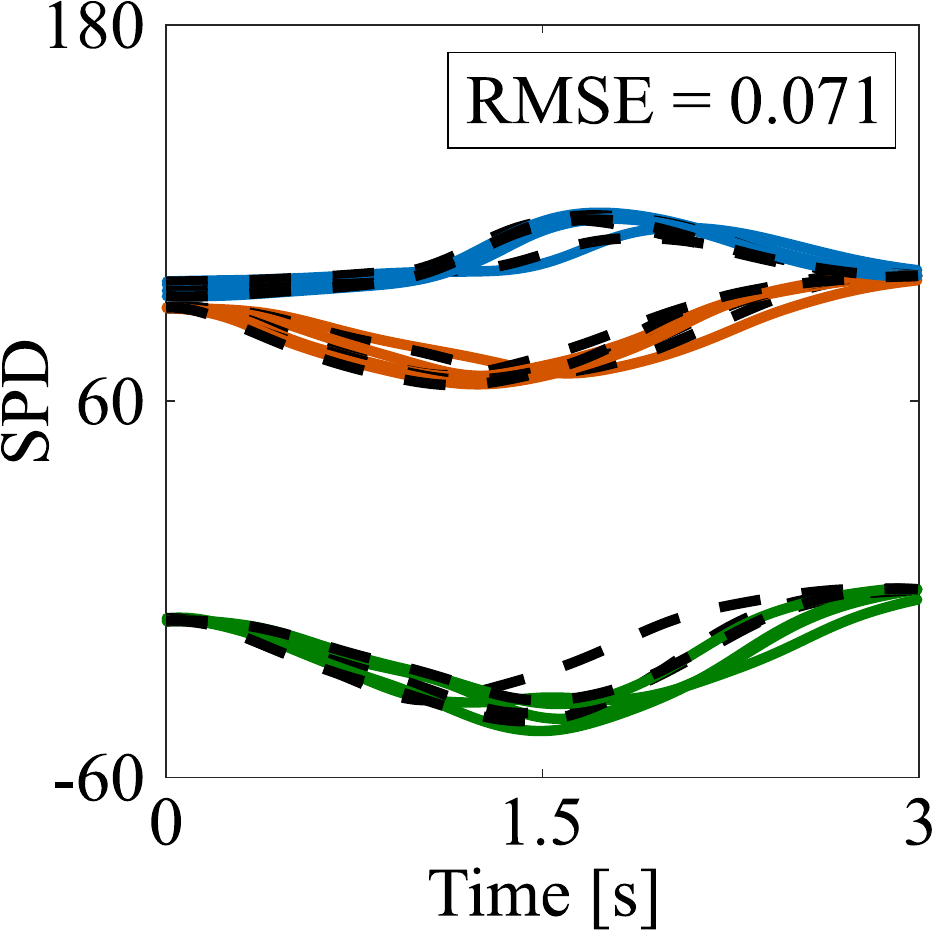}}
    \caption{{Results for the data-efficiency test. \ac{ours} is trained by sub-sampling each demonstration to $100$ points. (a) \ac{ours} reproduces accurately the sub-sampled demonstrations. (b) Adjusting the sampling time \ac{ours} reproduces accurately the original demonstrations ($1000$ points) without re-training. (c) \ac{eflow} is inaccurate when learning from sub-sampled demonstrations. (d) \ac{eflow} achieves the \ac{ours} accuracy when learning from the original demonstrations ($1000$ points).}}
     \label{fig:data_efficiency}
\end{figure}

The quantitative results of this evaluation are shown in~\tabref{tab:comparison_lasa_new}. \ac{ours} accurately represents manifold data and it outperforms \ac{fdm}. Indeed, \ac{ours} is $47\,$\% more accurate on \ac{uq} data and $30\,$\% more accurate on \ac{spd} data than \ac{fdm}. This is an expected result as \ac{fdm} learns from a single demonstration obtained in this case by averaging the $4$ demosntrations in each class. We expect a similar result by applying \ac{dmp}-based approaches~\cite{Ude2014Orientation, abudakka2020Geometry}.        

\begin{figure*}[t]
	\centering
	\def\svgwidth{\textwidth}{\fontsize{8}{8}
		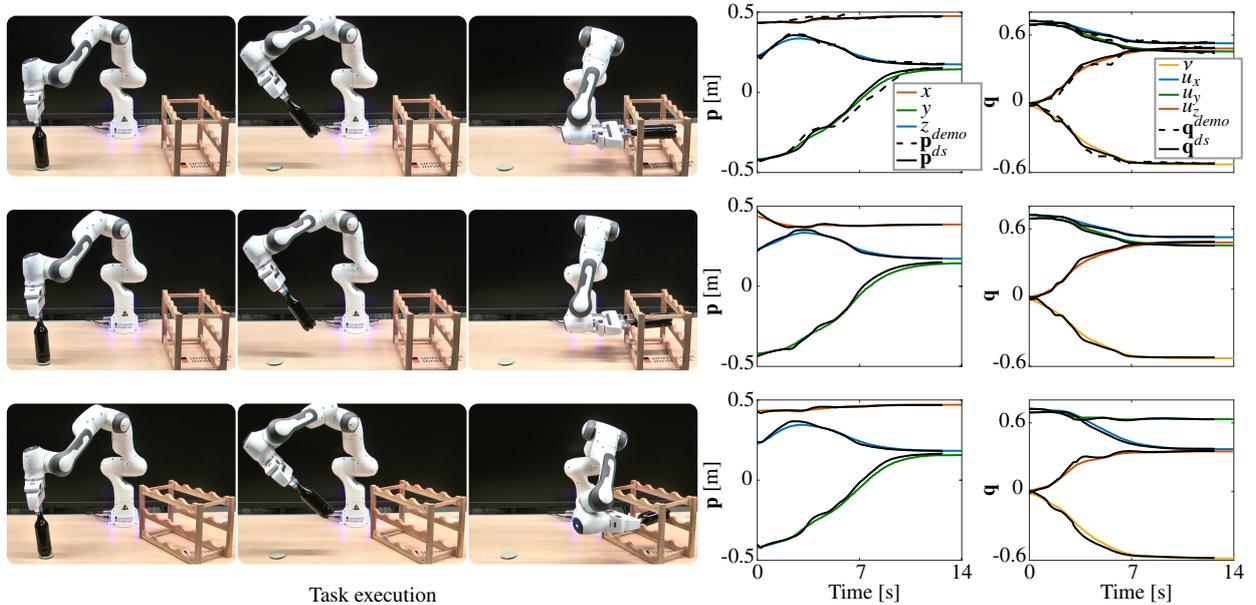}
	\caption{Results for the bottle stacking experiment. (Top) The robot reproduces the trajectory generated by \ac{ours} from known (demonstrated) starting and goal poses. (Middle) The robot stacks the bottle at a different position ($\bm{p}_{\text{g}} = [0.385, 0.143, 0.172]\trsp\,$m). (Bottom) The rack is rotated and the robot stacks the bottle at a different pose ($\bm{p}_{\text{g}} = [0.469, 0.200, 0.165]\trsp\,$m and $\boldfrak{q}_{\text{g}} = -0.58 + [0.37, 0.63, 0.35]\trsp$).}
	\label{fig:bottle_stack_results}
\end{figure*}

Regarding the training time, \ac{ours} learns a \ac{uq} motion ($4$D) on average in $0.755\,$s, while \ac{fdm} takes only $0.201\,$s. For \ac{spd} data ($2\times 2$ matrices), \ac{ours} needs  on average in $1.419\,$s the learn a motion, while \ac{fdm} takes only $0.325\,$s. This is also expected as \ac{fdm} uses only $1$ demonstration for training, resulting in $4$ times less data than \ac{ours}. To summarize, \ac{fdm} is learns very quickly and it is recommended in application where $<0.5\,$s training time is needed. However, most applications have not such a strict training time requirement but need accurate motions. In this case, \ac{ours} offers high accuracy with a reasonable training time. 
%
%
{On the contrary, \ac{eflow} needs about $900\,$s to fit a single motion with $1000$ points. This may be acceptable is some applications, but it is a clear limitation in smart manufacturing where both precision and usability play a key role. }

\section{Robot Experiments}
\label{sec:robot_experiment}
This section presents experiments\footnote{{A video of the experiments is available as supplementary material.}} with a $7$ \ac{dof} manipulator (\panda). The robot behavior is governed by the Cartesian impedance controller
\begin{equation}
\begin{split}
	\bm{F}_{p}&=\bm{K}_{p} \left(\bm{p}_d-\bm{p}\right) + \bm{D}_{p} \left(\dot{\bm{p}}_d-\dot{\bm{p}}\right), \\
	\bm{F}_{o}&=\bm{K}_{o}\, \LogQ_{\q_d}(\q) + \bm{D}_{o} \left(\bm{\omega}_d-\bm{\omega}\right)
	\end{split}
	\label{eq:robot_control}
\end{equation}
where the subscript $p$ stands for position, $o$ for orientation, and $d$ for desired. $\bm{p}_d$ and $\q_d$ are desired position and orientation (expressed as \ac{uq}) of the robot end-effector. $\bm{p}$ and $\q$ indicate the measured position and orientation of the end-effector. Desired and measured linear (angular) velocities are indicated as $\dot{\bm{p}}_d$ and $\dot{\bm{p}}$ (${\bm{\omega}}_d$ and ${\bm{\omega}}$). $\bm{K}_{p\slash o}$ and $\bm{D}_{p\slash o}$ are the robot stiffness and damping matrices expressed in the robot base frame. Given the stiffness matrices $\bm{K}_{p\slash o}$---computed as detailed later in this section---the damping matrices $\bm{D}_{p\slash o}$ are obtaiend by the double diagonalization design~\cite{albu2003cartesian}. Cartesian forces defined in~\eqref{eq:robot_control} are mapped into joint torques using the transpose of the manipulator Jacobian ($\bm{J}\trsp$), 
\begin{equation}
\bm{\tau}_d = \bm{J}\trsp	
\begin{bmatrix}
\bm{F}_{p} \\ \bm{F}_{o} 
\end{bmatrix},
\label{eq:jntImpd}
\end{equation}

\subsection{Bottle stacking}
\label{subsec:stack}
The task of stacking bottles in a rack requires continuous adjustment of position and orientation (see \figref{fig:bottle_stack_results}).	Apart from reaching a desired (stacking) pose, the robot should follow accurately the demonstrated trajectory to prevent hitting the rack. We provide $3$ kinesthetic demonstrations starting at different locations and converging to the stacking pose defined by $\bm{p}_{\text{g}} = [0.474, 0.185, 0.155]\trsp\,$m and $\boldfrak{q}_{\text{g}} = -0.520 + [0.542, 0.442, 0.491]\trsp$.  The demonstrations are of the form $\{\{ \bm{p}^{demo}_{l,d}, \boldfrak{q}^{demo}_{l,d} \}_{l=1}^{L}\}_{d=1}^{D}$ where $L$ is the total length of the demonstrations and $D=3$ is the number of demonstrations. Demonstrated positions and orientations are encoded into two stable \acp{ds} using \ac{ours}. We, empirically, use $10$ Gaussian components for each system. It is worth mentioning that, in order to fit position trajectories, we simply replace logarithmic and exponential maps in~\algoref{alg:ds_riem} with an identity map. The robot is controlled with the Cartesian impedance controller~\eqref{eq:robot_control} where  $\bm{p}_d$ and $\boldfrak{q}_d$ are generated with \ac{ours}. The stiffness matrices are kept to constant high values ($\bm{K}_{p}=\text{diag}([1000, 1000, 1000])\,$N/m and $\bm{K}_{o}=\text{diag}([150, 150, 150])\,$Nm/rad) in this task. 
The results of the learning procedure are shown in the top row of~\figref{fig:bottle_stack_results}.

One of the interesting properties of \ac{ds}-based trajectory generators is the possibility to converge to different goals. Changing the goal is possible also on Riemannian manifolds by following the approach we have presented in~\secref{subsec:proposed}. To demonstrate the robustness of \ac{ours} to goal switching we repeated the stacking task in different conditions. \Figref{fig:bottle_stack_results} (middle) shows the robot successfully stacking the bottle at a different position ($\bm{p}_{\text{g}} = [0.385, 0.143, 0.172]\trsp\,$m). \Figref{fig:bottle_stack_results} (bottom) shows the robot successfully performing the stacking task with a rotated rack, that implies a significant change in the stacking orientation ($\boldfrak{q}_{\text{g}} = -0.58 + [0.37, 0.63, 0.35]\trsp$) and a less pronounced change in the goal position ($\bm{p}_{\text{g}} = [0.469, 0.200, 0.165]\trsp\,$m).

The  results of this experiment show that \ac{ours} accurately encodes full pose trajectories while ensuring convergence to a given target (even if not explicitly demonstrated) and fulfilling the underlying geometric constraints (unit norm) in variable orientation data.

\subsection{Cooperative drilling}
\label{subsec:drill}
In this task, the robot picks a wooden plate from a container and moves it to a demonstrated pose where a human operator will drill it (see \figref{fig:first_img}). Therefore, the robot has to be stiff at the end of the motion to keep the desired pose during the interaction (drilling). During the motion, low impedance gains can be used to make the robot compliant. We provide $3$ kinesthetic demonstrations (see \figref{fig:drill_results}(a)) from different starting poses and converging to the same goal. The demonstrations are of the form $\{\{ \bm{p}^{demo}_{l,d}, \boldfrak{q}^{demo}_{l,d} \}_{l=1}^{L}\}_{d=1}^{D}$ where $L$ is the total length of the demonstrations and $D=3$ is the number of demonstrations.

As in the previous experiment, demonstrated positions and orientations are encoded into two stable \acp{ds} using \ac{ours}. We use $10$ Gaussian components for each system. The desired variable stiffness profiles are generated using the variability in the demonstrations as suggested in several previous work~\cite{calinon2010learning, silverio2019uncertainty, kronander2013learning}. More in details, we first notice that the Cartesian impedance controller~\eqref{eq:robot_control} assumes that position and orientation are decoupled. In other words, it assumes that positional errors only affect the force, while rotational errors only affect the torque. This allows us to learn independent linear and angular stiffness profiles. The idea is to compute a desired stiffness profile from the inverse of the covariance matrix.

\begin{figure}[t]
	\centering
	\def\svgwidth{\columnwidth}{\fontsize{7}{8}
		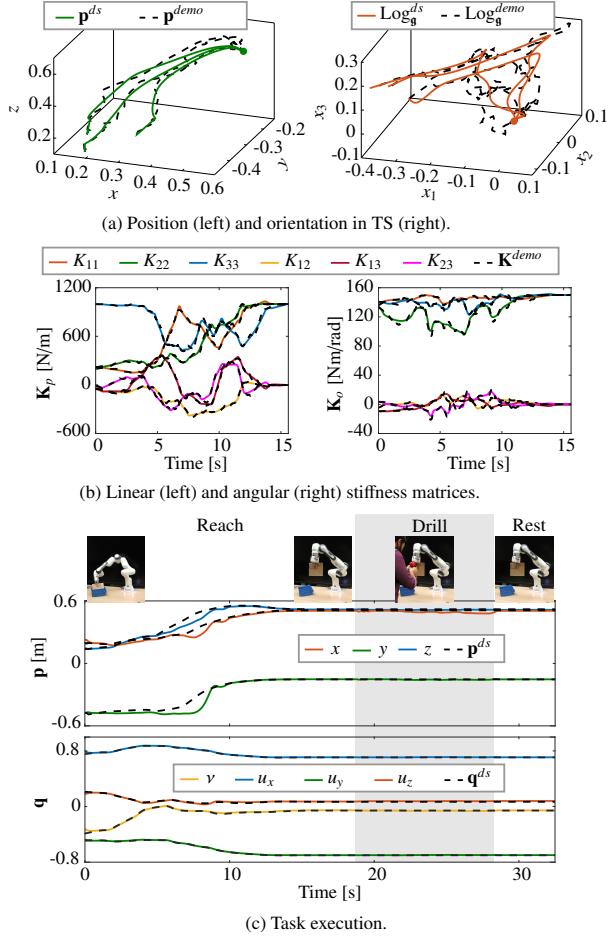}
	\caption{Results for the cooperative drilling experiment.}
	\label{fig:drill_results}
\end{figure}

For the linear stiffness matrix $\bm{K}_p$, we first  compute the covariance matrix for each of the $D$ demonstrated positions and for each time step $l=1,\ldots,L$ as
\begin{equation}
    \bm{\Sigma}_{p,l} = \frac{1}{D}\sum_{i=1}^{D} \left(\bm{p}^{demo}_{l,i} - \bm{\mu}_{p,l}\right) \left(\bm{p}^{demo}_{l,i} - \bm{\mu}_{p,l}\right)\trsp,
    \label{eq:lin_covariance}
\end{equation}
where the mean $\bm{\mu}_{p,l}$ is computed as
\begin{equation}
     \bm{\mu}_{p,l} = \frac{1}{D}\sum_{i=1}^{D} \bm{p}^{demo}_{l,i} .
    \label{eq:lin_covariance_mean}
\end{equation}
Then, we compute the eigenvalue decomposition of each $\bm{\Sigma}_{p,l} = \bm{V}_l\bm{\Lambda}_l\bm{V}_l\trsp$, where $\bm{V}_l$ is an orthogonal matrix and $\bm{\Lambda}_l = \text{diag}([\lambda_{1,l}, \lambda_{2,l}, \lambda_{3,l}])$. Since all the demonstrations end up in the same position, we know that the eigenvalues of $\bm{\Sigma}_{p,L}$ vanishes, \ie $\bm{\Lambda}_L = \bm{0}$ and $\bm{V}=\bm{I}$. Moreover, we want the stiffness to be maximum at $L$. Therefore, we compute the desired stiffness profile as
\begin{equation}
     \bm{K}^{demo}_{p,l} = \bm{V}_l \begin{bmatrix}
     \frac{1}{\lambda_{1,l}+\bar{k}_{p}} & 0 & 0\\ 0 & \frac{1}{\lambda_{2,l}+\bar{k}_{p}} & 0 \\ 0 & 0 & \frac{1}{\lambda_{3,l}+\bar{k}_{p}}
     \end{bmatrix}\bm{V}_l\trsp ,
    \label{eq:lin_stiff}
\end{equation}
where the maximum linear stiffness gain is set to $\bar{k}_{p} = 1000\,$N/m. As shown in~\figref{fig:drill_results}(b), the stiffness profile in~\eqref{eq:lin_stiff} converges to $\bm{K}_{p,L}=\text{diag}([\bar{k}_{p}, \bar{k}_{p}, \bar{k}_{p}])\,$N/m and varies according to the variability in the demonstrations. Note that existing approaches also impose a minimum value for the stiffness. This is straightforward to implement but it was not needed in the performed experiment as the minimum value of the stiffness computed by means of~\eqref{eq:lin_stiff} was already enough to track the desired trajectory.  

The angular stiffness matrix $\bm{K}_o$ is typically kept constant~\cite{calinon2010learning, silverio2019uncertainty} or diagonal~\cite{kronander2013learning} in related work. We propose instead to exploit the variance of the demonstrations in the tangent space of the \ac{uq} to derive a full stiffness profile. This is possible as the tangent space is locally Euclidean. The first step is to project the demonstrated orientations in the tangent space at the goal quaternion $\boldfrak{q}_{\text{g}}$, obtaining ${\{\{\bm{q}^{demo}_{l,d} = \text{Log}_{\boldfrak{q}_{\text{g}}}(\boldfrak{q}^{demo}_{l,d} )\}_{l=1}^{L}\}_{d=1}^{D}}$. We compute the covariance matrix of the tangent space demonstrations $\bm{q}^{demo}_{l,d}$ as
\begin{equation}
    \bm{\Sigma}_{q,l} = \frac{1}{D}\sum_{i=1}^{D} \left(\bm{q}^{demo}_{l,i} - \bm{\mu}_{q,l}\right) \left(\bm{q}^{demo}_{l,i} - \bm{\mu}_{q,l}\right)\trsp,
    \label{eq:ang_covariance}
\end{equation}
where the mean $\bm{\mu}_{q,l}= \frac{1}{D}\sum_{i=1}^{D} \bm{q}^{demo}_{l,i}$. As for the linear stiffness, we compute the eigenvalue decomposition of $\bm{\Sigma}_{o,l} = \bm{U}_l\bm{\Gamma}_l\bm{U}_l$, where $\bm{U}_l$ is an orthogonal matrix and $\bm{\Gamma}_l = \text{diag}([\gamma_{1,l}, \gamma_{2,l}, \gamma_{3,l}])$. Since all the tangent space data end up to zero---as the tangent space is placed at the goal--we know that the eigenvalues of $\bm{\Sigma}_{o,L}$ vanishes, \ie $\bm{\Gamma}_L = \bm{0}$ and $\bm{U}=\bm{I}$. Moreover, we want the stiffness to be maximum at $L$. Therefore, we compute the desired stiffness profile as
\begin{equation}
     \bm{K}^{demo}_{o,l} = \bm{U}_l \begin{bmatrix}
     \frac{1}{\gamma_{1,l}+\bar{k}_{o}} & 0 & 0\\ 0 & \frac{1}{\gamma_{2,l}+\bar{k}_{o}} & 0 \\ 0 & 0 & \frac{1}{\gamma_{3,l}+\bar{k}_{o}}
     \end{bmatrix}\bm{U}_l\trsp ,
    \label{eq:ang_stiff}
\end{equation}
where the maximum angular stiffness gain is set to $\bar{k}_{0} = 150\,$Nm/rad. As shown in~\figref{fig:drill_results}(b), the stiffness profile computed in~\eqref{eq:lin_stiff} converges to $\bm{K}_{o,L}=\text{diag}([\bar{k}_{o}, \bar{k}_{o}, \bar{k}_{o}])\,$Nm/rad and varies according to the variability in the demonstrations.

The generated linear and angular stiffness profiles are encoded into two stable \acp{ds} using \ac{ours}. We, empirically, use $15$ Gaussian components for each system. The results of the learning procedure,  shown in~\figref{fig:drill_results}(b), confirm that \ac{ours} accurately reproduces complex \ac{spd} profiles while ensuring convergence to a given goal. 

After the learning, the pose trajectory and stiffness profiles are used to control the robot (see~\figref{fig:drill_results}(c)). The robot picks the wooden plate from a (blue) container and reaches the drill pose. During the motion the robot is complaint which allows a safer response to possible external perturbation. The goal pose, instead, is reached with maximum stiffness. As shown in~\figref{fig:drill_results}(c), during the drilling task the maximum position deviation along the drilling direction ($x$-axis) is $3.1\,$cm, while the maximum orientation deviation about the $z$-axis (where the robot perceives the highest momentum) is $1.8\,$deg. This shows that the robot is capable to keep the goal pose, letting the human co-worker to drill the wooden plate.

\section{Conclusions}
\label{sec:concl}

In this paper, we presented \acf{ours}, an approach to learn stable \acp{ds} evolving on Riemannian manifolds. \ac{ours} builds on theoretical stability results, derived for dynamics evolving on Riemannian manifolds, to learn stable and accurate \ac{ds} representations of Riemannian data. Similar to its Euclidean counterparts, \ac{ours} learns a diffeomorphic transformation between a simple stable system and a set of complex demonstrations. The key difference \wrt Euclidean approaches is that  \ac{ours} uses tools from differential geometry to correctly represent complex manifold data, such as orientations and stiffness matrices, with their underlying geometric constraints, \eg unit norm for unit quaternion orientation and symmetry and positive definiteness for stiffness matrices. The proposed approach is firstly evaluated in simulation and compared with an existing approach, modified to deal with Riemannian data. Due to the lack of publicly available Riemannian datasets, we developed a procedure to augment a popular---and potentially any other---Euclidean benchmark with \ac{uq} and \ac{spd} profiles. Finally, in order to perform a thorough evaluation, we also conducted a set of experiments with a real robot performing bottle stacking and cooperative (with a human operator) drilling. Overall, the conducted evaluation shows that \ac{ours} represents a good compromise between accuracy and training time, and that it can be effectively adopted to generate complex robotic skills on manifolds.

\ac{ours} has been evaluated on orientation (\ac{uq}) and stiffness (\ac{spd}) profiles, but it may be extended to other Riemannian manifolds. Therefore, our future research will focus on investigating the possibility to learn stable \ac{ds} on diverse manifolds like Grassmann or hyperbolic. Grassmann manifolds elegantly encode orthogonal projections, while hyperbolic manifolds represent a continuous embedding of discrete structures with possible application to task and motion planning. These manifolds are widely unexploited in robotics and can potentially unlock new applications.


\appendix

\section{Jacobian of the mean of a GMR}\label{app:gmr_jacobian}
Recall that
\begin{equation}
\begin{split}
\psi(\bP) &= \sum_{k=1}^K h_k(\bP)\hat{\psi}(\bP), \\
\hat{\psi}(\bP) &= \bm{\mu}_k^{b} + \bm{\Sigma}_k^{ba}(\bm{\Sigma}_k^{aa})^{-1}(\bP - \bm{\mu}_k^{a}).
\label{eq:D_hat}
\end{split}
\end{equation}
Using the chain rule and~\eqref{eq:D_hat}, $J_\psi(\bP)$ writes as:
\begin{equation}
\begin{split}
\bm{J}_\psi(\bP) = \frac{\partial \psi(\bP)}{\partial \bP} = & \sum_{k=1}^K \frac{\partial h_k(\bP)}{\partial \bP} \hat{\psi}(\bP)\trsp \\ & + h_k(\bP)\frac{\partial \hat{\psi}(\bP)}{\partial \bP}.
\label{eq:J_chain_rule}
\end{split}
\end{equation}
Let us compute the two partial derivatives at the right side of~\eqref{eq:J_chain_rule} separately. Considering the expression of $\hat{\psi}(\bP)$ in~\eqref{eq:D_hat}, and applying the chain rule, it is easy to verify that
\begin{equation}
\frac{\partial \hat{\psi}(\bP)}{\partial \bP} = \bm{\Sigma}_k^{ab}(\bm{\Sigma}_k^{aa})^{-1}.
\label{eq:Jac_D_hat}
\end{equation}
Using the quotient rule, and setting $\hat{\gauss}_k = \gauss(\bP|\bm{\mu}_k^{a}, \bm{\Sigma}_k^{aa})$, the expression of $\frac{\partial h_k(\bP)}{\partial \bP}$ writes as
\begin{equation}
\begin{split}
    \frac{\partial h_k(\bP)}{\partial \bP} = & \frac{\pi_k\frac{\partial \hat{\gauss}_k}{\partial \bP}\sum_{i=1}^K \pi_i\hat{\gauss}_i}{\left(\sum_{i=1}^K \pi_i\hat{\gauss}_i\right)^2} \\ &-\frac{\pi_k \hat{\gauss}_k \sum_{i=1}^K \pi_i\frac{\partial \hat{\gauss}_i }{\partial \bP}}{\left(\sum_{i=1}^K \pi_i\hat{\gauss}_i\right)^2}  .
    \label{eq:h_division_rule}
\end{split}
\end{equation}
Recall that the derivative of a multivariate Gaussian distribution $\hat{\gauss}$ \wrt the input is given by~\cite{petersen2012matrix}
\begin{equation}
    \frac{\partial \hat{\gauss}}{\partial \bP} = -\hat{\gauss} \bm{\Sigma}^{-1}( \bP - \bm{\mu}). 
    \label{eq:grad_gauss}
\end{equation}
Using~\eqref{eq:grad_gauss} to compute the derivatives in~\eqref{eq:h_division_rule} we obtain:
\begin{equation}
\begin{split}
    &\frac{\partial h_k(\bP)}{\partial \bP} = \frac{-\pi_k\hat{\gauss}_k (\bm{\Sigma}_k^{aa})^{-1}(\bP - \bm{\mu}_k^a)\sum_{i=1}^K \pi_i\hat{\gauss}_i}{\left(\sum_{i=1}^K \pi_i\hat{\gauss}_i\right)^2} \\
    &+ \frac{\pi_k \hat{\gauss}_k \sum_{i=1}^K \pi_i \hat{\gauss}_i (\bm{\Sigma}_i^{aa})^{-1}(\bP - \bm{\mu}_i^a)}{\left(\sum_{i=1}^K \pi_i\hat{\gauss}_i\right)^2} \\
    &= \frac{\pi_k\hat{\gauss}_k}{\sum_{i=1}^K \pi_i\hat{\gauss}_i} \left( -(\bm{\Sigma}_k^{aa})^{-1}(\bP - \bm{\mu}_k^a) \frac{\sum_{i=1}^K \pi_i\hat{\gauss}_i}{\sum_{i=1}^K \pi_i\hat{\gauss}_i}\right.\\
    &+ \left. \frac{\sum_{i=1}^K \pi_i \hat{\gauss}_i(\bm{\Sigma}_i^{aa})^{-1}(\bP - \bm{\mu}_i^a)}{\sum_{i=1}^K \pi_i\hat{\gauss}_i}\right) \\
    &= h_k(\bP) \left( -(\bm{\Sigma}_k^{aa})^{-1}(\bP - \bm{\mu}_k^P)\right.\\
    &+ \left. \sum_{i=1}^K h_i(\bP)(\bm{\Sigma}_i^{aa})^{-1}(\bP - \bm{\mu}_i^a)\right).
\end{split}
\label{eq:Jac_h_k}
\end{equation}
By substituting~\eqref{eq:Jac_D_hat} and~\eqref{eq:Jac_h_k} into~\eqref{eq:J_chain_rule}, we obtain the sought expression of the Jacobian in~\eqref{eq:gmr_jacobian}.

\section*{Acknowledgements}
Part of the research presented in this work has been conducted when M. Saveriano was at the Department of Computer Science, University of Innsbruck, Innsbruck, Austria.

This work has been partially supported by the Austrian Research Foundation (Euregio IPN 86-N30, OLIVER) and by CHIST-ERA project IPALM (Academy of Finland decision 326304).

\bibliographystyle{elsarticle-num}
\bibliography{ref.bib}

\end{document}